\documentclass{article}
\usepackage{fullpage}

\usepackage[utf8]{inputenc}
\usepackage[T1]{fontenc}
\usepackage{hyperref}
\usepackage[numbers]{natbib}
\usepackage{algorithm}
\usepackage{algorithmic}
\floatstyle{ruled}
\restylefloat{algorithm}
\usepackage{url}
\usepackage{booktabs}
\usepackage{amsfonts}
\usepackage{nicefrac}
\usepackage{microtype}
\usepackage{graphicx}
\usepackage{subcaption}
\usepackage{multicol}
\usepackage{multirow}
\usepackage{etoc}
\usepackage{minitoc}
\usepackage{titletoc}  
\usepackage{amsmath}
\usepackage{amssymb}
\usepackage{mathtools}
\usepackage{amsthm}
\usepackage{xcolor}
\usepackage[normalem]{ulem}
\usepackage{wrapfig}
\usepackage[capitalize,noabbrev]{cleveref}
\usepackage[textsize=tiny]{todonotes}

\graphicspath{{img/}{media/}}

\definecolor{mydarkblue}{rgb}{0,0.08,0.45}
\hypersetup{
    colorlinks = true,
    linkcolor  = mydarkblue,
    citecolor  = mydarkblue,
    urlcolor   = mydarkblue,
    filecolor  = mydarkblue,
    pdfborder  = {0 0 0},
}

\newif\ifcomments
\commentsfalse

\definecolor{deeppurple}{rgb}{0.4,0.0,0.6}

\definecolor{fweducolor}{rgb}{0.55,0.0,0.55}
\newcommand{\fwedu}[1]{#1}

\theoremstyle{plain}
\newtheorem{theorem}{Theorem}[section]

\newtheorem{lemma}[theorem]{Lemma}

\theoremstyle{definition}

\newtheorem{assumption}[theorem]{Assumption}
\theoremstyle{remark}
\newtheorem{remark}[theorem]{Remark}

\newtheorem*{theorem-restated}{Theorem}
\newtheorem*{assumption-restated}{Assumption}
\newtheorem*{corollary-restated}{Corollary}
\newtheorem*{lemma-restated}{Lemma}

\title{RMNP: Row-Momentum Normalized Preconditioning for Scalable Matrix-Based Optimization}

\newcommand\blfootnote[1]{%
  \begingroup
  \renewcommand\thefootnote{}\footnote{#1}%
  \addtocounter{footnote}{-1}%
  \endgroup
}

\author{
  Shenyang Deng$^{1,\ast}$, 
  Zhuoli Ouyang$^{1,\ast}$, 
  Tianyu Pang$^{1}$, 
  Zihang Liu$^{2,3}$, \\
  Ruochen Jin$^{1}$,
  Shuhua Yu$^{4}$,
  Yaoqing Yang$^{1}$ \\[3mm]
  $^1$Dartmouth College \\
  $^2$International Computer Science Institute \\
  $^3$University of California, Berkeley \\
  $^4$Meta \\[3mm]
  \small $^1$\texttt{\{shenyang.deng.gr, zhuoli.ouyang.gr, tianyu.pang.gr, ruochen.jin.gr, yaoqing.yang\}@dartmouth.edu} \\
  \small $^{2,3}$\texttt{zihang.liu@berkeley.edu} \quad $^4$\texttt{yu.shuhuaxh@gmail.com}
}

\begin{document}
\date{}
\maketitle
\blfootnote{$^\ast$Indicates equal contribution.}

\begin{abstract}
Preconditioned adaptive methods have gained significant attention for training deep neural networks, as they capture rich curvature information of the loss landscape . The central challenge in this field lies in balancing preconditioning effectiveness with computational efficiency of implementing the preconditioner. Among recent advances, \textsc{Muon} stands out by using Newton-Schulz iteration to obtain preconditioned updates without explicitly constructing the preconditioning matrix. Despite its advantages, the efficiency of \textsc{Muon} still leaves room for further improvement.
In this paper, we introduce \textsc{RMNP} (Row Momentum Normalized Preconditioning), an optimizer that replaces Newton-Schulz iteration with a simple row-wise($d_{\text{in}}$) $\ell_2$ normalization operation, motivated by the empirically observed diagonal block structure of the Transformer layerwise Hessian. {We empirically verified that orthogonalization and row-wise(on input dim) $\ell_2$ normalization are asymptotically equivalent in the case of the transformer.} This substitution reduces the per-iteration computational complexity from $\mathcal{O}(mn\cdot\min(m,n))$ to $\mathcal{O}(mn)$ for an $m\times n$ weight matrix while maintaining comparable optimization performance. Theoretically, we establish convergence guarantees for \textsc{RMNP} in the non-convex setting that match recent results for \textsc{Muon} optimizers, achieving the minimax optimal complexity. Extensive experiments on large language model pretraining show that \textsc{RMNP} delivers competitive optimization performance compared with \textsc{Muon} while substantially reducing preconditioning wall-clock time. Our code is available at \href{https://github.com/Dominator-Index/RMNP}{this link}.
\end{abstract}

\section{Introduction}
Adaptive algorithms, such as those introduced in \citet{duchi2011adaptive,tieleman2012lecture,kingma2014adam,loshchilov2018decoupled}, have achieved remarkable success in deep learning optimization. These methods employ diagonal preconditioning \citep{duchi2011adaptive}, which scales each parameter independently based on historical gradient information. However, this diagonal structure ignores correlations among parameters, limiting the optimizer's ability to handle ill-conditioned problems with complex parameter interactions. This creates a fundamental gap between practical diagonal methods and the theoretically optimal full-matrix preconditioning.

Recent work has revisited matrix-based preconditioning to address these limitations. In particular, studies on full Gauss-Newton methods~\citep{abreu2025fullgn} demonstrate that utilizing complete curvature information can lead to qualitatively improved convergence behavior in large language models. However, directly applying updates of the form $w_t = w_{t-1} - H_t^{-1} d_t$ remains computationally prohibitive: if the preconditioner $H_t$ is constructed using the full Hessian, the computational overhead becomes extreme and scales poorly with model size. Consequently, practical optimizer design has focused on structured approximations with first-order information to balance performance with efficiency.

Classic methods such as \textsc{K-FAC}~\cite{martens2015optimizing}, \textsc{PSGD}~\cite{li2018preconditioned}, and \textsc{Shampoo}~\cite{gupta2018shampoo} achieve this balance through structured matrix preconditioners that approximate curvature with lower-dimensional factors. \textsc{Shampoo}, for example, employs a Kronecker-factored preconditioner:
\begin{equation}
    H = L \otimes R
\end{equation}
where $L$ and $R$ are smaller matrices capturing row and column correlations, respectively. This factorization preserves essential curvature information while dramatically reducing computational demands. Subsequent works including \textsc{K-BFGS}~\cite{ren2021kronecker} and \textsc{ASGO}~\cite{an2025asgo} further refine this approach with sparse or low-rank updates to minimize memory overhead.

More recently, methods such as \textsc{Muon}~\cite{jordan2024muon} have introduced an alternative perspective on matrix-based adaptivity (see Algorithm~\ref{algomuon}). Rather than explicitly forming the full preconditioner $H^{-1}$, \textsc{Muon} employs Newton-Schulz iterations to implicitly compute the preconditioned updates $H_t^{-1}d_t$ through matrix polynomials, enabling matrix-level adaptation without direct inversion. Subsequent refinements further improve the stability and efficiency of this approach~\cite{tian2022amos, vyas2025soap, si2025adamuon, liu2025cosmos}. Overall, these methods move beyond element-wise diagonal preconditioning by incorporating structured off-diagonal curvature information, aiming to achieve a \textbf{better trade-off between optimization performance and computational cost}. However, despite this conceptual advancement, the reliance on iterative matrix polynomial evaluations in \textsc{Muon} incurs a computational complexity of $\mathcal{O}(mn\cdot\min(m,n))$ for an $m \times n$ weight matrix, which can become a dominant bottleneck as model dimensions~grow.

\begin{figure*}[!t]
\centering

\begin{minipage}[t]{0.49\textwidth}
\hrule
\vspace{2pt}
\captionsetup{type=algorithm,justification=raggedright,singlelinecheck=false,aboveskip=0pt,belowskip=0pt}
\caption{\textsc{Muon} \cite{jordan2024muon}}
\label{algomuon}
\hrule
\vspace{2pt}
\begin{algorithmic}[1]
\REQUIRE $\eta_t > 0, \beta \in [0,1), W_0 \in \mathbb{R}^{m \times n}$, loss $f$
\STATE $V_0 \leftarrow \mathbf{0}_{m \times n}$
\FOR{$t = 1$ to $T$}
    \STATE $G_t \leftarrow \nabla f(W_t; \xi^t)$
    \STATE $V_t \leftarrow \beta V_{t-1} + (1-\beta) G_t$
    \STATE $D_t \leftarrow \text{NS}_5(V_t)$, $D_t\approx(V_tV_t^T)^{-\frac{1}{2}}V_t$
    \STATE $W_{t+1} \leftarrow W_t - \eta_t D_t$
\ENDFOR
\end{algorithmic}
\vspace{2pt}
\hrule
\end{minipage}\hfill
\begin{minipage}[t]{0.49\textwidth}
\hrule
\vspace{2pt}
\captionsetup{type=algorithm,justification=raggedright,singlelinecheck=false,aboveskip=0pt,belowskip=0pt}
\caption{\textsc{RMNP}}
\label{algoRMNP}
\hrule
\vspace{2pt}
\begin{algorithmic}[1]
\REQUIRE $\eta_t > 0, \beta \in [0,1), W_0 \in \mathbb{R}^{m \times n}$, loss $f$
\STATE $V_0 \leftarrow \mathbf{0}_{m \times n}$
\FOR{$t = 1$ to $T$}
    \STATE $G_t \leftarrow \nabla f(W_t; \xi^t)$
    \STATE $V_t \leftarrow \beta V_{t-1} + (1-\beta) G_t$
    \STATE $D_t \leftarrow \text{RN}(V_t)$, $D_t=(\text{diag}(V_tV_t^T))^{-\frac{1}{2}}V_t$
    \STATE $W_{t+1} \leftarrow W_t - \eta_t D_t$
\ENDFOR
\end{algorithmic}
\vspace{2pt}
\hrule
\end{minipage}

\vspace{6pt}

\begin{minipage}[t]{0.49\textwidth}
\vspace{0pt}
\centering
\includegraphics[width=\linewidth]{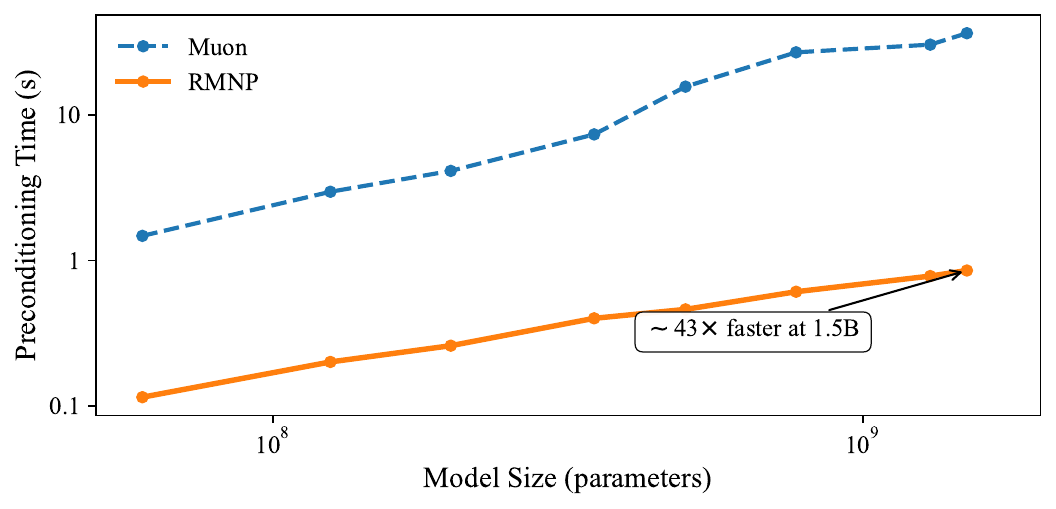}
\vspace{-12pt}
\captionsetup{type=figure}
\caption{Time overhead comparison. The figure illustrates the wall-clock time for 100 computation steps for preconditioning process of \textsc{RMNP} versus \textsc{Muon}.}
\label{Time_Comprare}
\end{minipage}\hfill
\begin{minipage}[t]{0.49\textwidth}
\vspace{0pt}
\centering
\captionsetup{type=table}
\scriptsize
\setlength{\tabcolsep}{2.5pt}
\renewcommand{\arraystretch}{1.04}
\resizebox{\linewidth}{!}{%
\begin{tabular}{@{}lccc@{}}
\hline
Work & Smooth & Conv. & Complexity \\
\hline
\multicolumn{4}{c}{\textit{\textsc{Muon}}} \\
\hline
\cite{shen2025convergence} & $L_F$ & $\|\nabla f\|_\ast$ & $O(m^2L\sigma^2\Delta\epsilon^{-4})$ \\
\cite{kim2026convergencemuonnewtonschulz} & $L_\ast$ & $\|\nabla f\|_\ast$ & $O(mL_\ast\sigma^2\Delta\epsilon^{-4})$ \\
\cite{shen2025convergence} & $L_\ast$ & $\|\nabla f\|_\ast$ & $O(mL_\ast\sigma^2\Delta\epsilon^{-4})$ \\
\hline
\multicolumn{4}{c}{\textit{\textsc{RMNP}}} \\
\hline
Thm.~\ref{thm:fro-convergence} & $L_F$ & $\|\nabla f\|_F$ & $O(m^2L_F\sigma^2\Delta\epsilon^{-4})$ \\
Thm.~\ref{thm:12-fro-convergence} & $L_F$ & $\|\nabla f\|_{1,2}$ & $O(m^2L_F\sigma^2\Delta\epsilon^{-4})$ \\
Thm.~\ref{thm:inf2-convergence} & $L_{\infty,2}$ & $\|\nabla f\|_{1,2}$ & $O(mL_{\infty,2}\sigma^2\Delta\epsilon^{-4})$ \\
\hline
\end{tabular}%
}
\caption{Comparison of Convergence Results. $L_F,L_\ast$ denotes the corresponding smoothness coefficient and $\|\nabla f\|_F, \|\nabla f\|_\ast$ the corresponding convergence criterion.}
\label{tab:comparison}
\end{minipage}
\end{figure*}

In this paper, we show that the computational complexity of \textsc{Muon} can be further reduced without sacrificing its matrix-level adaptivity. Specifically, motivated by recent empirical and theoretical findings on the structure of Transformer Hessians~\cite{zhang2024transformers,dong2025towards}, we introduce Row Momentum Normalized Preconditioning (\textsc{RMNP}, Algorithm~\ref{algoRMNP}).
\textsc{RMNP} achieves optimization performance comparable to \textsc{Muon} while substantially reducing the preconditioning overhead, with a per-iteration computational complexity of $\mathcal{O}(mn)$. We further benchmark the wall-clock time of both optimizers under identical settings, as shown in Figure~\ref{Time_Comprare}, demonstrating an order-of-magnitude reduction in preconditioning cost.

Mechanistically, \textsc{RMNP} replaces the Newton-Schulz iteration in \textsc{Muon} with a simple row-wise $\ell_2$ normalization. In Section \ref{sec_preconditioner_in}, we provide a mathematical interpretation of this operation from a preconditioning perspective, showing that it corresponds to a further structured approximation of \textsc{K-FAC} aligned with the observed block-diagonal dominance of Transformer curvature. We also discuss how \textsc{RMNP} differs from related approaches and provide practical hyperparameter recommendations.
Furthermore, we provide non-convex convergence guarantees for \textsc{RMNP}. As summarized in Table \ref{tab:comparison}, our results match the best-known theoretical guarantees for \textsc{Muon}~\cite{shen2025convergence,kim2026convergencemuonnewtonschulz}. This complexity result achieves the minimax optimality in non-convex smooth setup~\cite{arjevani2023lower}. For more related work, please refer to the Appendix \ref{sec_related}.

\begin{figure*}[!htb]
    \centering
    \includegraphics[width=\textwidth]{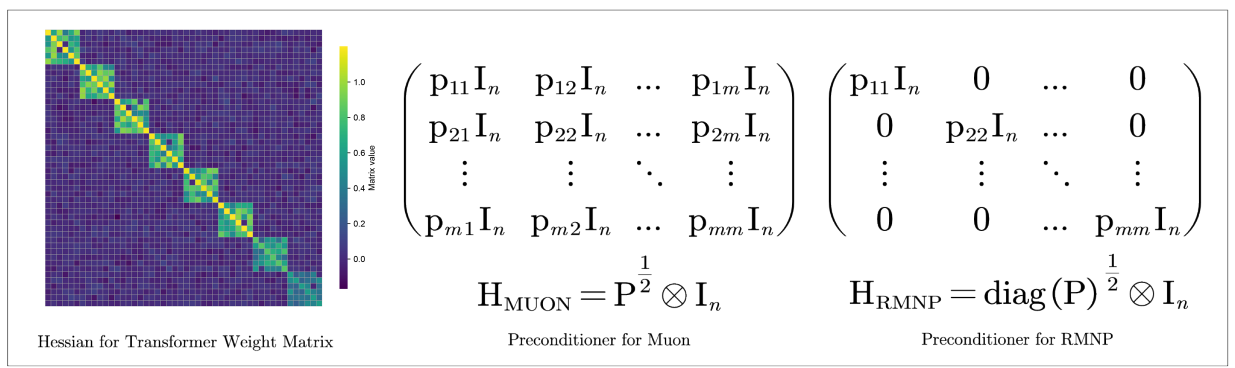}
    \caption{Comparison among Transformer layerwise Hessian, Preconditioner for \textsc{Muon} ,  and Preconditioner for \textsc{RMNP}. The figure of Transformer layerwise Hessian is conceptual, the real case can be widely found in \cite{zhang2024adammini,zhang2024transformers,dong2025towards}.
    $P=V_tV_t^T$,  $m$ and $n$ are the number of rows and columns of the weight matrix, respectively. In  Section \ref{subsec:diagonal_dominance} we further verified through experiments that the \textsc{Muon} preconditioner has such a certain diagonal dominance property.  }
    \label{fig:preconditioner_comparison}
\end{figure*}

Our key contributions are summarized as follows:

\begin{itemize}
\item \textbf{Structure-Aware Preconditioning with Lower Computational Complexity.} We propose \textsc{RMNP}, a matrix-based adaptive optimizer that replaces Newton--Schulz iterations in \textsc{Muon} with a row-wise $\ell_2$ normalization operation motivated by the observed block-diagonal dominance of Transformer curvature. This design preserves matrix-level adaptivity while reducing the per-iteration computational complexity from $\mathcal{O}(mn \cdot \min(m,n))$ to $\mathcal{O}(mn)$.

\item \textbf{Empirical Analysis and Evaluation on Large Language Models.}
We empirically validate the diagonal dominance properties of the \textsc{Muon} preconditioner that underlie our design hypothesis. We also conduct comparative experiments across various model architectures spanning multiple scales. Our results demonstrate that \textsc{RMNP} consistently matches or exceeds the final perplexity of \textsc{Muon} while achieving up to an order-of-magnitude reduction in preconditioning wall-clock time.
\item \textbf{Non-Convex Convergence Guarantees.} We establish convergence analysis for \textsc{RMNP} under the non-convex smooth setting. Our theoretical results provide convergence guarantees that are on par with the current state-of-the-art theory for \textsc{Muon}, ensuring the robustness of our proposed method despite the reduced complexity. We also show that our results achieve minimax optimal complexity.\end{itemize}

\section{Related Work}
\label{sec_related}

{\noindent\textbf{Discussion with Recent Row Normalization Optimizers.~}\citet{zhang2024adammini} is the first work to introduce row-wise normalization into optimizer design, assigning a single learning rate per row (i.e., per output neuron) of each weight matrix to drastically cut Adam's memory while matching its performance, and \citet{pethick2025training} subsequently proposed the abstract LMO framework that unifies many modern optimizers as steepest descent under a chosen norm. Follow the LMO framework, a number of papers derive row- or column-normalized optimizers from this viewpoint. \textsc{SRON}~\citep{anonymous2025sron} applies row-wise normalization to plain SGD, motivated by row-level gradient disparities in attention. \textsc{SCALE}~\citep{glentis2025minimalist} shows that column-wise normalization (which is along the $d_{\text{in}}$ dimension, consistent with the normalization axis of the aforementioned works) plus last-layer momentum is a minimal modification to SGD that matches Adam. \textsc{SWAN}~\citep{ma2025swan} combines row-wise standardization with gradient whitening as a stateless preprocessing. \textsc{MNGD}~\citep{scetbon2025gradient} generalizes this via an alternating scheme enforcing multiple norms simultaneously. \textsc{Mano}~\citep{gu2026mano} recasts row normalization as Riemannian optimization on a rotational Oblique manifold. \textsc{MOGA}~\citep{xu2026width} derives row/column normalization from mean-normalized operator norms, yielding width-independent smoothness and $\mu$P-style learning-rate transfer.

\noindent\textbf{Why steepest-descent analyses cannot explain NN-specific benefits.}~As illustrated in Figure~\ref{fig:problem-classes}, all of the above works expect \citet{zhang2024adammini} analyze their algorithms benefit through the steepest-descent lens (It mainly refers to the abstract LMO framework.), which inherently considers only the \emph{worst-case problem for the algorithm} within a broad problem class such as nonconvex $L$-smooth, and therefore provides only a floor guarantee. While such a guarantee is meaningful in its own right, and we provide a similar result in our paper, \emph{it cannot explain why a particular norm is specifically well-suited to neural-network optimization}: the analysis is agnostic to the actual loss landscape, so any norm choice looks equally justifiable at the worst-case level. To understand why these algorithms actually work on NNs, one has to examine the concrete problem structure itself. \textbf{Our analysis therefore departs from the steepest-descent viewpoint and starts from the curvature structure of neural networks. Motivated by recent work on the Hessian structure\cite{zhang2024adammini,zhang2024transformers,dong2025towards} of neural networks, we verify that full orthogonalization and row $\ell_2$-normalization exhibit a high-dimensional asymptotic equivalence for Transformers.}}

\begin{wrapfigure}{r}{0.4\textwidth}
    \vspace{-25pt} 
    \centering
    \includegraphics[width=0.4\textwidth]{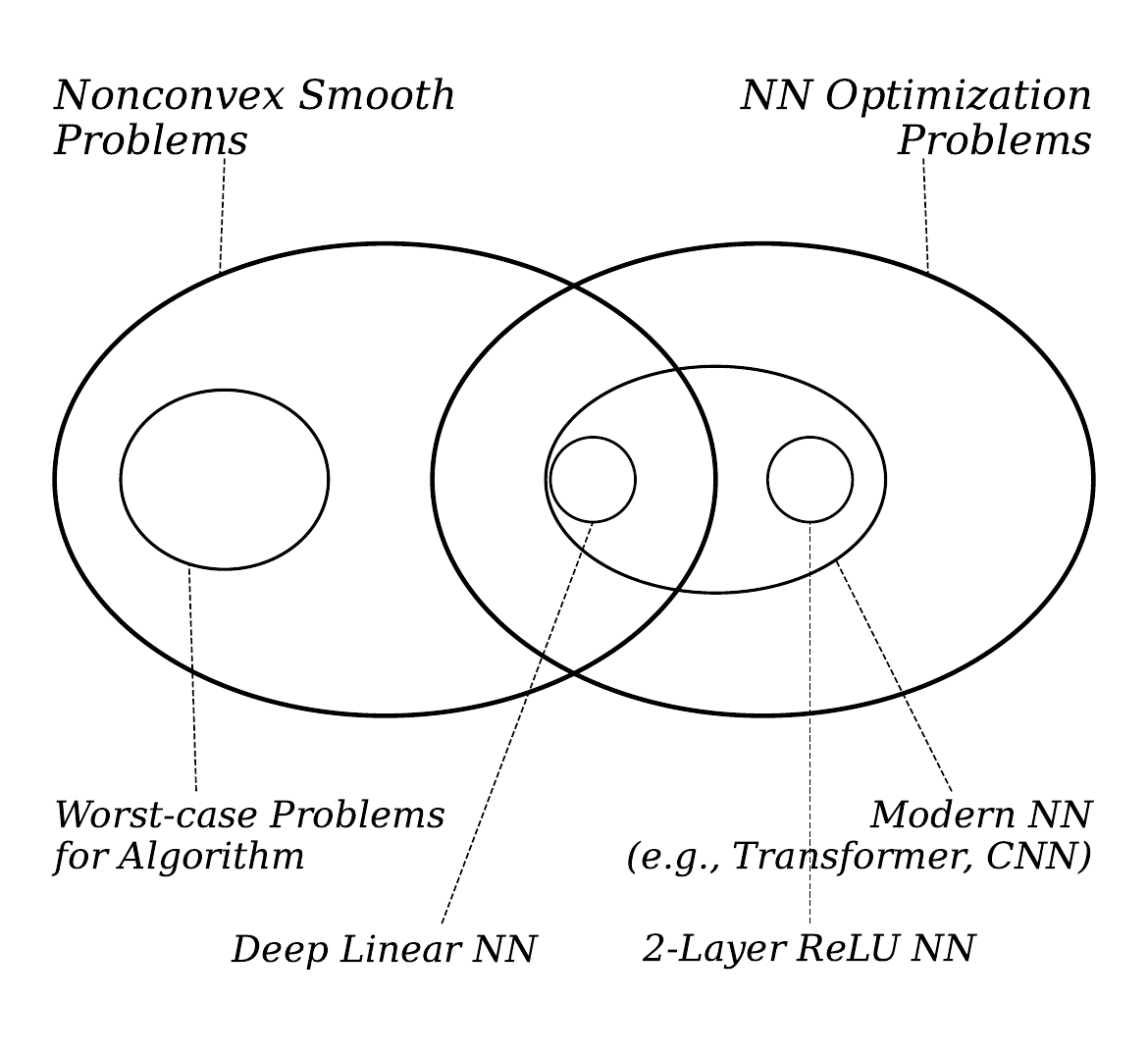}
    \vspace{-25pt}
   \caption{{Worst-case problems for an algorithm may not capture the essential properties of NN optimization problems.}}
\label{fig:problem-classes}
    \label{fig:problem-classes}
    \vspace{-12pt}
\end{wrapfigure}

\paragraph{Preconditioned Optimization Algorithms}
Preconditioned optimization methods aim to reshape the gradient by incorporating curvature information, thereby accelerating convergence in ill-conditioned problems. Early approaches such as \textsc{AdaGrad}~\cite{duchi2011adaptive} and \textsc{RMSProp}~\cite{tieleman2012lecture} employ diagonal preconditioning that adapts to the per-coordinate geometry of gradients. While computationally efficient, diagonal preconditioners fail to capture parameter correlations that naturally arise in neural network training. To address this limitation, matrix-based preconditioning methods have been developed. \textsc{K-FAC}~\cite{martens2015optimizing} approximates the Fisher information matrix using Kronecker-factored structure, exploiting the layer-wise organization of neural networks. \textsc{PSGD}~\cite{li2018preconditioned,li2022black} introduces Lie group preconditioners that maintain geometric properties during optimization. \textsc{Shampoo}~\cite{gupta2018shampoo} generalizes preconditioning to tensor spaces, maintaining separate preconditioners for each dimension through Kronecker factorization. Recent work has further improved upon \textsc{Shampoo}, with \textsc{SOAP}~\cite{vyas2025soap} stabilizing it through Adam-style updates, while distributed implementations~\cite{shi2023distributed} enable scaling to large models. Extensions such as \textsc{K-BFGS}~\cite{ren2021kronecker} and \textsc{ASGO}~\cite{an2025asgo} explore sparse or low-rank updates to reduce memory overhead. More recently, \textsc{Muon}~\cite{jordan2024muon,liu2025muon} employs orthogonalization via Newton-Schulz iteration as a form of preconditioning for matrix parameters. Several variants have emerged~\cite{si2025adamuon, li2025normuon, liu2025cosmos, wen2025fantastic, pang2026htmuonimprovingmuonheavytailed}, including \textsc{AdaMuon}~\cite{si2025adamuon} which combines \textsc{Muon} with Adam-style adaptivity, and \textsc{COSMOS}~\cite{liu2025cosmos} which introduces hybrid mechanisms for memory-efficient training. Studies on full Gauss-Newton methods~\cite{abreu2025fullgn} demonstrate that complete second-order information can substantially improve convergence, motivating the search for practical approximations that balance computational cost with optimization effectiveness.

\paragraph{Hessian Properties of Neural Networks}

Understanding the structure of the Hessian matrix is crucial for designing effective optimization algorithms, as the geometric properties of loss landscapes strongly influence training dynamics~\citep{li2018visualizing}. Early spectral analysis~\citep{sagun2016eigenvalues,sagun2017empirical,deng2026suspiciousalignmentsgdfinegrained,deng2026depthdataanalysishessian} revealed that neural network Hessians exhibit a characteristic eigenvalue spectrum: a bulk of near-zero eigenvalues with a small number of isolated large outliers. Subsequent work~\citep{ghorbani2019investigation} observed that gradients predominantly align with these outlier eigenvectors during training. \citet{wu2020dissecting} further demonstrated that layer-wise Hessians can be approximated using Kronecker factorization, explaining their persistent low-rank structure. Theoretical analyses~\citep{singh2021analytic,liao2021hessian} have provided rigorous explanations for these phenomena, deriving exact formulas for Hessian rank and connecting eigenvalue structure to data properties. Most relevant to our work, \citet{zhang2024transformers} made a significant discovery: the layer-wise Hessian of Transformers exhibits row-wise block-diagonal dominance, where diagonal blocks (corresponding to within-row parameter interactions) have significantly larger magnitudes than off-diagonal blocks (cross-row interactions). This observation has been further investigated by \citet{dong2025towards}, who provide theoretical characterizations of this structured dominance pattern. This row-wise block structure directly motivates our algorithm design, suggesting that row-level preconditioning may suffice to capture essential curvature information while maintaining computational efficiency.

\paragraph{Convergence Analysis of Adaptive Algorithms}

Theoretical understanding of adaptive optimization algorithms in non-convex settings has advanced significantly in recent years. For first-order adaptive methods, \citet{chen2022towards} established convergence guarantees for \textsc{Adam} in the non-convex setting, while \citet{li2024d} analyzed \textsc{RMSProp} and its momentum extension, proving $O(\sqrt{d}T^{1/4})$ convergence rates measured in $\ell_1$ norm. A key recent development is the recognition that different optimizers achieve provable advantages under specific geometric structures. \citet{xie2025adam} demonstrate that \textsc{Adam} exploits $\|\cdot\|_{\ell_{\infty}}$-smoothness geometry, achieving improved convergence when measured in the dual $\|\cdot\|_{\ell_1}$ norm. This geometry-dependent analysis has been extended to matrix optimization: \citet{shen2025convergence} and \citet{kim2026convergencemuonnewtonschulz} establish convergence of \textsc{Muon} under nuclear norm smoothness, showing $O(m)$ complexity compared to the $O(m^2)$ complexity under Frobenius smoothness. These results reveal that matching the optimizer structure to the problem geometry yields substantial complexity improvements beyond what standard Euclidean or Frobenius analysis would suggest. Information-theoretic lower bounds~\cite{arjevani2023lower} establish that $\epsilon^{-4}$ sample complexity is optimal for finding $\epsilon$-stationary points in the non-convex stochastic setting, providing fundamental limits for algorithm design.

\section{Method}
\subsection{\textsc{RMNP} Preconditioner}
\label{sec_preconditioner_in}
Recent work reveals that layer-wise Hessians of Transformers exhibit row-wise block-diagonal dominance \citep{zhang2024transformers}. As illustrated in Figure \ref{fig:preconditioner_comparison} (left), diagonal blocks—corresponding to interactions among parameters within the same row—have significantly larger magnitudes than off-diagonal blocks formed by cross-row interactions. This empirical finding is theoretically proven by \citet{dong2025towards} under specific configurations. Under above condition, the effective curvature of the loss is primarily concentrated on these diagonal blocks.

Preconditioning can be interpreted as correcting the descent direction within an ill-conditioned loss landscape according to the orientation and scale of the landscape's curvature. Adjustments utilizing the inverse Hessian are regarded as the optimal preconditioner under a quadratic approximation. Meanwhile, the Muon orthogonalized update can be understood as a specific preconditioning method that relies on the outer product of momentum and delivers highly favorable empirical results. By Lemma 4 in \citet{gupta2018shampoo}, the \textsc{Muon} preconditioner can be characterized in the following form:
\begin{equation}
    H_{\text{MUON}} = (V_t V_t^T)^{\frac{1}{2}} \otimes I_n
\end{equation}
where $V_t \in \mathbb{R}^{m \times n}$ denotes the momentum matrix at training step $t$, with $m = d_{\text{out}}$ and $n = d_{\text{in}}$ following the convention of \textsc{Muon}~\citep{jordan2024muon}; without loss of generality we assume $m \le n$, otherwise the same analysis applies to $V_t^\top$.

Building on these observations \citep{zhang2024transformers}, we hypothesize that the dominant curvature information resides in the row-wise diagonal blocks, while cross-row interactions contribute negligibly. This motivates approximating the preconditioner by retaining only diagonal blocks and zeroing out off-diagonal blocks, as shown in Figure \ref{fig:preconditioner_comparison}, yielding the \textsc{RMNP} preconditioner:
\begin{equation}
    H_{\text{RMNP}} = \left(\text{diag}(V_t V_t^T)\right)^{\frac{1}{2}} \otimes I_n
\end{equation}
where $\text{diag}(\cdot)$ extracts diagonal elements to form a diagonal matrix: $[\text{diag}(M)]_{ii} = M_{ii}$ and $[\text{diag}(M)]_{ij} = 0$ for $i \neq j$. This structure preserves only the row-wise blocks because $(V_t V_t^T)_{ii}$ captures interactions within the $i$-th row of $V_t$, while the Kronecker product $\text{diag}(\cdot) \otimes I_n$ applies this scaling independently to each row.

The resulting preconditioned update $\text{diag}(V_t V_t^T)^{-\frac{1}{2}} V_t$ reduces to row-wise $\ell_2$ normalization:
\begin{equation}
    \left[ \left(\text{diag}(V_t V_t^T)\right)^{-\frac{1}{2}} V_t \right]_{i,:} = \frac{V_{t,i:}}{\sqrt{(V_t V_t^T)_{ii}}} = \frac{V_{t,i:}}{\|V_{t,i:}\|_{\ell_2}}
\end{equation}
where $V_{t,i:}$ denotes the $i$-th row of $V_t$ and $\|V_{t,i:}\|_{\ell_2} = \sqrt{(V_t V_t^T)_{ii}}$. This dramatically reduces computational complexity compared to \textsc{Muon}'s Newton-Schulz iteration. The above conjecture is equivalent to implying that the Gram matrix $V_t V_t^\top$ exhibits a certain diagonal dominance property. In the following subsection, we empirically verify this property of the \textsc{Muon} preconditioner.

\subsection{Analysis of \textsc{Muon} Preconditioner}
\label{subsec:diagonal_dominance}
To investigate the properties of the preconditioner, we analyze the Gram matrix $V_t V_t^T \in \mathbb{R}^{m \times m}$, constructed from the matrix parameter $V_t \in \mathbb{R}^{m \times n}$ at step $t$. We define a row-wise metric $r_i$ to quantify the ratio of the diagonal element to the average magnitude of off-diagonal entries in the $i$-th row:
\begin{equation}
r_i \triangleq \frac{(V_t V_t^T)_{ii}}{\frac{1}{m-1}\sum_{j \neq i} \left| (V_t V_t^T)_{ij} \right|} = \frac{\|V_{t,i:}\|_2^2}{\frac{1}{m-1}\sum_{j \neq i} \left| V_{t,i:} (V_{t,j:})^T \right|}.
\end{equation}
where $(V_t V_t^T)_{ij}$ denotes the entry at row $i$ and column $j$ of the Gram matrix. Based on these row-wise ratios, we introduce the following three aggregate metrics to evaluate the global diagonal dominance across all rows of the matrix.
We define \emph{average diagonal dominance ratio} ($r_{\text{avg}}$),  \emph{minimum diagonal dominance ratio} ($r_{\text{min}}$) and  \emph{maximum diagonal dominance ratio} ($r_{\text{max}}$) as follows:
\begin{equation}
r_{\mathrm{avg}} = \frac{1}{m}\sum_{i=1}^{m} r_i, \quad r_{\min} = \min_{i\in\{1,\ldots,m\}} r_i, \quad r_{\max} = \max_{i\in\{1,\ldots,m\}} r_i.
\end{equation}
Regarding the interpretation, values of $r_i > 1$ indicate that the diagonal element dominates the average off-diagonal magnitude in row $i$, suggesting stronger diagonal dominance. Values approaching $1$ suggest that the diagonal element is comparable to the average off-diagonal magnitude, while values significantly greater than $1$ indicate that $V_t V_t^T$ closely approximates a diagonal matrix structure.
\begin{figure*}[!htb]
    \centering
    \includegraphics[width=1.00\linewidth]{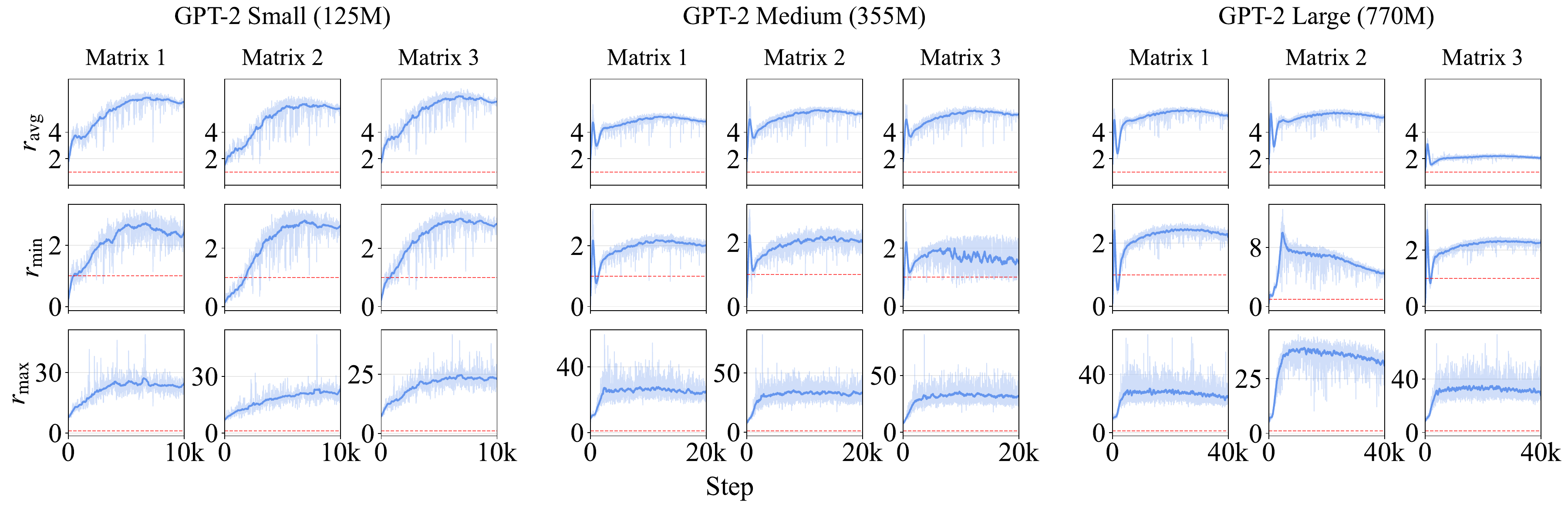}
        \caption{Per-parameter diagonal dominance ratios $r_{\text{avg}}$, $r_{\min}$, $r_{\max}$ (rows) for three representative matrix parameters (columns) during GPT-2 Small (125M), GPT-2 Medium (355M) and GPT-2 Large (770M) pre-training. Transparent curves: raw values; solid curves: smoothed with window size 50. Red dashed line: $y=1$ threshold.}
\label{fig:dominance_curves}
\end{figure*}
To validate our method empirically, we tracked these metrics across all matrix parameters of GPT-2 Small (125M), GPT-2 Medium (355M), and GPT-2 Large (770M) during training. We visualize the evolution of these metrics for 3 randomly selected matrices in Figure~\ref{fig:dominance_curves}. Furthermore, we report the global statistics ($\overline{r}_{\text{avg}}, \overline{r}_{\text{min}}, \overline{r}_{\text{max}}$), which average these three statistics across all matrix parameters in the network, in Figure~\ref{fig:global_dominance_curves}. The experimental setup follows that of the previous GPT-2 experiments on OpenWebText; see Appendix~\ref{appendix:hyperparams} for training hyperparameters and Appendix~\ref{appendix:diagonal_dominance_setup} for implementation details. Additional LLaMA per-parameter results are provided in Appendix~\ref{appendix:diagonal_dominance_setup}; see Figure~\ref{fig:Llama_dominance_curves_appendix}.

\begin{figure*}[!htb]
    \centering
    \includegraphics[width=1.00\linewidth]{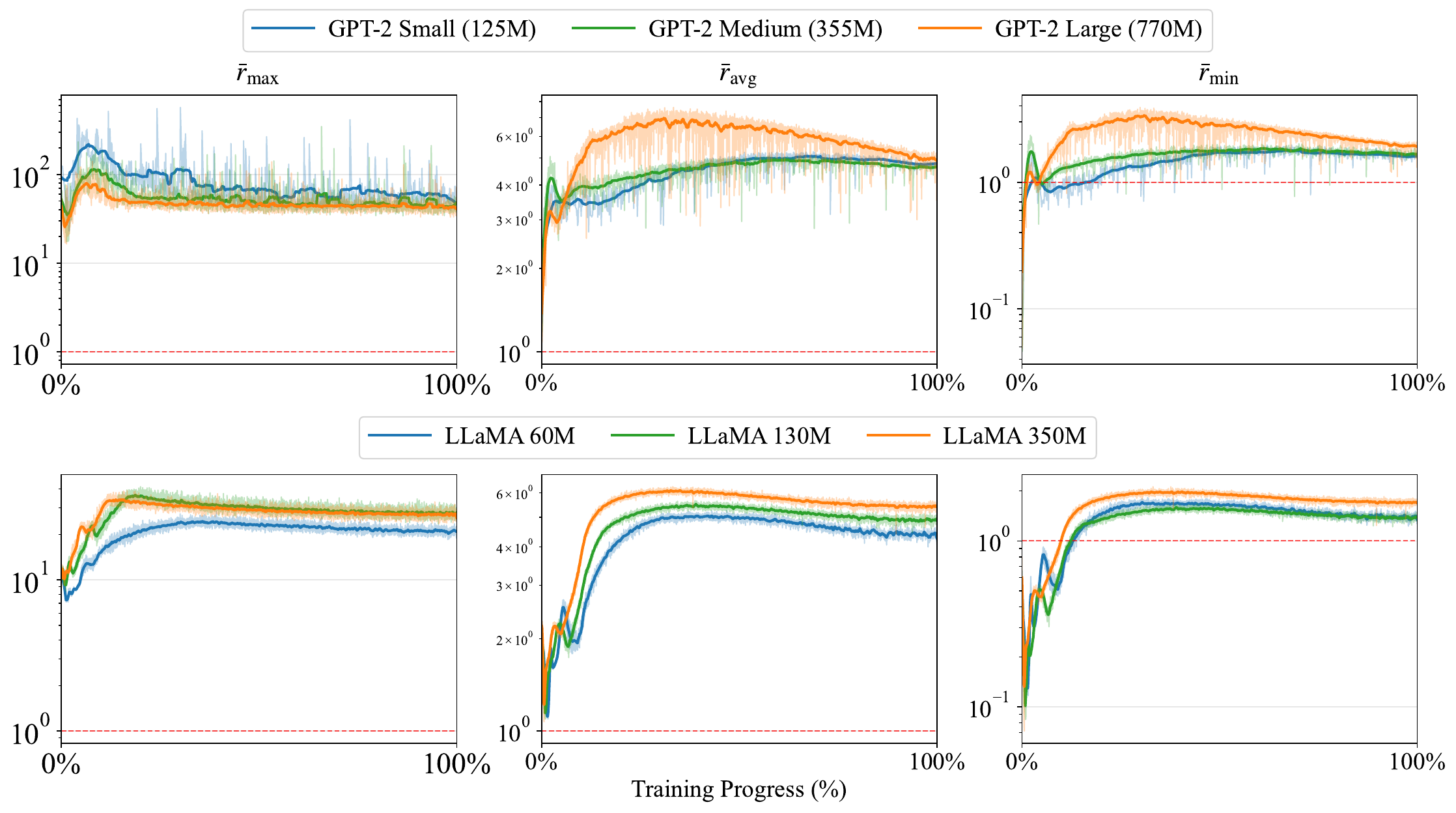}
        \caption{Global diagonal dominance ratios $\overline{r}_{\text{avg}}$, $\overline{r}_{\min}$, $\overline{r}_{\max}$ (columns) averaged across all matrix parameters, comparing across model scales for two architectures: GPT-2 Small (125M), Medium (355M), and Large (770M) pre-trained on OpenWebText (top row), and LLaMA 60M, 130M, and 350M pre-trained on C4 (bottom row). The x-axis is rescaled to the relative training progress (\%) so that all model scales within a row align on a shared horizontal range; the y-axis is in log scale. Transparent curves: raw values; solid curves: smoothed with window size 50. Red dashed line: $y=1$ threshold. For both architectures, the metrics quickly rise above 1 after warm-up and remain mostly above 1, and the magnitude of $\overline{r}_{\text{avg}}$, $\overline{r}_{\min}$, $\overline{r}_{\max}$ tends to grow with model scale, confirming strong and progressively more pronounced diagonal dominance throughout training on both Transformer families.}
\label{fig:global_dominance_curves}
\end{figure*}

As illustrated in Figure~\ref{fig:dominance_curves}, the three representative matrices exhibit strong diagonal dominance, with all three ratio metrics consistently exceeding the baseline of 1 throughout training. For these matrices in GPT-2 Small, $r_{\text{min}}$ stabilizes above 2, $r_{\text{avg}}$ exceeds 5, and $r_{\text{max}}$ reaches approximately 25. Furthermore, the global statistics across all matrices show a similar trend. As shown in Figure~\ref{fig:global_dominance_curves}, the global statistics in GPT-2 Small stabilize at levels indicative of strong diagonal dominance: $\overline{r}_{\text{min}}$ is approximately 1.6, $\overline{r}_{\text{avg}}$ is around 4.9, and $\overline{r}_{\text{max}}$ reaches about 60. It is also worth noting that, in the GPT-2 Medium and Large regimes, the preconditioner exhibits increasingly pronounced diagonal dominance as model size grows. This confirms that the observed diagonal dominance is not an isolated phenomenon but a systematic property of the training dynamics.

\vspace{-4pt}


\section{Main Experimental Results}
In this section, we demonstrate that \textsc{RMNP} achieves competitive optimization performance while maintaining high computational efficiency. We first show that \textsc{RMNP} reduces the preconditioning computational cost by an order of magnitude compared to \textsc{Muon}, demonstrating its scalability advantages. We then evaluate \textsc{RMNP} against \textsc{AdamW} and \textsc{Muon}, two prevalent optimizers for training large language models, on the GPT-2 and LLaMA model series. GPT-2 models are trained on OpenWebText~\cite{Gokaslan2019OpenWeb} \fwedu{and FineWeb-Edu-100B~\cite{penedo2024finewebdatasetsdecantingweb}}, while LLaMA models are trained on C4~\cite{raffel2020exploring}.

\subsection{Experimental Setup}
\label{sec:exp_setup}

\paragraph{\textsc{Muon}}
Following the setup in \citet{jordan2024muon,liu2025muon}, we employ a mixed update strategy where matrix parameters are optimized using \textsc{Muon} and non-matrix parameters using \textsc{AdamW}. We introduce two distinct learning rate hyperparameters, $\text{lr}_{\text{AdamW}}$ and $\text{lr}_{\text{Matrix}}$, both following a cosine annealing schedule with a 10\% warmup period.

\paragraph{\textsc{RMNP}}
For \textsc{RMNP}, we align our experimental setup with the \textsc{Muon} protocol described above. We employ an almost identical mixed update strategy, applying \textsc{RMNP} to matrix parameters and \textsc{AdamW} to non-matrix parameters. Similarly, we utilize two learning rates, $\text{lr}_{\text{AdamW}}$ and $\text{lr}_{\text{Matrix}}$, both subject to a cosine annealing schedule with a 10\% warmup. Consistent with the baseline settings, we fix the \textsc{AdamW} hyperparameters ($\beta=(0.9, 0.95)$, weight decay $0.1$) and exclusively tune the learning rate for the matrix optimizer, $\text{lr}_{\text{Matrix}}$, during the search process.

\paragraph{\textsc{AdamW}}
For the \textsc{AdamW} setup, we follow the standard setup in \citet{yuan2025mars} for training GPT-2, and \citet{he2025alphadecay} for LLaMA. We set $\beta=(0.9, 0.95)$ and weight decay $0.1$, and a cosine annealing schedule with 10\% warm up which consistent with the \textsc{AdamW} configuration used in the mixed update strategy above.

\paragraph{GPT-2 Pre-Training on OpenWebText}
Experiments on GPT-2 are conducted based on the implementation of ~\citet{yuan2025mars}, using the OpenWebText dataset~\cite{Gokaslan2019OpenWeb} and the GPT-2 tokenizer. We pretrain three scales of GPT-2 models: small (125M parameters), medium (355M parameters), and large (770M parameters). For model configurations, we set the dropout rate to 0.0 and disable biases.  Training hyperparameters are listed in Tables~\ref{tab:gpt2_config} and~\ref{tab:llama_config} in Appendix~\ref{appendix:model_config}. \fwedu{We also evaluate on FineWeb-Edu-100B~\cite{penedo2024finewebdatasetsdecantingweb,karpathy2024finewebedu100b} across four GPT-2 scales (Small, Medium, Large, and XLarge (1.5B)); see Appendix~\ref{appendix:finewebedu} for configurations and results.}

\paragraph{LLaMA Pre-Training on C4}
Experiments on LLaMA are conducted on the C4 dataset~\citep{raffel2020exploring}. We pretrain four scales of LLaMA models: LLaMA-60M, LLaMA-130M, LLaMA-350M, and LLaMA-1B. Training hyperparameters are listed in Table~\ref{tab:llama_config} in Appendix~\ref{appendix:model_config}.

\begin{table}[h]
\centering
\caption{Efficiency comparison between \textsc{Muon}  and \textsc{RMNP}'s preconditioning cost  on GPT-2 models. Time measured over 100 steps with batch size 16 on a single RTX Pro 6000 GPU.}
\begin{tabular}{lccccc}
\toprule
\multirow{2}{*}{Size} & \multicolumn{2}{c}{Time Cost (s)}  &\multirow{2}{*}{Speedup $(\times)$}\\
\cmidrule(lr){2-3}
& \textsc{Muon} & \textsc{RMNP} & \\
\midrule
60M  & 1.480   & \textbf{0.115}& 12.9 \\
125M & 2.975  & \textbf{0.201}& 14.8\\
200M & 4.140   & \textbf{0.260} & 15.9\\
355M & 7.380   & \textbf{0.401}& 18.4\\
500M & 15.720  & \textbf{0.462} &34.0\\
770M & 27.070  & \textbf{0.611} & 44.3\\
1.3B & 30.570  & \textbf{0.783} & 39.0\\
1.5B & 36.650  & \textbf{0.855} &42.9\\
\bottomrule
\end{tabular}
\label{tab:efficiency}
\end{table}

\subsection{Preconditioning Time Cost}
\label{sec:efficiency}
Since \textsc{RMNP} and \textsc{Muon} primarily differ in their choice of preconditioner, where \textsc{Muon} applies Newton--Schulz orthogonalization whereas \textsc{RMNP} uses row normalization, we benchmark the preconditioner-operator overhead of \textsc{RMNP} against \textsc{Muon}. Specifically, we report the per-iteration time attributable to the preconditioner operator (\emph{Step Time}) and the cumulative time over 100 iterations (\emph{Total Time}). Experiments are run on GPT-2 models ranging from 60M to 1.5B parameters with a batch size of 16. See Appendix~\ref{app:efficiency_config} for detailed model configurations.

As shown in Table~\ref{tab:efficiency}, \textsc{RMNP} achieves significant speedup over \textsc{Muon} across all model sizes. The row normalization in \textsc{RMNP} is approximately 13--44$\times$ faster than the Newton-Schulz orthogonalization in \textsc{Muon}. This result underscores \textsc{RMNP}'s computational efficiency. More importantly, as model size grows and Newton--Schulz orthogonalization increasingly becomes the dominant bottleneck in end-to-end training throughput, \textsc{RMNP}'s lightweight preconditioner offers a more scalable alternative, indicating strong potential for training at very large scale. For example, in Table~\ref{tab:efficiency}, for GPT-2 60M, \textsc{Muon}'s preconditioning cost per 100 steps is only 1.48 seconds, and \textsc{RMNP} provides a 12.9× speedup. However, for GPT-2 1.5B, the preconditioning cost per 100 steps increases to 36.65 seconds, while \textsc{RMNP} achieves a 42.9× speedup. See Appendix~\ref{appendix:wall_clock} for detailed results including memory usage.

\subsection{Pretraining Performance}
\label{sec:pretrain_performance}
\paragraph{\textsc{RMNP} consistently outperforms \textsc{Muon} and \textsc{AdamW} in GPT-2 experiments. } As shown in Figure~\ref{fig:gpt_2_results}, across the Small, Medium, and Large settings, while efficiently reducing the preconditioner-operator overhead, \textsc{RMNP} still delivers more competitive results than both baselines in terms of evaluation perplexity: on the Small setting it improves over \textsc{Muon} by 0.04 and over \textsc{AdamW} by 1.37; on the Medium setting the improvements are 0.07 and 1.49; and on the Large setting they are 0.24 and 0.84, respectively.
 This consistent pattern suggests that \textsc{RMNP}'s efficiency gains in Table~\ref{tab:efficiency} do not come at the expense of optimization quality; instead, it preserves strong optimization behavior while reducing preconditioning overhead, yielding a favorable speed--accuracy trade-off across model scales under a standard large-model training protocol. Our GPT-2 experiments on OpenWebText match the results reported in ~\citet{yuan2025mars}. We conduct an extensive hyperparameter grid search for both \textsc{Muon} and \textsc{RMNP}; see Table~\ref{tab:small_hyperparam} and~\ref{tab:medium_hyperparam} in Appendix~\ref{hyper_search}. \fwedu{Results on FineWeb-Edu-100B further confirm this trend (Appendix~\ref{appendix:finewebedu}).} Per-step training and validation loss curves for all GPT-2 scales on both datasets are reported in Appendix~\ref{appendix:training_curves} (Figures~\ref{fig:owt_small_curves}--\ref{fig:fwedu_xlarge_curves}); the corresponding gradient clip-rate trajectories are shown in Appendix~\ref{appendix:clip_rate}. The advantage of \textsc{RMNP} also persists under a $2\times$ extended training budget (Appendix~\ref{appendix:extended_training}, Table~\ref{tab:extended_training}).

\paragraph{\textsc{RMNP} consistently outperforms \textsc{Muon} and \textsc{AdamW} in  LLaMA experiments.}
As shown in Figure~\ref{fig:llama_2_results}, \textsc{RMNP} consistently achieves comparable perplexity to \textsc{Muon} across all model sizes, while maintaining a slight performance edge.
Specifically, \textsc{RMNP} demonstrates modest improvements over the baseline: on the LLaMA-60M setting, it decreases perplexity by 0.63 compared to \textsc{Muon} and 4.33 compared to \textsc{AdamW}; on the LLaMA-130M setting, the gain is 0.28 over \textsc{Muon} and 1.10 compared to \textsc{AdamW}; and on the LLaMA-350M setting, the improvement is 0.02 over \textsc{Muon}.
This pattern suggests that \textsc{RMNP} is able to fully match the optimization quality of \textsc{Muon} without the heavy preconditioning overhead, effectively delivering efficiency gains without sacrificing performance.
It is worth noting that we perform a systematic hyperparameter grid search for both \textsc{Muon} and \textsc{RMNP}; see Table~\ref{tab:llama_60m_hyperparam},~\ref{tab:llama_130m_hyperparam} and~\ref{tab:llama_350m_hyperparam} in Appendix~\ref{hyper_search}. Per-step training and validation loss curves for all four LLaMA scales are reported in Appendix~\ref{appendix:curves_c4} (Figures~\ref{fig:c4_60m_curves}--\ref{fig:c4_1b_curves}). We also study the effect of also applying the matrix optimizer to the LM-head and embedding parameters in Appendix~\ref{appendix:lm_head_embedding_ablation} (Tables~\ref{tab:llama_60m_lmhead_hyperparam} and~\ref{tab:llama_130m_lmhead_hyperparam}); a final-perplexity summary across all settings is provided in Appendix~\ref{appendix:final_ppl_summary}.

\begin{figure}[!htb]
    \centering
    \includegraphics[width=\linewidth]{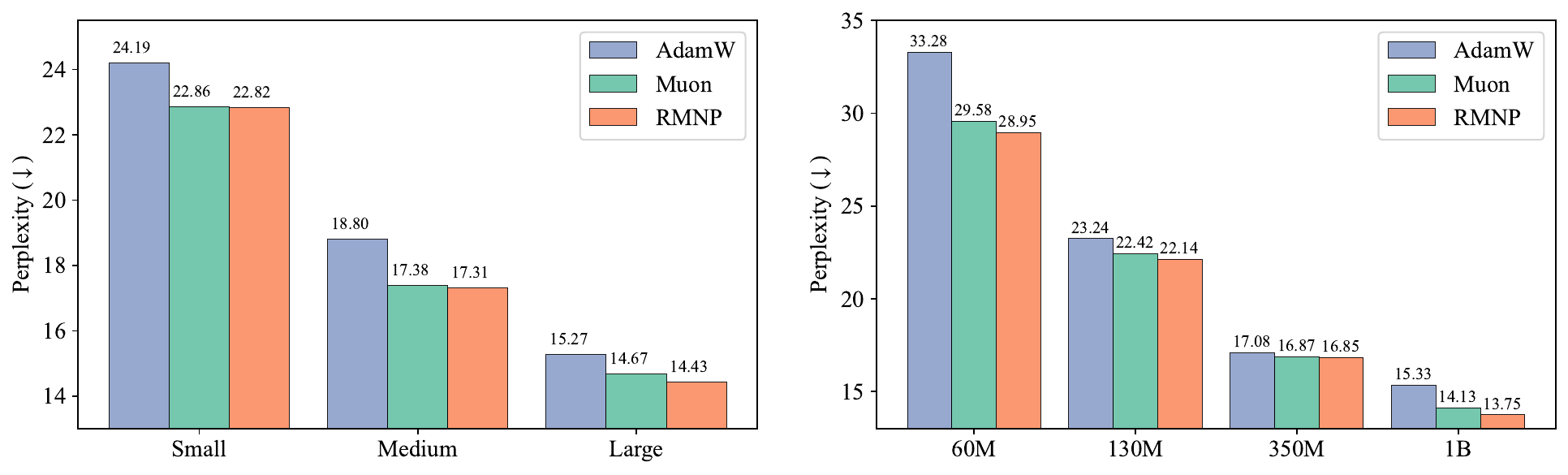}
    \caption{\textbf{Left:} Results for GPT-2 on OpenWebText: Small (125M) trained with 5B tokens; Medium (355M) trained with 10B tokens, and Large (770M) trained with 20B tokens. Numeric values are reported in Table~\ref{tab:owt_results}. \fwedu{FineWeb-Edu-100B results are in Figure~\ref{fig:fwedu_results_bar} and Table~\ref{tab:fwedu_results}.}
    \textbf{Right:} Results for LLaMA: 60M trained with 1B tokens; 130M trained with 2B tokens, 350M trained with 6B tokens, and 1B trained with 9B tokens. Numeric values are reported in Table~\ref{tab:llama_c4_results}.}
    \label{fig:gpt_2_results}
    \label{fig:llama_2_results}
\end{figure}

\section{Non-Convex Convergence}
\label{Theory_Section}

In this section, we present the convergence analysis of our proposed method under the non-convex smooth setting. Our setup is consistent with many existing analyses for adaptive algorithms \cite{chen2022towards,xie2025adam,li2024d,shen2025convergence,kim2026convergencemuonnewtonschulz}, assuming only the smoothness of the loss function, alongside unbiased stochastic gradients and bounded second moments, as detailed in Section \ref{sec_assumption}.

Recent work reveals that optimizers can achieve provable advantages under specific geometric structures beyond standard $\ell_2$ or Frobenius smoothness. For instance, \citet{xie2025adam} discuss benefits of \textsc{Adam} under $\|\cdot\|_{\ell_{\infty}}$-smoothness with convergence measured in $\|\cdot\|_{\ell_1}$, while \citet{shen2025convergence,kim2026convergencemuonnewtonschulz} establish advantages of \textsc{Muon} under $\|\cdot\|_2$-smoothness with convergence measured in nuclear norm. Similarly, we identify the geometric structure under which \textsc{RMNP} achieves provable benefits. We establish three convergence results: under the standard $\|\cdot\|_F$-smoothness assumption, we prove convergence in both the Frobenius norm sense (Theorem~\ref{thm:fro-convergence}) and the $\|\cdot\|_{1,2}$ norm sense (Theorem~\ref{thm:12-fro-convergence}). More importantly, under the $\|\cdot\|_{\infty,2}$-smoothness assumption, we establish improved convergence guarantees in the $\|\cdot\|_{1,2}$ norm sense (Theorem~\ref{thm:inf2-convergence}), revealing that \textsc{RMNP} similarly benefits from its matched geometric structure.

\subsection{Notation}
Let $W \in \mathbb{R}^{m \times n}$ denote the parameter matrix, where $W_{i,:} \in \mathbb{R}^n$ denotes the $i$-th row. The matrix inner product is $\langle Z, W \rangle = \text{Tr}(Z^\top W)$. We use the Frobenius norm $\|W\|_F = \sqrt{\sum_{i,j} W_{i,j}^2}$, the mixed norm $\|W\|_{1,2} = \sum_{i=1}^m \|W_{i,:}\|_2$, and the norm $\|W\|_{\infty,2} = \max_{i=1,\ldots,m} \|W_{i,:}\|_2$. These satisfy the duality $|\langle A, B \rangle| \leq \|A\|_{1,2} \|B\|_{\infty,2}$. We use $\mathbb{E}[\cdot]$ to denote the expectation and $\mathbb{E}_t[\cdot \mid \mathcal{F}_{t-1}]$ to denote the conditional expectation given $\mathcal{F}_{t-1}$. Without loss of generality, we assume $m \leq n$; otherwise the same analysis applies to $V_t^\top$.

\subsection{Assumptions}

\label{sec_assumption}
\begin{assumption}[Lipschitz Gradient]\label{assump:lipschitz}
The gradient of $f: \mathbb{R}^{m \times n} \to \mathbb{R}$ is Lipschitz continuous in one of the following norms:

\textbf{(a)} \textit{Frobenius norm:} There exists $L_F > 0$ such that for all $W, W' \in \mathbb{R}^{m \times n}$,
$$\|\nabla f(W) - \nabla f(W')\|_F \leq L_F\|W - W'\|_F.$$

\textbf{(b)} \textit{$(1,2)$-norm with respect to $(\infty,2)$-norm:} There exists $L_{\infty,2} > 0$ such that for all $W, W' \in \mathbb{R}^{m \times n}$,
$$\|\nabla f(W) - \nabla f(W')\|_{(1,2)} \leq L_{\infty,2}\|W - W'\|_{\infty,2}.$$
\end{assumption}

\begin{assumption}[Unbiased Gradient Estimator]\label{assump:unbiased}
For all $t$ and $W_t$,
$$\mathbb{E}_t[G_t \mid \mathcal{F}_{t-1}] = \mathbb{E}_t[\nabla f(W_t; \xi^t) \mid \mathcal{F}_{t-1}] = \nabla f(W_t).$$
\end{assumption}

\begin{assumption}[Bounded Gradient Variance]\label{assump:variance}
There exists a constant $\sigma > 0$ such that for all $t$ and $W_t$,
$$\mathbb{E}_t[\|G_t - \nabla f(W_t)\|_F^2 \mid \mathcal{F}_{t-1}] \leq \frac{\sigma^2}{B},$$
where $B$ is the batch size (i.e., the number of samples used to compute $G_t$).
\end{assumption}

\begin{assumption}[Lower Bound]\label{assump:lower}
$f$ is bounded below with $f^* = \inf_{W} f(W)$. Define $\Delta = f(W_0) - f^*$.
\end{assumption}

\subsection{Main Results}

We now present our main theoretical results, which establish convergence guarantees for \textsc{RMNP} under different smoothness assumptions and convergence criteria. Our analysis reveals how the choice of matrix norms---both in the smoothness assumption and in the convergence measure---affects the sample complexity.
\begin{theorem}[$\|\cdot\|_F$- Lipschitz]\label{thm:fro-convergence}
Under Assumptions~\ref{assump:lipschitz}(a), \ref{assump:unbiased}, \ref{assump:variance}, and \ref{assump:lower}, if Algorithm \ref{algoRMNP} uses constant $\eta_t = \eta$ and momentum $\beta \in [0,1)$, then
\begin{equation}\label{eq:fro-main}
\frac{1}{T}\sum_{t=1}^{T}\mathbb{E}\left[\|\nabla f(W_t)\|_F\right] \leq \frac{\Delta}{T\eta} + (\sqrt{m}+1)\left[\left(1-\frac{1}{T}\right)\frac{L_F\eta\sqrt{m}\beta}{1-\beta} + \frac{\sigma}{\sqrt{B}}\sqrt{\frac{1-\beta}{1+\beta}}\right] + \frac{L_F\eta m}{2}.
\end{equation}
\end{theorem}

\begin{remark}[Complexity for Theorem~\ref{thm:fro-convergence}]
If we set $B=1$, $\eta=\sqrt{\frac{(1-\beta)\Delta}{L_F m T}}$, and
$1-\beta=\min\left\{\frac{\sqrt{L_F\Delta}}{(\sqrt{m}+1)\sigma\sqrt{T}}, 1\right\},$
then the bound in \eqref{eq:fro-main} yields
\begin{equation*}
\frac{1}{T}\sum_{t=1}^{T}\mathbb{E}\left[\|\nabla f(W_t)\|_F\right] \leq O\left(\sqrt[4]{\frac{m^2 L_F \Delta \sigma^2}{T}} + \sqrt{\frac{L_F m\Delta}{T}} + \frac{m\sigma^2}{\sqrt{L_F\Delta T}}\right).
\end{equation*}
Thus, we can find an $\epsilon$-stationary point (in Frobenius norm) of $f$ with a complexity of $O(m^2 L_F \sigma^2 \Delta \epsilon^{-4})$, exhibiting an $O(m^2)$ dimension dependence.
\end{remark}

The detailed proof of Theorem \ref{thm:fro-convergence} can be found in Appendix \ref{sec_proofthem}. While Theorem~\ref{thm:fro-convergence} establishes convergence in the Frobenius norm, matrix optimization problems often involve alternative matrix norms that better capture the underlying structure. Our next result analyzes convergence in the $\|\cdot\|_{1,2}$ norm under the same Frobenius smoothness assumption, demonstrating that \textsc{RMNP} achieves comparable complexity guarantees across different convergence measures.

\begin{theorem}[$\|\cdot\|_{1,2}$-Convergence under $\|\cdot\|_F$-Lipschitz]\label{thm:12-fro-convergence}
Under Assumptions~\ref{assump:lipschitz}(a), \ref{assump:unbiased}, \ref{assump:variance}, and \ref{assump:lower}, if Algorithm \ref{algoRMNP} uses constant $\eta_t = \eta$ and momentum $\beta \in [0,1)$, then
\begin{equation}\label{eq:12-fro-main}
\frac{1}{T}\sum_{t=1}^{T}\mathbb{E}\left[\|\nabla f(W_t)\|_{1,2}\right] \leq \frac{\Delta}{T\eta} + 2\left[\left(1-\frac{1}{T}\right)\frac{L_F\eta m\beta}{1-\beta} + \frac{\sqrt{m}\sigma}{\sqrt{B}}\sqrt{\frac{1-\beta}{1+\beta}}\right] + \frac{L_F\eta m}{2}.
\end{equation}
\end{theorem}

\begin{remark}[Complexity for Theorem~\ref{thm:12-fro-convergence}]
If we set $B=1$, $\eta=\sqrt{\frac{(1-\beta)\Delta}{L_F m T}}$, and
$$1-\beta=\min\left\{\frac{\sqrt{L_F\Delta}}{2\sqrt{m}\sigma\sqrt{T}}, 1\right\},$$
then the bound in \eqref{eq:12-fro-main} yields
\begin{equation*}
\frac{1}{T}\sum_{t=1}^{T}\mathbb{E}\left[\|\nabla f(W_t)\|_{1,2}\right] \leq O\left(\sqrt[4]{\frac{m^2 L_F \Delta \sigma^2}{T}} + \sqrt{\frac{L_F m\Delta}{T}} + \frac{m\sigma^2}{\sqrt{L_F\Delta T}}\right).
\end{equation*}
Thus, we can find an $\epsilon$-stationary point (in $\|\cdot\|_{1,2}$ norm) of $f$ with a complexity of $O(m^2 L_F \sigma^2 \Delta \epsilon^{-4})$, exhibiting an $O(m^2)$ dimension dependence.
\end{remark}

The detailed proof of Theorem \ref{thm:12-fro-convergence} can be found in Appendix \ref{sec_proofthem}. The preceding results demonstrate that under Frobenius smoothness, both convergence measures achieve $O(m^2)$ complexity. However, when the objective function exhibits a different geometric structure---specifically, when the gradient is Lipschitz continuous with respect to the $\|\cdot\|_{\infty,2}$ norm---\textsc{RMNP}'s row normalization operation can exploit this structure more effectively. Our final result establishes a significantly improved complexity bound in this setting.

\begin{theorem}[$\|\cdot\|_{1,2}$-Lipschitz]\label{thm:inf2-convergence}
Under Assumptions~\ref{assump:lipschitz}(b), \ref{assump:unbiased}, \ref{assump:variance}, and \ref{assump:lower}, if Algorithm~\ref{algoRMNP} uses constant $\eta_t = \eta$ and momentum $\beta \in [0,1)$, then
\begin{equation}\label{eq:inf2-main}
\frac{1}{T}\sum_{t=1}^{T}\mathbb{E}\left[\|\nabla f(W_t)\|_{1,2}\right] \leq \frac{\Delta}{T\eta} + 2\left[\left(1-\frac{1}{T}\right)\frac{L_{\infty,2}\eta\beta}{1-\beta} + \frac{\sqrt{m}\sigma}{\sqrt{B}}\sqrt{\frac{1-\beta}{1+\beta}}\right] + \frac{L_{\infty,2}\eta}{2}.
\end{equation}
\end{theorem}

\begin{remark}[Complexity for Theorem~\ref{thm:inf2-convergence}]
If we set $B=1$, $\eta=\sqrt{\frac{(1-\beta)\Delta}{L_{\infty,2} T}}$, and
$1-\beta=\min\left\{\frac{\sqrt{L_{\infty,2}\Delta}}{2\sqrt{m}\sigma\sqrt{T}}, 1\right\},$
then the bound in \eqref{eq:inf2-main} yields
\begin{equation*}
\frac{1}{T}\sum_{t=1}^{T}\mathbb{E}\left[\|\nabla f(W_t)\|_{1,2}\right] \leq O\left(\sqrt[4]{\frac{m L_{\infty,2} \Delta \sigma^2}{T}} + \sqrt{\frac{L_{\infty,2}\Delta}{T}} + \frac{\sqrt{m}\sigma^2}{\sqrt{L_{\infty,2}\Delta T}}\right).
\end{equation*}
Thus, we can find an $\epsilon$-stationary point (in $\|\cdot\|_{1,2}$ norm) of $f$ with a complexity of $O(m L_{\infty,2} \sigma^2 \Delta \epsilon^{-4})$.
\end{remark}
The detailed proof of Theorem \ref{thm:inf2-convergence} can be found in Appendix \ref{sec_proofthem}.

\subsection{Comparison with Related Work}
We now compare our theoretical results with recent work on \textsc{Muon} optimizers, as summarized in Table~\ref{tab:comparison}. Similar to \textsc{Muon}, \textsc{RMNP} demonstrates geometry-dependent advantages: different smoothness assumptions lead to different convergence guarantees. Under Frobenius norm smoothness (Assumption~\ref{assump:lipschitz}(a)), both Theorem~\ref{thm:fro-convergence} and Theorem~\ref{thm:12-fro-convergence} achieve a sample complexity of $O(m^2 L_F \sigma^2 \Delta \epsilon^{-4})$, matching the recent results for \textsc{Muon}~\cite{shen2025convergence}. More importantly, under the $\|\cdot\|_{\infty,2}$-smoothness assumption (Assumption~\ref{assump:lipschitz}(b)), Theorem~\ref{thm:inf2-convergence} achieves an improved complexity of $O(m L_{\infty,2} \sigma^2 \Delta \epsilon^{-4})$, representing a quadratic improvement in dimension dependence from $O(m^2)$ to $O(m)$. This mirrors \textsc{Muon}'s improvement under nuclear norm smoothness, where convergence of $\|\nabla f\|_*$ also achieves $O(m)$ complexity~\cite{shen2025convergence}. Although the geometric structures differ---\textsc{Muon} exploits nuclear norm geometry while \textsc{RMNP} exploits $\|\cdot\|_{1,2}$ geometry---both methods achieve the same $O(m)$ dimension dependence in their respective favorable settings. This improvement stems from \textsc{RMNP}'s row normalization operation, which naturally aligns with the row-wise structure present in the $\|\cdot\|_{1,2}$ geometry.

\section{Conclusion}
In this paper, we introduced \textbf{RMNP} (Row Momentum Normalized Preconditioning), an efficient optimizer that significantly advances preconditioned adaptive methods for deep neural network training. Motivated by the diagonal block dominance structure observed in Transformer Hessians, RMNP replaces the Newton-Schulz iteration in \textsc{Muon} with a simple row-wise $\ell_2$ normalization operation, reducing the per-iteration computational complexity from $\mathcal{O}(mn\cdot\min(m,n))$ to $\mathcal{O}(mn)$---an order of magnitude improvement.

Our contributions span three key dimensions. \textbf{Algorithmically}, RMNP achieves substantial efficiency gains, delivering 13--44$\times$ speedup \textbf{on the preconditioning process} over \textsc{Muon} across model scales from 60M to 1.5B parameters while maintaining comparable memory usage. \textbf{Empirically}, extensive experiments on GPT-2 (125M, 355M, 770M, and 1.5B on FineWeb-Edu-100B) and LLaMA (60M, 130M, 350M, and 1B) demonstrate that RMNP consistently matches or outperforms both \textsc{Muon} and AdamW in terms of final performance. Our empirical analysis validates the diagonal dominance property of the \textsc{Muon} preconditioner, providing strong support for RMNP's design principle. We also provide practical hyperparameter recommendations, showing that $\text{lr}_{\text{Matrix}}$ is the primary factor influencing performance. \textbf{Theoretically}, we establish rigorous convergence guarantees in the non-convex setting that match recent results for \textsc{Muon} optimizers. As summarized in Table~\ref{tab:comparison}, RMNP achieves sample complexity of $O(m^2 L_F \sigma^2 \Delta \epsilon^{-4})$ under Frobenius smoothness and an improved $O(m L_{\infty,2} \sigma^2 \Delta \epsilon^{-4})$ under $\|\cdot\|_{\infty,2}$-smoothness, and both result's complexity achieving information-theoretic minimax optimality \cite{arjevani2023lower}. \looseness-1

By effectively balancing preconditioning effectiveness with computational efficiency, RMNP provides a more scalable preconditioning approach that becomes particularly advantageous when \textsc{Muon}'s preconditioning process emerges as a computational bottleneck in large-scale training scenarios.
\section*{Acknowledgments}

We thank our collaborators, colleagues, and funding agencies. This work is supported by the DARPA AIQ program, the U.S. Department of Energy under Award Number DE-SC0025584, Dartmouth College, and Lambda AI. We also thank the three ICML reviewers, the Area Chair, and Yushun Zhang for their valuable feedback and discussions on our paper. We have incorporated their suggestions to refine the work and include additional interesting findings and supporting evidence.

\bibliography{example_paper}
\bibliographystyle{unsrtnat}

\clearpage
\appendix
\onecolumn

\setcounter{tocdepth}{2}
\section*{\centering Appendix Contents}
\noindent\hrulefill\par\medskip

\startcontents[appendix]
\printcontents[appendix]{}{0}{}
\noindent\hrulefill\par
\bigskip

\section{Proof of Theorem}
\label{appendix:proof_main_theorem}
\subsection{Notation}
\label{appendix:notation}
We first recall our notation here, Let $W \in \mathbb{R}^{m \times n}$ denote the parameter matrix, where $W_{i,:} \in \mathbb{R}^n$ denotes the $i$-th row. The matrix inner product is $\langle Z, W \rangle = \text{Tr}(Z^\top W)$. We use the Frobenius norm $\|W\|_F = \sqrt{\sum_{i,j} W_{i,j}^2}$, the mixed norm $\|W\|_{1,2} = \sum_{i=1}^m \|W_{i,:}\|_2$, and the norm $\|W\|_{\infty,2} = \max_{i=1,\ldots,m} \|W_{i,:}\|_2$. These satisfy the duality $|\langle A, B \rangle| \leq \|A\|_{1,2} \|B\|_{\infty,2}$. We use $\mathbb{E}[\cdot]$ to denote the expectation and $\mathbb{E}_t[\cdot \mid \mathcal{F}_{t-1}]$ to denote the conditional expectation given $\mathcal{F}_{t-1}$.
\subsection{Assumptions}
\label{appendix:assumptions}

\begin{assumption-restated}[Lipschitz Gradient]\label{assump:lipschitz-app}
The gradient of $f: \mathbb{R}^{m \times n} \to \mathbb{R}$ is Lipschitz continuous in one of the following norms:

\textbf{(a)} \textit{Frobenius norm:} There exists a constant $L_F > 0$ such that for all $W, W' \in \mathbb{R}^{m \times n}$,
$$\|\nabla f(W) - \nabla f(W')\|_F \leq L_F\|W - W'\|_F.$$

\textbf{(b)} \textit{Mixed norm:} There exists a constant $L_{\infty,2} > 0$ such that for all $W, W' \in \mathbb{R}^{m \times n}$,
$$\|\nabla f(W) - \nabla f(W')\|_{1,2} \leq L_{\infty,2}\|W - W'\|_{\infty,2}.$$
\end{assumption-restated}

\begin{assumption-restated}[Unbiased Estimator]\label{assump:unbiased-app}
For all iterations $t$ and all parameter values $W_t$,
$$\mathbb{E}_t[G_t \mid \mathcal{F}_{t-1}] = \nabla f(W_t).$$
\end{assumption-restated}

\begin{assumption-restated}[Bounded Variance]\label{assump:variance-app}
There exists a constant $\sigma > 0$ such that for all iterations $t$ and all parameter values $W_t$,
$$\mathbb{E}_t[\|G_t - \nabla f(W_t)\|_F^2 \mid \mathcal{F}_{t-1}] \leq \frac{\sigma^2}{B},$$
where $B$ is the batch size.
\end{assumption-restated}

\begin{assumption-restated}[Lower Bound]\label{assump:lower-app}
The objective function $f$ is bounded below with $f^* = \inf_{W \in \mathbb{R}^{m \times n}} f(W) > -\infty$. We define the initial optimality gap as $\Delta = f(W_0) - f^*$.
\end{assumption-restated}

\subsection{Proof of Lemmas}
\label{appendix:proof_lemmas}

\begin{lemma}
\label{lem:fro-rn-property-app}
Let $V \in \mathbb{R}^{m \times n}$ be any matrix and define $D = \text{RN}(V)$. Then:
\begin{enumerate}
    \item $\|D\|_F = \sqrt{m}$,
    \item $\langle V, D \rangle = \sum_{i=1}^m \|V_{i,:}\|_2 \geq \|V\|_F$.
\end{enumerate}
\end{lemma}

\begin{proof}
By definition, $D_{i,:} = V_{i,:}/\|V_{i,:}\|_2$, so $\|D_{i,:}\|_2 = 1$ for all $i$. Thus
$$\|D\|_F^2 = \sum_{i=1}^m \|D_{i,:}\|_2^2 = m.$$

For the inner product,
\begin{align*}
\langle V, D \rangle &= \sum_{i=1}^m \sum_{j=1}^n V_{i,j} \cdot \frac{V_{i,j}}{\|V_{i,:}\|_2}\\
&= \sum_{i=1}^m \frac{\|V_{i,:}\|_2^2}{\|V_{i,:}\|_2}\\
&= \sum_{i=1}^m \|V_{i,:}\|_2.
\end{align*}

For the inequality, let $a_i = \|V_{i,:}\|_2 \geq 0$. Squaring both sides of $\sum_i a_i \geq \sqrt{\sum_i a_i^2}$, we need
$$\left(\sum_{i=1}^m a_i\right)^2 \geq \sum_{i=1}^m a_i^2.$$
This holds since $\left(\sum_i a_i\right)^2 = \sum_i a_i^2 + 2\sum_{i<j} a_i a_j \geq \sum_i a_i^2$.
\end{proof}

\begin{lemma}
\label{lem:inf2-property-app}
Let $V \in \mathbb{R}^{m \times n}$ be any matrix and define $D = \text{RN}(V)$. Then:
\begin{enumerate}
    \item $\|D\|_{\infty,2} = 1$,
    \item $\langle V, D \rangle = \|V\|_{1,2}$.
\end{enumerate}
\end{lemma}

\begin{proof}
By definition of row normalization, $D_{i,:} = V_{i,:}/\|V_{i,:}\|_2$ for all $i$, which gives $\|D_{i,:}\|_2 = 1$ for all $i$. Therefore,
$$\|D\|_{\infty,2} = \max_{i=1,\ldots,m} \|D_{i,:}\|_2 = 1.$$

For the inner product, we have
\begin{align*}
\langle V, D \rangle &= \sum_{i=1}^m \sum_{j=1}^n V_{i,j} \cdot \frac{V_{i,j}}{\|V_{i,:}\|_2}\\
&= \sum_{i=1}^m \frac{\|V_{i,:}\|_2^2}{\|V_{i,:}\|_2}\\
&= \sum_{i=1}^m \|V_{i,:}\|_2\\
&= \|V\|_{1,2}.
\end{align*}
\end{proof}

\begin{lemma}
\label{lem:smooth-property-app}
Under Assumption~\ref{assump:lipschitz}(a), for any $W, W' \in \mathbb{R}^{m \times n}$,
$$f(W') \leq f(W) + \langle \nabla f(W), W' - W \rangle + \frac{L_F}{2}\|W' - W\|_F^2.$$
\end{lemma}

\begin{lemma}
\label{lem:fro-descent-app}
Under Assumption~\ref{assump:lipschitz}(a), for any iteration $t$,
$$f(W_t) - f(W_{t+1}) \geq \eta\langle \nabla f(W_t), D_t \rangle - \frac{L_F\eta^2 m}{2}.$$
\end{lemma}

\begin{proof}
We apply Lemma~\ref{lem:smooth-property-app} with $W = W_t$ and $W' = W_{t+1} = W_t - \eta D_t$:
\begin{align*}
f(W_{t+1}) &\leq f(W_t) + \langle \nabla f(W_t), W_{t+1} - W_t \rangle + \frac{L_F}{2}\|W_{t+1} - W_t\|_F^2\\
&= f(W_t) + \langle \nabla f(W_t), -\eta D_t \rangle + \frac{L_F}{2}\|-\eta D_t\|_F^2\\
&= f(W_t) - \eta\langle \nabla f(W_t), D_t \rangle + \frac{L_F\eta^2}{2}\|D_t\|_F^2\\
&= f(W_t) - \eta\langle \nabla f(W_t), D_t \rangle + \frac{L_F\eta^2}{2} \cdot m,
\end{align*}
where the last equality uses $\|D_t\|_F = \sqrt{m}$ from Lemma~\ref{lem:fro-rn-property-app}(i). Rearranging gives
$$f(W_t) - f(W_{t+1}) \geq \eta\langle \nabla f(W_t), D_t \rangle - \frac{L_F\eta^2 m}{2}.$$
\end{proof}

\begin{lemma}
\label{lem:fro-inner-bound-app}
Let $E_t = V_t - \nabla f(W_t)$. Then
$$\langle \nabla f(W_t), D_t \rangle \geq \|\nabla f(W_t)\|_F - (\sqrt{m}+1)\|E_t\|_F.$$
\end{lemma}

\begin{proof}
We decompose the inner product by writing $\nabla f(W_t) = V_t - E_t$:
\begin{align*}
\langle \nabla f(W_t), D_t \rangle &= \langle V_t - E_t, D_t \rangle\\
&= \langle V_t, D_t \rangle - \langle E_t, D_t \rangle.
\end{align*}

By Lemma~\ref{lem:fro-rn-property-app}(ii), we have
\begin{align*}
\langle V_t, D_t \rangle &= \sum_{i=1}^m \|V_{t,i,:}\|_2\\
&\geq \|V_t\|_F.
\end{align*}

For the error term, by the Cauchy-Schwarz inequality,
\begin{align*}
|\langle E_t, D_t \rangle| &\leq \|E_t\|_F \|D_t\|_F\\
&= \|E_t\|_F \cdot \sqrt{m},
\end{align*}
where we used Lemma~\ref{lem:fro-rn-property-app}(i).

Since $V_t = \nabla f(W_t) + E_t$, the reverse triangle inequality gives
\begin{align*}
\|V_t\|_F &= \|\nabla f(W_t) + E_t\|_F\\
&\geq \|\nabla f(W_t)\|_F - \|E_t\|_F.
\end{align*}

Combining all inequalities:
\begin{align*}
\langle \nabla f(W_t), D_t \rangle &= \langle V_t, D_t \rangle - \langle E_t, D_t \rangle\\
&\geq \|V_t\|_F - |\langle E_t, D_t \rangle|\\
&\geq \|V_t\|_F - \sqrt{m}\|E_t\|_F\\
&\geq (\|\nabla f(W_t)\|_F - \|E_t\|_F) - \sqrt{m}\|E_t\|_F\\
&= \|\nabla f(W_t)\|_F - (1 + \sqrt{m})\|E_t\|_F\\
&= \|\nabla f(W_t)\|_F - (\sqrt{m}+1)\|E_t\|_F.
\end{align*}
\end{proof}

\begin{lemma}
\label{lem:fro-error-app}
Define the stochastic noise $\xi_t = G_t - \nabla f(W_t)$ for $t \geq 1$, which satisfies $\mathbb{E}_t[\xi_t \mid \mathcal{F}_{t-1}] = 0$ by Assumption~\ref{assump:unbiased}. Then
$$\sum_{t=1}^{T} \mathbb{E}[\|E_t\|_F] \leq (T-1)\frac{L_F\eta\sqrt{m}\beta}{1-\beta} + T\frac{\sigma}{\sqrt{B}}\sqrt{\frac{1-\beta}{1+\beta}}.$$
\end{lemma}

\begin{proof}
From the momentum update rule $V_t = \beta V_{t-1} + (1-\beta)G_t$ and the definition $E_t = V_t - \nabla f(W_t)$, we derive:
\begin{align*}
E_t &= V_t - \nabla f(W_t)\\
&= \beta V_{t-1} + (1-\beta)G_t - \nabla f(W_t).
\end{align*}
We add and subtract $\beta\nabla f(W_{t-1})$ and $(1-\beta)\nabla f(W_t)$:
\begin{align*}
E_t &= \beta V_{t-1} - \beta\nabla f(W_{t-1}) + \beta\nabla f(W_{t-1})\\
&\quad + (1-\beta)G_t - (1-\beta)\nabla f(W_t) + (1-\beta)\nabla f(W_t) - \nabla f(W_t)\\
&= \beta(V_{t-1} - \nabla f(W_{t-1})) + \beta(\nabla f(W_{t-1}) - \nabla f(W_t))\\
&\quad + (1-\beta)(G_t - \nabla f(W_t))\\
&= \beta E_{t-1} + \beta(\nabla f(W_{t-1}) - \nabla f(W_t)) + (1-\beta)\xi_t.
\end{align*}

Assuming $E_0 = 0$ (since $V_0 = 0$), we expand this recursion by repeated application. For $t \geq 1$, we can show by induction that
$$E_t = \sum_{j=1}^{t-1} \beta^{t-j}(1-\beta)\xi_j + \sum_{j=1}^{t-1} \beta^{t-j}(\nabla f(W_j) - \nabla f(W_{j+1})).$$

For the base case $t=1$, we have $E_1 = \beta \cdot 0 + \beta(\nabla f(W_0) - \nabla f(W_1)) + (1-\beta)\xi_1$, which matches the formula when both sums are empty (upper limit $j=0$). For the inductive step, assume the formula holds for $t-1$. Then:
\begin{align*}
E_t &= \beta E_{t-1} + \beta(\nabla f(W_{t-1}) - \nabla f(W_t)) + (1-\beta)\xi_t\\
&= \beta\left[\sum_{j=1}^{t-2} \beta^{t-1-j}(1-\beta)\xi_j + \sum_{j=1}^{t-2} \beta^{t-1-j}(\nabla f(W_j) - \nabla f(W_{j+1}))\right]\\
&\quad + \beta(\nabla f(W_{t-1}) - \nabla f(W_t)) + (1-\beta)\xi_t\\
&= \sum_{j=1}^{t-2} \beta^{t-j}(1-\beta)\xi_j + \beta(1-\beta)\xi_t\\
&\quad + \sum_{j=1}^{t-2} \beta^{t-j}(\nabla f(W_j) - \nabla f(W_{j+1})) + \beta(\nabla f(W_{t-1}) - \nabla f(W_t))\\
&= \sum_{j=1}^{t-1} \beta^{t-j}(1-\beta)\xi_j + \sum_{j=1}^{t-1} \beta^{t-j}(\nabla f(W_j) - \nabla f(W_{j+1})).
\end{align*}

By the triangle inequality,
$$\|E_t\|_F \leq \left\|\sum_{j=1}^{t-1} \beta^{t-j}(1-\beta)\xi_j\right\|_F + \left\|\sum_{j=1}^{t-1} \beta^{t-j}(\nabla f(W_j) - \nabla f(W_{j+1}))\right\|_F.$$

For the gradient difference term, by the triangle inequality and Assumption~\ref{assump:lipschitz}(a),
\begin{align*}
&\left\|\sum_{j=1}^{t-1} \beta^{t-j}(\nabla f(W_j) - \nabla f(W_{j+1}))\right\|_F\\
&\leq \sum_{j=1}^{t-1} \beta^{t-j}\|\nabla f(W_j) - \nabla f(W_{j+1})\|_F\\
&\leq \sum_{j=1}^{t-1} \beta^{t-j} L_F\|W_j - W_{j+1}\|_F\\
&= \sum_{j=1}^{t-1} \beta^{t-j} L_F\|W_j - (W_j - \eta D_j)\|_F\\
&= \sum_{j=1}^{t-1} \beta^{t-j} L_F\eta\|D_j\|_F\\
&= \sum_{j=1}^{t-1} \beta^{t-j} L_F\eta\sqrt{m}\\
&= L_F\eta\sqrt{m} \sum_{j=1}^{t-1} \beta^{t-j}\\
&= L_F\eta\sqrt{m} \sum_{k=1}^{t-1} \beta^k\\
&= L_F\eta\sqrt{m} \cdot \beta\frac{1-\beta^{t-1}}{1-\beta}\\
&\leq L_F\eta\sqrt{m} \cdot \frac{\beta}{1-\beta}.
\end{align*}
Summing over $t=1,\ldots,T$ and changing the order of summation:
\begin{align*}
&\sum_{t=1}^{T} \left\|\sum_{j=1}^{t-1} \beta^{t-j}(\nabla f(W_j) - \nabla f(W_{j+1}))\right\|_F\\
&\leq L_F\eta\sqrt{m} \sum_{t=1}^{T} \sum_{j=1}^{t-1} \beta^{t-j}\\
&= L_F\eta\sqrt{m} \sum_{j=1}^{T-1} \sum_{t=j+1}^{T} \beta^{t-j}\\
&= L_F\eta\sqrt{m} \sum_{j=1}^{T-1} \sum_{k=1}^{T-j} \beta^k\\
&= L_F\eta\sqrt{m} \sum_{j=1}^{T-1} \beta\frac{1-\beta^{T-j}}{1-\beta}\\
&\leq L_F\eta\sqrt{m} \sum_{j=1}^{T-1} \frac{\beta}{1-\beta}\\
&= (T-1)\frac{L_F\eta\sqrt{m}\beta}{1-\beta}.
\end{align*}
For the noise term, since $\mathbb{E}_j[\xi_j \mid \mathcal{F}_{j-1}] = 0$ and the noises are conditionally independent,
\begin{align*}
&\mathbb{E}\left[\left\|\sum_{j=1}^{t-1} \beta^{t-j}(1-\beta)\xi_j\right\|_F^2\right]\\
&= \mathbb{E}\left[\left\langle \sum_{j=1}^{t-1} \beta^{t-j}(1-\beta)\xi_j, \sum_{k=1}^{t-1} \beta^{t-k}(1-\beta)\xi_k \right\rangle\right]\\
&= \sum_{j=1}^{t-1}\sum_{k=1}^{t-1} \beta^{t-j}\beta^{t-k}(1-\beta)^2 \mathbb{E}[\langle \xi_j, \xi_k \rangle]\\
&= \sum_{j=1}^{t-1} \beta^{2(t-j)}(1-\beta)^2 \mathbb{E}[\|\xi_j\|_F^2]\\
&\leq \sum_{j=1}^{t-1} \beta^{2(t-j)}(1-\beta)^2 \cdot \frac{\sigma^2}{B}\\
&= \frac{\sigma^2}{B}(1-\beta)^2 \sum_{j=1}^{t-1} \beta^{2(t-j)}\\
&= \frac{\sigma^2}{B}(1-\beta)^2 \sum_{k=1}^{t-1} \beta^{2k}\\
&\leq \frac{\sigma^2}{B}(1-\beta)^2 \sum_{k=0}^{\infty} \beta^{2k}\\
&= \frac{\sigma^2}{B}(1-\beta)^2 \cdot \frac{1}{1-\beta^2}\\
&= \frac{\sigma^2}{B}(1-\beta)^2 \cdot \frac{1}{(1-\beta)(1+\beta)}\\
&= \frac{\sigma^2}{B} \cdot \frac{1-\beta}{1+\beta}.
\end{align*}
By Jensen's inequality,
$$\mathbb{E}\left[\left\|\sum_{j=1}^{t-1} \beta^{t-j}(1-\beta)\xi_j\right\|_F\right] \leq \sqrt{\mathbb{E}\left[\left\|\sum_{j=1}^{t-1} \beta^{t-j}(1-\beta)\xi_j\right\|_F^2\right]} \leq \frac{\sigma}{\sqrt{B}}\sqrt{\frac{1-\beta}{1+\beta}}.$$
Summing over $t=1,\ldots,T$:
$$\sum_{t=1}^{T} \mathbb{E}\left[\left\|\sum_{j=1}^{t-1} \beta^{t-j}(1-\beta)\xi_j\right\|_F\right] \leq T\frac{\sigma}{\sqrt{B}}\sqrt{\frac{1-\beta}{1+\beta}}.$$
Combining both bounds:
\begin{align*}
\sum_{t=1}^{T} \mathbb{E}[\|E_t\|_F] &\leq (T-1)\frac{L_F\eta\sqrt{m}\beta}{1-\beta} + T\frac{\sigma}{\sqrt{B}}\sqrt{\frac{1-\beta}{1+\beta}}.
\end{align*}
\end{proof}

\begin{lemma}
\label{lem:inf2-descent-app}
Under Assumption~\ref{assump:lipschitz}(b), for any iteration $t$,
$$f(W_t) - f(W_{t+1}) \geq \eta\langle \nabla f(W_t), D_t \rangle - \frac{L_{\infty,2}\eta^2}{2}.$$
\end{lemma}

\begin{proof}
We apply Lemma~\ref{lem:smooth-property-app} with $W = W_t$ and $W' = W_{t+1} = W_t - \eta D_t$. However, we need to be careful as Lemma~\ref{lem:smooth-property-app} is stated in terms of Frobenius norm. For the $\|\cdot\|_{\infty,2}$ case, we use the fundamental theorem of calculus directly:
\begin{align*}
f(W_{t+1}) - f(W_t) &= \int_0^1 \langle \nabla f(W_t + s(W_{t+1} - W_t)), W_{t+1} - W_t \rangle ds\\
&= \int_0^1 \langle \nabla f(W_t - s\eta D_t), -\eta D_t \rangle ds\\
&= -\eta\langle \nabla f(W_t), D_t \rangle\\
&\quad - \eta\int_0^1 \langle \nabla f(W_t - s\eta D_t) - \nabla f(W_t), D_t \rangle ds.
\end{align*}

For the second integral, by the duality $|\langle A, B \rangle| \leq \|A\|_{1,2}\|B\|_{\infty,2}$ and Assumption~\ref{assump:lipschitz}(b),
\begin{align*}
&\left|\langle \nabla f(W_t - s\eta D_t) - \nabla f(W_t), D_t \rangle\right|\\
&\leq \|\nabla f(W_t - s\eta D_t) - \nabla f(W_t)\|_{1,2} \|D_t\|_{\infty,2}\\
&\leq L_{\infty,2}\|W_t - s\eta D_t - W_t\|_{\infty,2} \|D_t\|_{\infty,2}\\
&= L_{\infty,2} \cdot s\eta\|D_t\|_{\infty,2} \cdot \|D_t\|_{\infty,2}\\
&= L_{\infty,2} s\eta \|D_t\|_{\infty,2}^2\\
&= L_{\infty,2} s\eta \cdot 1,
\end{align*}
where we used Lemma~\ref{lem:inf2-property-app}(i).

Therefore,
\begin{align*}
\left|\int_0^1 \langle \nabla f(W_t - s\eta D_t) - \nabla f(W_t), D_t \rangle ds\right| &\leq \int_0^1 L_{\infty,2} s\eta \, ds\\
&= L_{\infty,2} \eta \int_0^1 s \, ds\\
&= L_{\infty,2} \eta \cdot \frac{1}{2}\\
&= \frac{L_{\infty,2}\eta}{2}.
\end{align*}

Combining these results:
$$f(W_{t+1}) - f(W_t) \geq -\eta\langle \nabla f(W_t), D_t \rangle - \frac{L_{\infty,2}\eta^2}{2},$$
which rearranges to the desired inequality.
\end{proof}

\begin{lemma}
\label{lem:inf2-inner-bound-app}
Let $E_t = V_t - \nabla f(W_t)$. Then
$$\langle \nabla f(W_t), D_t \rangle \geq \|\nabla f(W_t)\|_{1,2} - 2\|E_t\|_{1,2}.$$
\end{lemma}

\begin{proof}
We decompose the inner product by writing $\nabla f(W_t) = V_t - E_t$:
\begin{align*}
\langle \nabla f(W_t), D_t \rangle &= \langle V_t - E_t, D_t \rangle\\
&= \langle V_t, D_t \rangle - \langle E_t, D_t \rangle.
\end{align*}

By Lemma~\ref{lem:inf2-property-app}(ii),
$$\langle V_t, D_t \rangle = \|V_t\|_{1,2}.$$

For the error term, by the duality between $\|\cdot\|_{1,2}$ and $\|\cdot\|_{\infty,2}$,
\begin{align*}
|\langle E_t, D_t \rangle| &\leq \|E_t\|_{1,2} \|D_t\|_{\infty,2}\\
&= \|E_t\|_{1,2} \cdot 1,
\end{align*}
where we used Lemma~\ref{lem:inf2-property-app}(i).

Since $V_t = \nabla f(W_t) + E_t$, the triangle inequality gives
\begin{align*}
\|V_t\|_{1,2} &= \|\nabla f(W_t) + E_t\|_{1,2}\\
&\geq \|\nabla f(W_t)\|_{1,2} - \|E_t\|_{1,2}.
\end{align*}

Combining all inequalities:
\begin{align*}
\langle \nabla f(W_t), D_t \rangle &= \langle V_t, D_t \rangle - \langle E_t, D_t \rangle\\
&\geq \|V_t\|_{1,2} - |\langle E_t, D_t \rangle|\\
&\geq \|V_t\|_{1,2} - \|E_t\|_{1,2}\\
&\geq (\|\nabla f(W_t)\|_{1,2} - \|E_t\|_{1,2}) - \|E_t\|_{1,2}\\
&= \|\nabla f(W_t)\|_{1,2} - 2\|E_t\|_{1,2}.
\end{align*}
\end{proof}

\begin{lemma}
\label{lem:inf2-error-app}
Define the stochastic noise $\xi_t = G_t - \nabla f(W_t)$ for $t \geq 1$, which satisfies $\mathbb{E}_t[\xi_t \mid \mathcal{F}_{t-1}] = 0$ by Assumption~\ref{assump:unbiased}. Then
$$\sum_{t=1}^{T} \mathbb{E}[\|E_t\|_{1,2}] \leq (T-1)\frac{L_{\infty,2}\eta\beta}{1-\beta} + T\frac{\sqrt{m}\sigma}{\sqrt{B}}\sqrt{\frac{1-\beta}{1+\beta}}.$$
\end{lemma}

\begin{proof}
As in the proof of Lemma~\ref{lem:fro-error-app}, we have the recursion
$$E_t = \beta E_{t-1} + \beta(\nabla f(W_{t-1}) - \nabla f(W_t)) + (1-\beta)\xi_t,$$
which expands to
$$E_t = \sum_{j=1}^{t-1} \beta^{t-j}(1-\beta)\xi_j + \sum_{j=1}^{t-1} \beta^{t-j}(\nabla f(W_j) - \nabla f(W_{j+1})).$$

By the triangle inequality,
$$\|E_t\|_{1,2} \leq \left\|\sum_{j=1}^{t-1} \beta^{t-j}(1-\beta)\xi_j\right\|_{1,2} + \left\|\sum_{j=1}^{t-1} \beta^{t-j}(\nabla f(W_j) - \nabla f(W_{j+1}))\right\|_{1,2}.$$

For the gradient difference term, by the triangle inequality and Assumption~\ref{assump:lipschitz}(b),
\begin{align*}
&\left\|\sum_{j=1}^{t-1} \beta^{t-j}(\nabla f(W_j) - \nabla f(W_{j+1}))\right\|_{1,2}\\
&\leq \sum_{j=1}^{t-1} \beta^{t-j}\|\nabla f(W_j) - \nabla f(W_{j+1})\|_{1,2}\\
&\leq \sum_{j=1}^{t-1} \beta^{t-j} L_{\infty,2}\|W_j - W_{j+1}\|_{\infty,2}\\
&= \sum_{j=1}^{t-1} \beta^{t-j} L_{\infty,2}\|W_j - (W_j - \eta D_j)\|_{\infty,2}\\
&= \sum_{j=1}^{t-1} \beta^{t-j} L_{\infty,2}\eta\|D_j\|_{\infty,2}\\
&= \sum_{j=1}^{t-1} \beta^{t-j} L_{\infty,2}\eta\\
&= L_{\infty,2}\eta \sum_{j=1}^{t-1} \beta^{t-j}\\
&= L_{\infty,2}\eta \sum_{k=1}^{t-1} \beta^k\\
&= L_{\infty,2}\eta \cdot \beta\frac{1-\beta^{t-1}}{1-\beta}\\
&\leq L_{\infty,2}\eta \cdot \frac{\beta}{1-\beta}.
\end{align*}

Summing over $t=1,\ldots,T$ and changing the order of summation:
\begin{align*}
&\sum_{t=1}^{T} \left\|\sum_{j=1}^{t-1} \beta^{t-j}(\nabla f(W_j) - \nabla f(W_{j+1}))\right\|_{1,2}\\
&\leq L_{\infty,2}\eta \sum_{t=1}^{T} \sum_{j=1}^{t-1} \beta^{t-j}\\
&= L_{\infty,2}\eta \sum_{j=1}^{T-1} \sum_{t=j+1}^{T} \beta^{t-j}\\
&= L_{\infty,2}\eta \sum_{j=1}^{T-1} \sum_{k=1}^{T-j} \beta^k\\
&= L_{\infty,2}\eta \sum_{j=1}^{T-1} \beta\frac{1-\beta^{T-j}}{1-\beta}\\
&\leq L_{\infty,2}\eta \sum_{j=1}^{T-1} \frac{\beta}{1-\beta}\\
&= (T-1)\frac{L_{\infty,2}\eta\beta}{1-\beta}.
\end{align*}

For the noise term, we use the fact that $\|\cdot\|_{1,2} \leq \sqrt{m}\|\cdot\|_F$ by Cauchy-Schwarz. Specifically, for any matrix $A$,
\begin{align*}
\|A\|_{1,2} &= \sum_{i=1}^m \|A_{i,:}\|_2\\
&\leq \sqrt{m} \sqrt{\sum_{i=1}^m \|A_{i,:}\|_2^2}\\
&= \sqrt{m} \|A\|_F.
\end{align*}

Therefore,
\begin{align*}
\left\|\sum_{j=1}^{t-1} \beta^{t-j}(1-\beta)\xi_j\right\|_{1,2} &\leq \sqrt{m} \left\|\sum_{j=1}^{t-1} \beta^{t-j}(1-\beta)\xi_j\right\|_F.
\end{align*}

By Cauchy-Schwarz inequality (for expectations) and Jensen's inequality,
\begin{align*}
&\mathbb{E}\left[\left\|\sum_{j=1}^{t-1} \beta^{t-j}(1-\beta)\xi_j\right\|_{1,2}\right]\\
&\leq \sqrt{m} \mathbb{E}\left[\left\|\sum_{j=1}^{t-1} \beta^{t-j}(1-\beta)\xi_j\right\|_F\right]\\
&\leq \sqrt{m} \sqrt{\mathbb{E}\left[\left\|\sum_{j=1}^{t-1} \beta^{t-j}(1-\beta)\xi_j\right\|_F^2\right]}.
\end{align*}

From the proof of Lemma~\ref{lem:fro-error-app}, we know that
$$\mathbb{E}\left[\left\|\sum_{j=1}^{t-1} \beta^{t-j}(1-\beta)\xi_j\right\|_F^2\right] \leq \frac{\sigma^2}{B} \cdot \frac{1-\beta}{1+\beta}.$$

Therefore,
$$\mathbb{E}\left[\left\|\sum_{j=1}^{t-1} \beta^{t-j}(1-\beta)\xi_j\right\|_{1,2}\right] \leq \sqrt{m} \cdot \frac{\sigma}{\sqrt{B}}\sqrt{\frac{1-\beta}{1+\beta}}.$$

Summing over $t=1,\ldots,T$:
$$\sum_{t=1}^{T} \mathbb{E}\left[\left\|\sum_{j=1}^{t-1} \beta^{t-j}(1-\beta)\xi_j\right\|_{1,2}\right] \leq T\frac{\sqrt{m}\sigma}{\sqrt{B}}\sqrt{\frac{1-\beta}{1+\beta}}.$$

Combining both bounds:
\begin{align*}
\sum_{t=1}^{T} \mathbb{E}[\|E_t\|_{1,2}] &\leq (T-1)\frac{L_{\infty,2}\eta\beta}{1-\beta} + T\frac{\sqrt{m}\sigma}{\sqrt{B}}\sqrt{\frac{1-\beta}{1+\beta}}.
\end{align*}
\end{proof}

\begin{lemma}
\label{lem:12-error-fro-app}
Under Assumption~\ref{assump:lipschitz}(a) (Frobenius smoothness), and under the same noise conditions as Lemma~\ref{lem:fro-error-app}, we have
$$\sum_{t=1}^{T} \mathbb{E}[\|E_t\|_{1,2}] \leq (T-1)\frac{L_F\eta m\beta}{1-\beta} + T\frac{\sqrt{m}\sigma}{\sqrt{B}}\sqrt{\frac{1-\beta}{1+\beta}}.$$
\end{lemma}

\begin{proof}
As in the proof of Lemma~\ref{lem:fro-error-app}, we have the recursion
$$E_t = \beta E_{t-1} + \beta(\nabla f(W_{t-1}) - \nabla f(W_t)) + (1-\beta)\xi_t,$$
which expands to
$$E_t = \sum_{j=1}^{t-1} \beta^{t-j}(1-\beta)\xi_j + \sum_{j=1}^{t-1} \beta^{t-j}(\nabla f(W_j) - \nabla f(W_{j+1})).$$

By the triangle inequality,
$$\|E_t\|_{1,2} \leq \left\|\sum_{j=1}^{t-1} \beta^{t-j}(1-\beta)\xi_j\right\|_{1,2} + \left\|\sum_{j=1}^{t-1} \beta^{t-j}(\nabla f(W_j) - \nabla f(W_{j+1}))\right\|_{1,2}.$$

For the gradient difference term, we first use $\|\cdot\|_{1,2} \leq \sqrt{m}\|\cdot\|_F$, then the triangle inequality and Assumption~\ref{assump:lipschitz}(a):
\begin{align*}
&\left\|\sum_{j=1}^{t-1} \beta^{t-j}(\nabla f(W_j) - \nabla f(W_{j+1}))\right\|_{1,2}\\
&\leq \sqrt{m}\left\|\sum_{j=1}^{t-1} \beta^{t-j}(\nabla f(W_j) - \nabla f(W_{j+1}))\right\|_F\\
&\leq \sqrt{m}\sum_{j=1}^{t-1} \beta^{t-j}\|\nabla f(W_j) - \nabla f(W_{j+1})\|_F\\
&\leq \sqrt{m}\sum_{j=1}^{t-1} \beta^{t-j} L_F\|W_j - W_{j+1}\|_F\\
&= \sqrt{m}\sum_{j=1}^{t-1} \beta^{t-j} L_F\eta\|D_j\|_F\\
&= \sqrt{m}\sum_{j=1}^{t-1} \beta^{t-j} L_F\eta\sqrt{m}\\
&= L_F\eta m \sum_{j=1}^{t-1} \beta^{t-j}\\
&= L_F\eta m \sum_{k=1}^{t-1} \beta^k\\
&= L_F\eta m \cdot \beta\frac{1-\beta^{t-1}}{1-\beta}\\
&\leq L_F\eta m \cdot \frac{\beta}{1-\beta}.
\end{align*}

Summing over $t=1,\ldots,T$ and changing the order of summation:
\begin{align*}
&\sum_{t=1}^{T} \left\|\sum_{j=1}^{t-1} \beta^{t-j}(\nabla f(W_j) - \nabla f(W_{j+1}))\right\|_{1,2}\\
&\leq L_F\eta m \sum_{t=1}^{T} \sum_{j=1}^{t-1} \beta^{t-j}\\
&= L_F\eta m \sum_{j=1}^{T-1} \sum_{t=j+1}^{T} \beta^{t-j}\\
&= L_F\eta m \sum_{j=1}^{T-1} \sum_{k=1}^{T-j} \beta^k\\
&= L_F\eta m \sum_{j=1}^{T-1} \beta\frac{1-\beta^{T-j}}{1-\beta}\\
&\leq L_F\eta m \sum_{j=1}^{T-1} \frac{\beta}{1-\beta}\\
&= (T-1)\frac{L_F\eta m\beta}{1-\beta}.
\end{align*}

For the noise term, the analysis is identical to Lemma~\ref{lem:inf2-error-app}. We use $\|\cdot\|_{1,2} \leq \sqrt{m}\|\cdot\|_F$ and the result from Lemma~\ref{lem:fro-error-app}:
$$\mathbb{E}\left[\left\|\sum_{j=1}^{t-1} \beta^{t-j}(1-\beta)\xi_j\right\|_{1,2}\right] \leq \sqrt{m} \cdot \frac{\sigma}{\sqrt{B}}\sqrt{\frac{1-\beta}{1+\beta}}.$$

Summing over $t=1,\ldots,T$:
$$\sum_{t=1}^{T} \mathbb{E}\left[\left\|\sum_{j=1}^{t-1} \beta^{t-j}(1-\beta)\xi_j\right\|_{1,2}\right] \leq T\frac{\sqrt{m}\sigma}{\sqrt{B}}\sqrt{\frac{1-\beta}{1+\beta}}.$$

Combining both bounds:
\begin{align*}
\sum_{t=1}^{T} \mathbb{E}[\|E_t\|_{1,2}] &\leq (T-1)\frac{L_F\eta m\beta}{1-\beta} + T\frac{\sqrt{m}\sigma}{\sqrt{B}}\sqrt{\frac{1-\beta}{1+\beta}}.
\end{align*}
\end{proof}

\subsection{Proof of Theorem}
\label{sec_proofthem}
\begin{theorem-restated}
\label{thm:fro-app}
~\ref{thm:fro-convergence}
Under Assumptions~\ref{assump:lipschitz}(a), \ref{assump:unbiased}, \ref{assump:variance}, and \ref{assump:lower}, if Algorithm~\ref{algoRMNP} uses constant step size $\eta_t = \eta$ for all $t$ and momentum parameter $\beta \in [0,1)$, then
\begin{align*}
&\frac{1}{T}\sum_{t=1}^{T}\mathbb{E}\left[\|\nabla f(W_t)\|_F\right]\\
&\leq \frac{\Delta}{T\eta} + (\sqrt{m}+1)\Bigg[\left(1-\frac{1}{T}\right)\frac{L_F\eta\sqrt{m}\beta}{1-\beta}\\
&\qquad + \frac{\sigma}{\sqrt{B}}\sqrt{\frac{1-\beta}{1+\beta}}\Bigg] + \frac{L_F\eta m}{2}.
\end{align*}
\end{theorem-restated}

\begin{proof}
We sum the descent inequality from Lemma~\ref{lem:fro-descent-app} over all iterations $t = 1, \ldots, T$:
\begin{align*}
\sum_{t=1}^{T} [f(W_t) - f(W_{t+1})] &\geq \sum_{t=1}^{T} \left[\eta\langle \nabla f(W_t), D_t \rangle - \frac{L_F\eta^2 m}{2}\right]\\
&= \eta\sum_{t=1}^{T}\langle \nabla f(W_t), D_t \rangle - \sum_{t=1}^{T}\frac{L_F\eta^2 m}{2}\\
&= \eta\sum_{t=1}^{T}\langle \nabla f(W_t), D_t \rangle - \frac{TL_F\eta^2 m}{2}.
\end{align*}

The left-hand side is a telescoping sum:
\begin{align*}
\sum_{t=1}^{T} [f(W_t) - f(W_{t+1})] &= [f(W_1) - f(W_2)] + [f(W_2) - f(W_3)] + \cdots + [f(W_T) - f(W_{T+1})]\\
&= f(W_1) - f(W_{T+1}).
\end{align*}

However, we need to account for the initial iteration. From the algorithm, $W_1$ is obtained from $W_0$ via $W_1 = W_0 - \eta D_0$. Including this, the telescoping sum gives:
$$\sum_{t=0}^{T-1} [f(W_t) - f(W_{t+1})] = f(W_0) - f(W_T).$$

For consistency with our indexing where we sum from $t=1$ to $T$, we have:
$$\sum_{t=1}^{T} [f(W_t) - f(W_{t+1})] = f(W_1) - f(W_{T+1}).$$

To include the initial step, we note that
$$f(W_0) - f(W_1) \geq \eta\langle \nabla f(W_0), D_0 \rangle - \frac{L_F\eta^2 m}{2}.$$

For simplicity, we proceed with the standard formulation where we analyze iterations $t=1,\ldots,T$ starting from $W_0$:
$$f(W_0) - f(W_T) \geq \eta\sum_{t=1}^{T}\langle \nabla f(W_t), D_t \rangle - \frac{TL_F\eta^2 m}{2}.$$

We now apply Lemma~\ref{lem:fro-inner-bound-app} to bound the inner product from below. For each $t$, we have:
$$\langle \nabla f(W_t), D_t \rangle \geq \|\nabla f(W_t)\|_F - (\sqrt{m}+1)\|E_t\|_F.$$

Summing over $t = 1, \ldots, T$:
\begin{align*}
\sum_{t=1}^{T}\langle \nabla f(W_t), D_t \rangle &\geq \sum_{t=1}^{T}\left[\|\nabla f(W_t)\|_F - (\sqrt{m}+1)\|E_t\|_F\right]\\
&= \sum_{t=1}^{T}\|\nabla f(W_t)\|_F - (\sqrt{m}+1)\sum_{t=1}^{T}\|E_t\|_F.
\end{align*}

Substituting this into our previous inequality:
\begin{align*}
f(W_0) - f(W_T) &\geq \eta\left[\sum_{t=1}^{T}\|\nabla f(W_t)\|_F - (\sqrt{m}+1)\sum_{t=1}^{T}\|E_t\|_F\right] - \frac{TL_F\eta^2 m}{2}\\
&= \eta\sum_{t=1}^{T}\|\nabla f(W_t)\|_F - \eta(\sqrt{m}+1)\sum_{t=1}^{T}\|E_t\|_F - \frac{TL_F\eta^2 m}{2}.
\end{align*}

Rearranging to isolate the gradient norm sum:
\begin{align*}
\eta\sum_{t=1}^{T}\|\nabla f(W_t)\|_F &\leq f(W_0) - f(W_T) + \eta(\sqrt{m}+1)\sum_{t=1}^{T}\|E_t\|_F + \frac{TL_F\eta^2 m}{2}.
\end{align*}

Since $f(W_T) \geq f^* = \inf_W f(W)$ by definition, we have $f(W_0) - f(W_T) \leq f(W_0) - f^* = \Delta$. Thus:
\begin{align*}
\eta\sum_{t=1}^{T}\|\nabla f(W_t)\|_F &\leq \Delta + \eta(\sqrt{m}+1)\sum_{t=1}^{T}\|E_t\|_F + \frac{TL_F\eta^2 m}{2}.
\end{align*}

Dividing both sides by $\eta$:
\begin{align*}
\sum_{t=1}^{T}\|\nabla f(W_t)\|_F &\leq \frac{\Delta}{\eta} + (\sqrt{m}+1)\sum_{t=1}^{T}\|E_t\|_F + \frac{TL_F\eta m}{2}.
\end{align*}

Taking the expectation of both sides:
\begin{align*}
\sum_{t=1}^{T}\mathbb{E}[\|\nabla f(W_t)\|_F] &\leq \frac{\Delta}{\eta} + (\sqrt{m}+1)\sum_{t=1}^{T}\mathbb{E}[\|E_t\|_F] + \frac{TL_F\eta m}{2}.
\end{align*}

We now apply Lemma~\ref{lem:fro-error-app} to bound the error accumulation:
$$\sum_{t=1}^{T} \mathbb{E}[\|E_t\|_F] \leq (T-1)\frac{L_F\eta\sqrt{m}\beta}{1-\beta} + T\frac{\sigma}{\sqrt{B}}\sqrt{\frac{1-\beta}{1+\beta}}.$$

Substituting this bound:
\begin{align*}
\sum_{t=1}^{T}\mathbb{E}[\|\nabla f(W_t)\|_F] &\leq \frac{\Delta}{\eta} + (\sqrt{m}+1)\left[(T-1)\frac{L_F\eta\sqrt{m}\beta}{1-\beta} + T\frac{\sigma}{\sqrt{B}}\sqrt{\frac{1-\beta}{1+\beta}}\right]\\
&\quad + \frac{TL_F\eta m}{2}.
\end{align*}

Dividing both sides by $T$:
\begin{align*}
\frac{1}{T}\sum_{t=1}^{T}\mathbb{E}[\|\nabla f(W_t)\|_F] &\leq \frac{\Delta}{T\eta} + (\sqrt{m}+1)\left[\frac{T-1}{T}\cdot\frac{L_F\eta\sqrt{m}\beta}{1-\beta}\right.\\
&\quad\left. + \frac{\sigma}{\sqrt{B}}\sqrt{\frac{1-\beta}{1+\beta}}\right] + \frac{L_F\eta m}{2}.
\end{align*}

Since $\frac{T-1}{T} = 1 - \frac{1}{T}$, we obtain:
\begin{align*}
\frac{1}{T}\sum_{t=1}^{T}\mathbb{E}[\|\nabla f(W_t)\|_F] &\leq \frac{\Delta}{T\eta} + (\sqrt{m}+1)\Bigg[\left(1-\frac{1}{T}\right)\frac{L_F\eta\sqrt{m}\beta}{1-\beta}\\
&\qquad + \frac{\sigma}{\sqrt{B}}\sqrt{\frac{1-\beta}{1+\beta}}\Bigg] + \frac{L_F\eta m}{2}.
\end{align*}

This completes the proof.
\end{proof}

\begin{theorem-restated}
\label{thm:inf2-app}
~\ref{thm:inf2-convergence}
Under Assumptions~\ref{assump:lipschitz}(b), \ref{assump:unbiased}, \ref{assump:variance}, and \ref{assump:lower}, if Algorithm~\ref{algoRMNP} uses constant step size $\eta_t = \eta$ for all $t$ and momentum parameter $\beta \in [0,1)$, then
\begin{align*}
&\frac{1}{T}\sum_{t=1}^{T}\mathbb{E}\left[\|\nabla f(W_t)\|_{1,2}\right]\\
&\leq \frac{\Delta}{T\eta} + 2\Bigg[\left(1-\frac{1}{T}\right)\frac{L_{\infty,2}\eta\beta}{1-\beta}\\
&\qquad + \frac{\sqrt{m}\sigma}{\sqrt{B}}\sqrt{\frac{1-\beta}{1+\beta}}\Bigg] + \frac{L_{\infty,2}\eta}{2}.
\end{align*}
\end{theorem-restated}

\begin{proof}
We sum the descent inequality from Lemma~\ref{lem:inf2-descent-app} over all iterations $t = 1, \ldots, T$:
\begin{align*}
\sum_{t=1}^{T} [f(W_t) - f(W_{t+1})] &\geq \sum_{t=1}^{T} \left[\eta\langle \nabla f(W_t), D_t \rangle - \frac{L_{\infty,2}\eta^2}{2}\right]\\
&= \eta\sum_{t=1}^{T}\langle \nabla f(W_t), D_t \rangle - \sum_{t=1}^{T}\frac{L_{\infty,2}\eta^2}{2}\\
&= \eta\sum_{t=1}^{T}\langle \nabla f(W_t), D_t \rangle - \frac{TL_{\infty,2}\eta^2}{2}.
\end{align*}

The left-hand side is the telescoping sum $f(W_0) - f(W_T)$ (using the same argument as in the proof of Theorem~\ref{thm:fro-app}). Thus:
$$f(W_0) - f(W_T) \geq \eta\sum_{t=1}^{T}\langle \nabla f(W_t), D_t \rangle - \frac{TL_{\infty,2}\eta^2}{2}.$$

We now apply Lemma~\ref{lem:inf2-inner-bound-app} to bound the inner product from below. For each $t$, we have:
$$\langle \nabla f(W_t), D_t \rangle \geq \|\nabla f(W_t)\|_{1,2} - 2\|E_t\|_{1,2}.$$

Summing over $t = 1, \ldots, T$:
\begin{align*}
\sum_{t=1}^{T}\langle \nabla f(W_t), D_t \rangle &\geq \sum_{t=1}^{T}\left[\|\nabla f(W_t)\|_{1,2} - 2\|E_t\|_{1,2}\right]\\
&= \sum_{t=1}^{T}\|\nabla f(W_t)\|_{1,2} - 2\sum_{t=1}^{T}\|E_t\|_{1,2}.
\end{align*}

Substituting this into our previous inequality:
\begin{align*}
f(W_0) - f(W_T) &\geq \eta\left[\sum_{t=1}^{T}\|\nabla f(W_t)\|_{1,2} - 2\sum_{t=1}^{T}\|E_t\|_{1,2}\right] - \frac{TL_{\infty,2}\eta^2}{2}\\
&= \eta\sum_{t=1}^{T}\|\nabla f(W_t)\|_{1,2} - 2\eta\sum_{t=1}^{T}\|E_t\|_{1,2} - \frac{TL_{\infty,2}\eta^2}{2}.
\end{align*}

Rearranging to isolate the gradient norm sum:
\begin{align*}
\eta\sum_{t=1}^{T}\|\nabla f(W_t)\|_{1,2} &\leq f(W_0) - f(W_T) + 2\eta\sum_{t=1}^{T}\|E_t\|_{1,2} + \frac{TL_{\infty,2}\eta^2}{2}.
\end{align*}

Since $f(W_T) \geq f^*$, we have $f(W_0) - f(W_T) \leq \Delta$. Thus:
\begin{align*}
\eta\sum_{t=1}^{T}\|\nabla f(W_t)\|_{1,2} &\leq \Delta + 2\eta\sum_{t=1}^{T}\|E_t\|_{1,2} + \frac{TL_{\infty,2}\eta^2}{2}.
\end{align*}

Dividing both sides by $\eta$:
\begin{align*}
\sum_{t=1}^{T}\|\nabla f(W_t)\|_{1,2} &\leq \frac{\Delta}{\eta} + 2\sum_{t=1}^{T}\|E_t\|_{1,2} + \frac{TL_{\infty,2}\eta}{2}.
\end{align*}

Taking the expectation of both sides:
\begin{align*}
\sum_{t=1}^{T}\mathbb{E}[\|\nabla f(W_t)\|_{1,2}] &\leq \frac{\Delta}{\eta} + 2\sum_{t=1}^{T}\mathbb{E}[\|E_t\|_{1,2}] + \frac{TL_{\infty,2}\eta}{2}.
\end{align*}

We now apply Lemma~\ref{lem:inf2-error-app} to bound the error accumulation:
$$\sum_{t=1}^{T} \mathbb{E}[\|E_t\|_{1,2}] \leq (T-1)\frac{L_{\infty,2}\eta\beta}{1-\beta} + T\frac{\sqrt{m}\sigma}{\sqrt{B}}\sqrt{\frac{1-\beta}{1+\beta}}.$$

Substituting this bound:
\begin{align*}
\sum_{t=1}^{T}\mathbb{E}[\|\nabla f(W_t)\|_{1,2}] &\leq \frac{\Delta}{\eta} + 2\left[(T-1)\frac{L_{\infty,2}\eta\beta}{1-\beta} + T\frac{\sqrt{m}\sigma}{\sqrt{B}}\sqrt{\frac{1-\beta}{1+\beta}}\right]\\
&\quad + \frac{TL_{\infty,2}\eta}{2}.
\end{align*}

Dividing both sides by $T$:
\begin{align*}
\frac{1}{T}\sum_{t=1}^{T}\mathbb{E}[\|\nabla f(W_t)\|_{1,2}] &\leq \frac{\Delta}{T\eta} + 2\left[\frac{T-1}{T}\cdot\frac{L_{\infty,2}\eta\beta}{1-\beta}\right.\\
&\quad\left. + \frac{\sqrt{m}\sigma}{\sqrt{B}}\sqrt{\frac{1-\beta}{1+\beta}}\right] + \frac{L_{\infty,2}\eta}{2}.
\end{align*}

Since $\frac{T-1}{T} = 1 - \frac{1}{T}$, we obtain:
\begin{align*}
\frac{1}{T}\sum_{t=1}^{T}\mathbb{E}[\|\nabla f(W_t)\|_{1,2}] &\leq \frac{\Delta}{T\eta} + 2\Bigg[\left(1-\frac{1}{T}\right)\frac{L_{\infty,2}\eta\beta}{1-\beta}\\
&\qquad + \frac{\sqrt{m}\sigma}{\sqrt{B}}\sqrt{\frac{1-\beta}{1+\beta}}\Bigg] + \frac{L_{\infty,2}\eta}{2}.
\end{align*}

This completes the proof.
\end{proof}
\begin{theorem-restated}
\label{thm:12-fro-app}
~\ref{thm:12-fro-convergence}
Under Assumptions~\ref{assump:lipschitz}(a), \ref{assump:unbiased}, \ref{assump:variance}, and \ref{assump:lower}, if Algorithm~\ref{algoRMNP} uses constant step size $\eta_t = \eta$ for all $t$ and momentum parameter $\beta \in [0,1)$, then
\begin{align*}
&\frac{1}{T}\sum_{t=1}^{T}\mathbb{E}\left[\|\nabla f(W_t)\|_{1,2}\right]\\
&\leq \frac{\Delta}{T\eta} + 2\Bigg[\left(1-\frac{1}{T}\right)\frac{L_F\eta m\beta}{1-\beta}\\
&\qquad + \frac{\sqrt{m}\sigma}{\sqrt{B}}\sqrt{\frac{1-\beta}{1+\beta}}\Bigg] + \frac{L_F\eta m}{2}.
\end{align*}
\end{theorem-restated}

\begin{proof}
We sum the descent inequality from Lemma~\ref{lem:fro-descent-app} over all iterations $t = 1, \ldots, T$:
\begin{align*}
\sum_{t=1}^{T} [f(W_t) - f(W_{t+1})] &\geq \sum_{t=1}^{T} \left[\eta\langle \nabla f(W_t), D_t \rangle - \frac{L_F\eta^2 m}{2}\right]\\
&= \eta\sum_{t=1}^{T}\langle \nabla f(W_t), D_t \rangle - \sum_{t=1}^{T}\frac{L_F\eta^2 m}{2}\\
&= \eta\sum_{t=1}^{T}\langle \nabla f(W_t), D_t \rangle - \frac{TL_F\eta^2 m}{2}.
\end{align*}

The left-hand side is the telescoping sum $f(W_0) - f(W_T)$ (using the same argument as in the proof of Theorem~\ref{thm:fro-app}). Thus:
$$f(W_0) - f(W_T) \geq \eta\sum_{t=1}^{T}\langle \nabla f(W_t), D_t \rangle - \frac{TL_F\eta^2 m}{2}.$$

We now apply Lemma~\ref{lem:inf2-inner-bound-app} to bound the inner product from below. For each $t$, we have:
$$\langle \nabla f(W_t), D_t \rangle \geq \|\nabla f(W_t)\|_{1,2} - 2\|E_t\|_{1,2}.$$

Summing over $t = 1, \ldots, T$:
\begin{align*}
\sum_{t=1}^{T}\langle \nabla f(W_t), D_t \rangle &\geq \sum_{t=1}^{T}\left[\|\nabla f(W_t)\|_{1,2} - 2\|E_t\|_{1,2}\right]\\
&= \sum_{t=1}^{T}\|\nabla f(W_t)\|_{1,2} - 2\sum_{t=1}^{T}\|E_t\|_{1,2}.
\end{align*}

Substituting this into our previous inequality:
\begin{align*}
f(W_0) - f(W_T) &\geq \eta\left[\sum_{t=1}^{T}\|\nabla f(W_t)\|_{1,2} - 2\sum_{t=1}^{T}\|E_t\|_{1,2}\right] - \frac{TL_F\eta^2 m}{2}\\
&= \eta\sum_{t=1}^{T}\|\nabla f(W_t)\|_{1,2} - 2\eta\sum_{t=1}^{T}\|E_t\|_{1,2} - \frac{TL_F\eta^2 m}{2}.
\end{align*}

Rearranging to isolate the gradient norm sum:
\begin{align*}
\eta\sum_{t=1}^{T}\|\nabla f(W_t)\|_{1,2} &\leq f(W_0) - f(W_T) + 2\eta\sum_{t=1}^{T}\|E_t\|_{1,2} + \frac{TL_F\eta^2 m}{2}.
\end{align*}

Since $f(W_T) \geq f^*$, we have $f(W_0) - f(W_T) \leq \Delta$. Thus:
\begin{align*}
\eta\sum_{t=1}^{T}\|\nabla f(W_t)\|_{1,2} &\leq \Delta + 2\eta\sum_{t=1}^{T}\|E_t\|_{1,2} + \frac{TL_F\eta^2 m}{2}.
\end{align*}

Dividing both sides by $\eta$:
\begin{align*}
\sum_{t=1}^{T}\|\nabla f(W_t)\|_{1,2} &\leq \frac{\Delta}{\eta} + 2\sum_{t=1}^{T}\|E_t\|_{1,2} + \frac{TL_F\eta m}{2}.
\end{align*}

Taking the expectation of both sides:
\begin{align*}
\sum_{t=1}^{T}\mathbb{E}[\|\nabla f(W_t)\|_{1,2}] &\leq \frac{\Delta}{\eta} + 2\sum_{t=1}^{T}\mathbb{E}[\|E_t\|_{1,2}] + \frac{TL_F\eta m}{2}.
\end{align*}

We now apply Lemma~\ref{lem:12-error-fro-app} to bound the error accumulation:
$$\sum_{t=1}^{T} \mathbb{E}[\|E_t\|_{1,2}] \leq (T-1)\frac{L_F\eta m\beta}{1-\beta} + T\frac{\sqrt{m}\sigma}{\sqrt{B}}\sqrt{\frac{1-\beta}{1+\beta}}.$$

Substituting this bound:
\begin{align*}
\sum_{t=1}^{T}\mathbb{E}[\|\nabla f(W_t)\|_{1,2}] &\leq \frac{\Delta}{\eta} + 2\left[(T-1)\frac{L_F\eta m\beta}{1-\beta} + T\frac{\sqrt{m}\sigma}{\sqrt{B}}\sqrt{\frac{1-\beta}{1+\beta}}\right]\\
&\quad + \frac{TL_F\eta m}{2}.
\end{align*}

Dividing both sides by $T$:
\begin{align*}
\frac{1}{T}\sum_{t=1}^{T}\mathbb{E}[\|\nabla f(W_t)\|_{1,2}] &\leq \frac{\Delta}{T\eta} + 2\left[\frac{T-1}{T}\cdot\frac{L_F\eta m\beta}{1-\beta}\right.\\
&\quad\left. + \frac{\sqrt{m}\sigma}{\sqrt{B}}\sqrt{\frac{1-\beta}{1+\beta}}\right] + \frac{L_F\eta m}{2}.
\end{align*}

Since $\frac{T-1}{T} = 1 - \frac{1}{T}$, we obtain:
\begin{align*}
\frac{1}{T}\sum_{t=1}^{T}\mathbb{E}[\|\nabla f(W_t)\|_{1,2}] &\leq \frac{\Delta}{T\eta} + 2\Bigg[\left(1-\frac{1}{T}\right)\frac{L_F\eta m\beta}{1-\beta}\\
&\qquad + \frac{\sqrt{m}\sigma}{\sqrt{B}}\sqrt{\frac{1-\beta}{1+\beta}}\Bigg] + \frac{L_F\eta m}{2}.
\end{align*}

This completes the proof.
\end{proof}

\section{Analysis of \textsc{Muon} Preconditioner}
\label{appendix:diagonal_dominance_setup}

This section provides implementation details for the diagonal dominance analysis presented in Section~\ref{subsec:diagonal_dominance}.

\paragraph{Metric Computation}
For each matrix parameter $V_t \in \mathbb{R}^{m \times n}$ in the network, we compute the diagonal dominance metrics as follows:

\begin{enumerate}
    \item \textbf{Gram Matrix Computation:} We first compute the Gram matrix $G = V_t V_t^T \in \mathbb{R}^{m \times m}$.

    \item \textbf{Row-wise Ratio Calculation:} For each row $i \in \{1, \ldots, m\}$, we compute the ratio $r_i$ between the diagonal element and the average magnitude of off-diagonal elements:
    \begin{equation}
        r_i = \frac{G_{ii}}{\frac{1}{m-1}\sum_{j \neq i} |G_{ij}|}
    \end{equation}
    where $G_{ii} = \|V_{t,i:}\|_2^2$ is the squared norm of the $i$-th row of $V_t$.

    \item \textbf{Per-Parameter Aggregation:} For each matrix parameter, we aggregate the row-wise ratios into three statistics:
    \begin{align}
        r_{\text{avg}} &= \frac{1}{m}\sum_{i=1}^{m} r_i, \\
        r_{\min} &= \min_{i \in \{1,\ldots,m\}} r_i, \\
        r_{\max} &= \max_{i \in \{1,\ldots,m\}} r_i.
    \end{align}

    \item \textbf{Global Aggregation:} The global statistics $\overline{r}_{\text{avg}}$, $\overline{r}_{\min}$, and $\overline{r}_{\max}$ are computed by averaging the corresponding per-parameter metrics across all $K$ matrix parameters in the network:
    \begin{align}
        \overline{r}_{\text{avg}} &= \frac{1}{K}\sum_{k=1}^{K} r_{\text{avg}}^{(k)}, \\
        \overline{r}_{\min} &= \frac{1}{K}\sum_{k=1}^{K} r_{\min}^{(k)}, \\
        \overline{r}_{\max} &= \frac{1}{K}\sum_{k=1}^{K} r_{\max}^{(k)},
    \end{align}
    where the superscript $(k)$ denotes the metric for the $k$-th matrix parameter.
\end{enumerate}

\paragraph{Logging Configuration}
We use Weights \& Biases (wandb) for metric tracking. The diagonal dominance ratios are computed and logged at every training step. The metrics are computed within the optimizer's \texttt{step()} function, immediately after the momentum update and before the Newton-Schulz orthogonalization. In distributed training settings, the per-parameter metrics are computed locally on each GPU (parameters are distributed across GPUs), and the global statistics are synchronized via \texttt{all\_reduce} operations.

\paragraph{Model and Training Configuration}
We conduct the analysis on both GPT-2 and LLaMA model families to align with the main pre-training setting. For GPT-2, we analyze Small (125M), Medium (355M), and Large (770M) on OpenWebText; for LLaMA, we analyze 60M, 130M, 350M, and 1B on C4. Model scales, training steps, warm-up schedules, sequence length, and batch size follow the settings in Section~D.2 (Appendix~\ref{appendix:model_config}). In particular, GPT-2 uses 10K/20K/40K steps with sequence length 1024 and batch size 480, while LLaMA uses 10K/20K/60K/90K steps with sequence length 256 and batch size 512. Optimization hyperparameters follow Appendix~\ref{appendix:hyperparams}; specifically, we use \textsc{Muon} with momentum $0.95$, weight decay $0.1$, and Newton-Schulz iteration steps of 5.

\paragraph{Visualization}
In all dominance figures of this appendix, transparent curves represent the raw logged values, while the solid curves are smoothed using simple moving average with a window size of 50. The red dashed line at $y=1$ serves as a reference threshold---values above this line indicate that the diagonal elements dominate over the average off-diagonal magnitude, confirming diagonal dominance of the Gram matrix. Per-parameter $r_{\text{avg}}, r_{\min}, r_{\max}$ for three representative matrix parameters of GPT-2 and LLaMA are shown in Figures~\ref{fig:dominance_curves_appendix} and~\ref{fig:Llama_dominance_curves_appendix}. The cross-scale, cross-architecture comparison of the global ratios $\overline{r}_{\text{avg}}, \overline{r}_{\min}, \overline{r}_{\max}$ is reproduced from the main body in Figure~\ref{fig:gpt2_llama_global_ratio_appendix}. Per-parameter ratios for the two largest models---GPT-2 XLarge (1.5B) and LLaMA 1B---are reported in Figure~\ref{fig:xl_1b_dominance_curves_appendix}.

\begin{figure*}[!htb]
    \centering
    \includegraphics[width=1.00\linewidth]{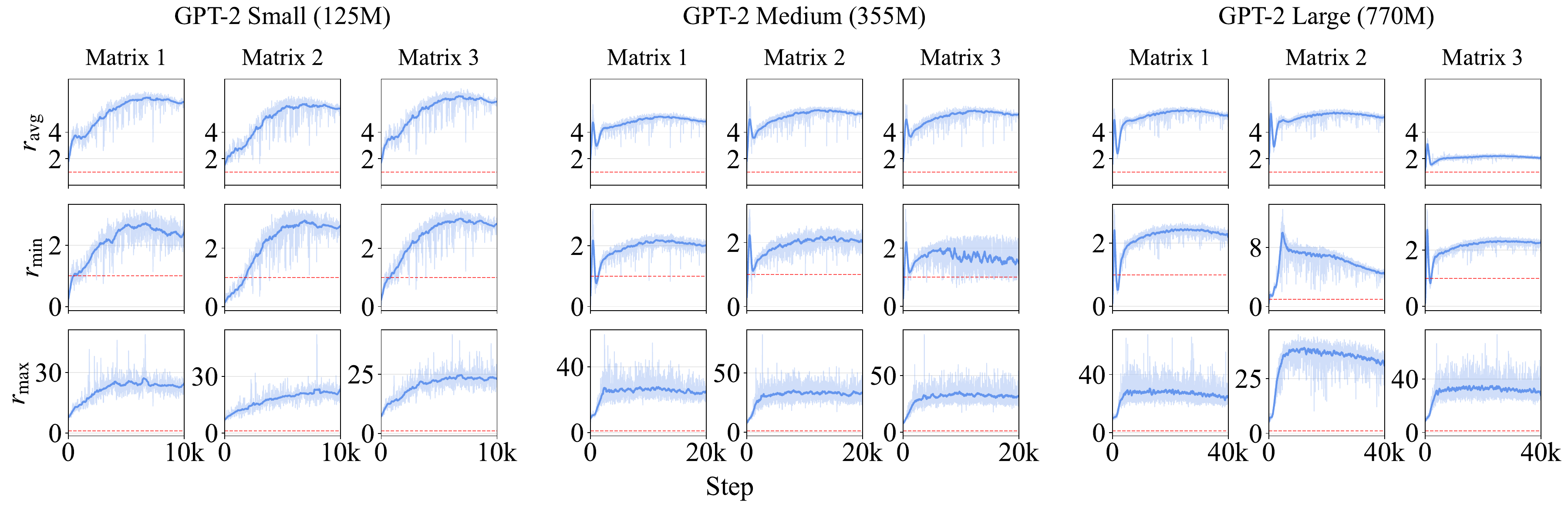}
    \caption{Per-parameter diagonal dominance ratios $r_{\text{avg}}$, $r_{\min}$, $r_{\max}$ (rows) for three representative matrix parameters (columns) during GPT-2 Small (125M), GPT-2 Medium (355M), and GPT-2 Large (770M) pre-training. Transparent curves: raw values; solid curves: smoothed with window size 50. Red dashed line: $y=1$ threshold.}
    \label{fig:dominance_curves_appendix}
\end{figure*}

\begin{figure*}[!htb]
    \centering
    \includegraphics[width=1.00\linewidth]{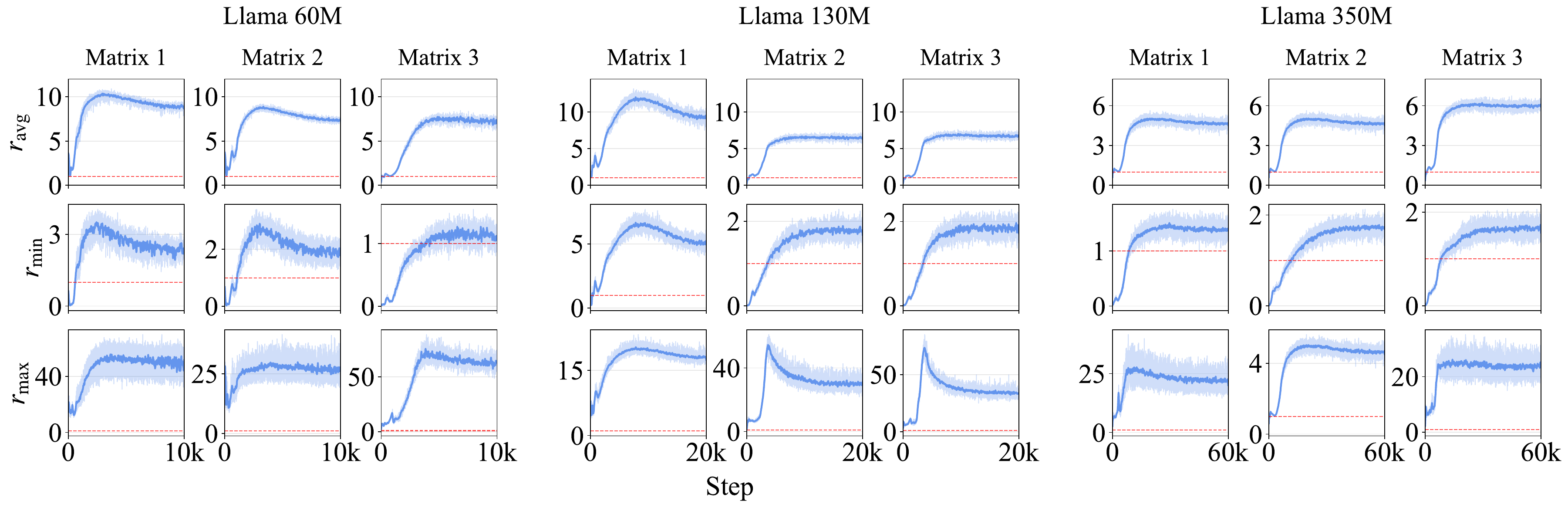}
    \caption{Per-parameter diagonal dominance ratios $r_{\text{avg}}$, $r_{\min}$, $r_{\max}$ (rows) for three representative matrix parameters (columns) during LLaMA 60M, LLaMA 130M, and LLaMA 350M pre-training. Transparent curves: raw values; solid curves: smoothed with window size 50. Red dashed line: $y=1$ threshold.}
    \label{fig:Llama_dominance_curves_appendix}
\end{figure*}

\begin{figure*}[!htb]
    \centering
    \includegraphics[width=1.00\linewidth]{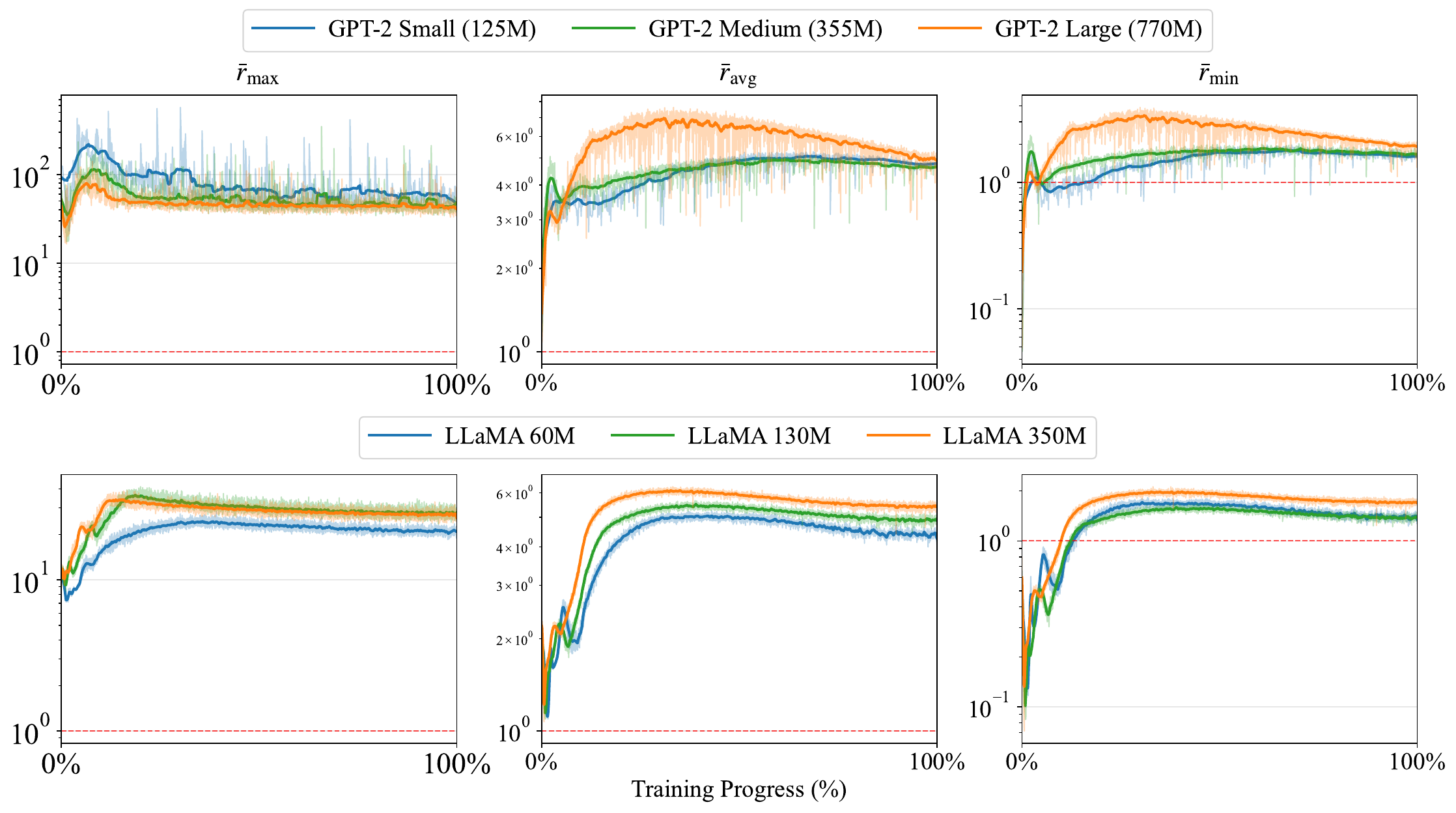}
    \caption{Cross-architecture, cross-scale comparison of the global diagonal dominance ratios $\overline{r}_{\text{avg}}$, $\overline{r}_{\min}$, $\overline{r}_{\max}$ (columns). Top row: GPT-2 Small (125M), Medium (355M), and Large (770M) on OpenWebText. Bottom row: LLaMA 60M, 130M, and 350M on C4. The x-axis is rescaled to the relative training progress (\%) so that all model scales within a row align on a shared horizontal range; the y-axis is in log scale. Transparent curves: raw values; solid curves: smoothed with window size 50. Red dashed line: $y=1$ threshold. The figure reproduces the main-body version (Figure~\ref{fig:global_dominance_curves}) and makes explicit that for both architecture families, larger models exhibit progressively stronger diagonal dominance across all three statistics.}
    \label{fig:gpt2_llama_global_ratio_appendix}
\end{figure*}

\begin{figure*}[!htb]
    \centering
    \includegraphics[width=1.00\linewidth]{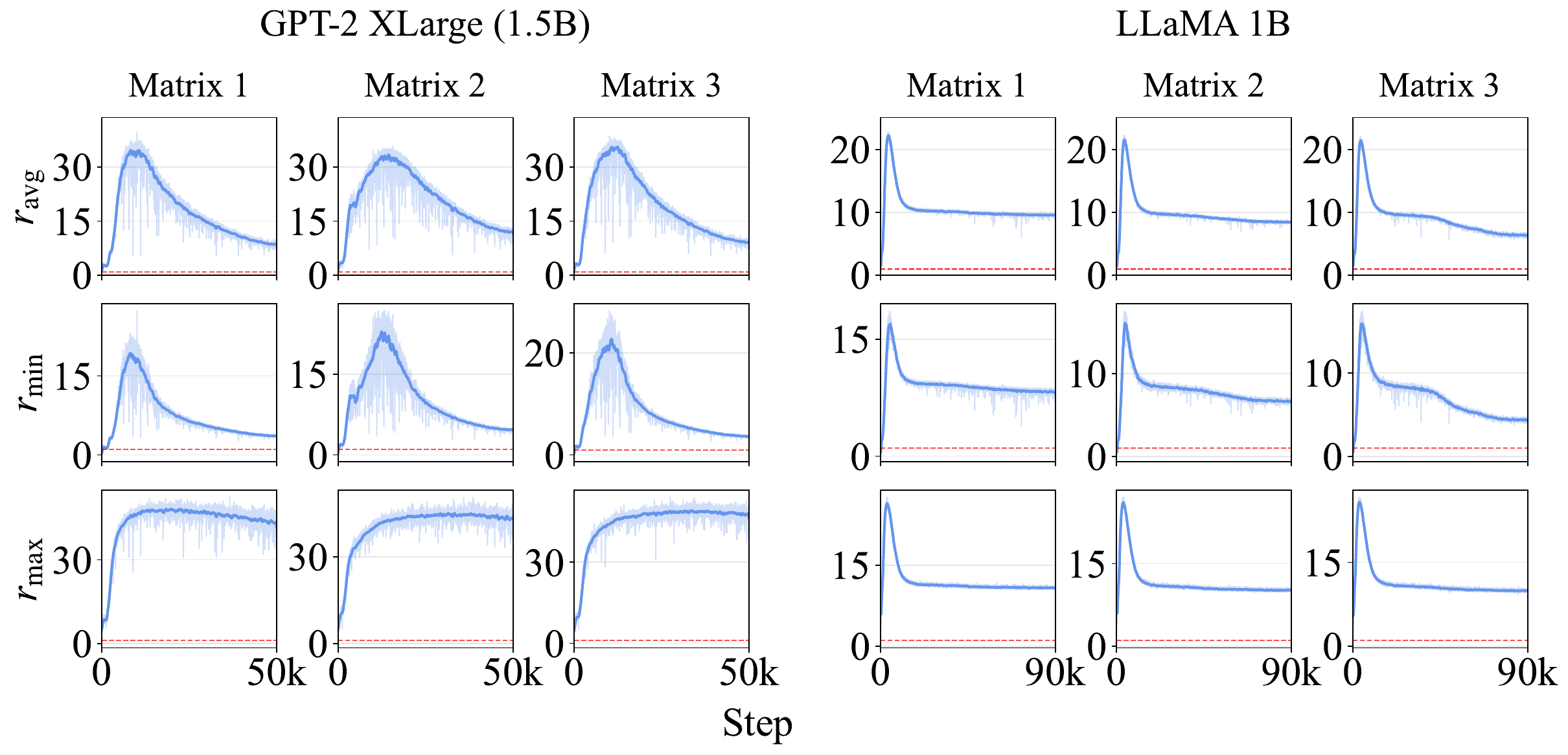}
    \caption{Per-parameter diagonal dominance ratios $r_{\text{avg}}$, $r_{\min}$, $r_{\max}$ (rows) for three representative matrix parameters (columns) on the largest scales evaluated in this paper: GPT-2 XLarge (1.5B) on FineWeb-Edu-100B (left block) and LLaMA 1B on C4 (right block). Transparent curves: raw values; solid curves: smoothed with window size 50. Red dashed line: $y=1$ threshold. All metrics remain comfortably above the threshold throughout training, confirming that the row-wise block-diagonal dominance of the \textsc{Muon} preconditioner persists at the largest scales we evaluate.}
    \label{fig:xl_1b_dominance_curves_appendix}
\end{figure*}

\section{Preconditioning Process Wall-Clock Time}
\label{appendix:wall_clock}

This section provides detailed efficiency measurements for the preconditioning time cost analysis presented in Section~\ref{sec:efficiency}. As shown in Table~\ref{tab:memory}, \textsc{RMNP} achieves significant speedups over \textsc{Muon} across all model sizes while maintaining identical memory usage. Specifically, \textsc{RMNP} reduces the preconditioner computation time by approximately $13\times$ to $43\times$, demonstrating its scalability advantage for large-scale training.

\begin{table*}[!htbp]
\centering
\caption{Efficiency comparison between \textsc{Muon} (Newton-Schulz orthogonalization) and \textsc{RMNP} (row normalization) on GPT-2 models. Time and memory usage are measured over 100 steps with batch size 16 on one single RTX Pro 6000 GPU.}
\begin{tabular}{lcccc}
\toprule
\multirow{2}{*}{Size} & \multicolumn{2}{c}{Time Cost (s)} & \multicolumn{2}{c}{Memory (MB)} \\
\cmidrule(lr){2-3} \cmidrule(lr){4-5}
& \textsc{Muon} & \textsc{RMNP} & \textsc{Muon} & \textsc{RMNP} \\
\midrule
60M  & 1.480  & \textbf{0.115} & 7804 & 7804 \\
125M & 2.975 & \textbf{0.201} & 11797 & 11797 \\
200M & 4.140  & \textbf{0.260} & 15352 & 15352 \\
355M & 7.380  & \textbf{0.401} & 23225 & 23225 \\
500M & 15.720 & \textbf{0.462} & 30011 & 30011 \\
770M & 27.070 & \textbf{0.611} & 41508 & 41508 \\
1.3B & 30.570 & \textbf{0.783} & 61043 & 61043 \\
1.5B & 36.650 & \textbf{0.855} & 69465 & 69465 \\
\bottomrule
\end{tabular}

\label{tab:memory}
\end{table*}

\subsection{Model Configuration for Preconditioning Time Cost}
\label{app:efficiency_config}
Table~\ref{tab:gpt2_config_efficiency} presents the detailed model configurations used for measuring preconditioning time cost in Section~\ref{sec:efficiency}.

\begin{table}[!htbp]
\centering
\caption{GPT-2 Model Configurations for Preconditioning Time Cost}
\label{tab:gpt2_config_efficiency}
\begin{tabular}{@{}lcccc@{}}
\toprule
\textbf{Model} & Params & Layers & Heads & $d_{\text{model}}$ \\
\midrule
GPT-2 60M    & 60M   & 6  & 10 & 640  \\
GPT-2 Small  & 125M  & 12 & 12 & 768  \\
GPT-2 200M   & 200M  & 16 & 14 & 896  \\
GPT-2 Medium & 355M  & 24 & 16 & 1024 \\
GPT-2 500M   & 500M  & 28 & 18 & 1152 \\
GPT-2 Large  & 770M  & 36 & 20 & 1280 \\
GPT-2 1.3B   & 1.3B & 44 & 24 & 1536 \\
GPT-2 XL     & 1.5B & 48 & 25 & 1600 \\
\bottomrule
\end{tabular}
\end{table}

\section{Hyperparameter Search for Pretraining Performance}
\label{hyper_search}

This section provides detailed hyperparameter search results for the pretraining experiments described in Section~\ref{sec:pretrain_performance}. We perform a systematic hyperparameter grid search for both \textsc{RMNP} and \textsc{Muon} across GPT-2 (Small and Medium) and LLaMA (60M, 130M, and 350M) models. Following the standard \textsc{Muon} training protocol, \textsc{RMNP} is integrated with \textsc{AdamW}, with the learning rate decoupled into $\text{lr}_{\text{AdamW}}$ and $\text{lr}_{\text{Matrix}}$. We fix $\text{lr}_{\text{AdamW}}$ and vary $\text{lr}_{\text{Matrix}}$ to evaluate its impact on convergence. For all LLaMA \textsc{RMNP} runs (60M / 130M / 350M / 1B), we further adopt a shared-LR convention $\text{lr}_{\text{AdamW}} = \text{lr}_{\text{Matrix}}$, i.e., the matrix LR reported in Tables~\ref{tab:llama_60m_hyperparam}--\ref{tab:llama_130m_hyperparam} and~\ref{tab:llama_350m_hyperparam} is also used as the \textsc{AdamW} LR for the non-matrix parameters in those rows; this differs from the GPT-2 protocol, where $\text{lr}_{\text{AdamW}}$ is held fixed independently of $\text{lr}_{\text{Matrix}}$. The results are summarized in Tables~\ref{tab:small_hyperparam} and~\ref{tab:medium_hyperparam} for GPT-2, and Tables~\ref{tab:llama_60m_hyperparam},~\ref{tab:llama_130m_hyperparam}, and~\ref{tab:llama_350m_hyperparam} for LLaMA. Due to compute constraints, we did not perform a full LR sweep for LLaMA-1B; instead, we use a fixed configuration: \textsc{AdamW} with $\text{lr}=6\times 10^{-4}$, \textsc{Muon} with $\text{lr}_{\text{AdamW}}=6\times 10^{-4}$ and $\text{lr}_{\text{Matrix}}=5\times 10^{-3}$, and \textsc{RMNP} with $\text{lr}_{\text{AdamW}}=\text{lr}_{\text{Matrix}}=5\times 10^{-3}$, all with weight decay $0.1$ and $\beta=(0.9, 0.95)$. All values reported are evaluation perplexity (lower is better). \fwedu{We also present GPT-2 experiments on FineWeb-Edu-100B~\cite{penedo2024finewebdatasetsdecantingweb}; see Tables~\ref{tab:gpt2_fwedu_config},~\ref{tab:fwedu_hyperparam}, and~\ref{tab:fwedu_results} in Appendix~\ref{appendix:model_config}.}

\subsection{Hyperparameter Settings}
\label{appendix:hyperparams}
This section provides detailed hyperparameter settings for the experiments described in Section~\ref{sec:exp_setup}.

\paragraph{\textsc{Muon}}
For \textsc{Muon}, we set the momentum to $0.95$ and weight decay to $0.1$. Following \citet{jordan2024muon}, we apply an RMS scaling coefficient to the learning rate:
\begin{equation}
    \eta = \text{lr}_{\text{Matrix}} \cdot \max\left(1, \sqrt{\frac{m}{n}}\right),
\end{equation}
where $m$ and $n$ denote the number of rows and columns of the parameter matrix, respectively. During hyperparameter search, we exclusively tune $\text{lr}_{\text{Matrix}}$. For \textsc{AdamW}, we set $\text{lr}_{\text{AdamW}} = 0.003$, $0.0015$, and $0.001$ for GPT-2 Small, medium, and large models, respectively.

\paragraph{\textsc{RMNP}}
To ensure a fair comparison, we adopt the same RMS scaling as \textsc{Muon}:
\begin{equation}
     \eta = \text{lr}_{\text{Matrix}} \cdot \max\left(1, \sqrt{\frac{m}{n}}\right),
\end{equation}
as well as identical hyperparameters for \textsc{AdamW}.

For GPT-2 experiments, the matrix optimizer is applied to all matrix parameters, including the LM head and token-embedding layers. For LLaMA experiments, the LM head and token-embedding parameters are handled by \textsc{AdamW} in the main results (Tables~\ref{tab:llama_60m_hyperparam},~\ref{tab:llama_130m_hyperparam}, and~\ref{tab:llama_350m_hyperparam}); an ablation on this choice is provided in Appendix~\ref{appendix:lm_head_embedding_ablation}.

\subsection{Model Configurations}
\label{appendix:model_config}
\label{appendix:finewebedu}
In this section, we present the model configurations and hyperparameters for GPT-2 (Table~\ref{tab:gpt2_config})\fwedu{, GPT-2 on FineWeb-Edu-100B (Table~\ref{tab:gpt2_fwedu_config}),} and LLaMA (Table~\ref{tab:llama_config}). All GPT-2 models are trained with a maximum sequence length of 1024 and a batch size of 480. All LLaMA models are trained with a maximum sequence length of 256 and a batch size of 512. GPT-2 Small and medium models are trained in parallel on 4 NVIDIA RTX Pro 6000 GPUs, while GPT-2 large models are trained on a single NVIDIA Blackwell B200 Tensor Core GPU. LLaMA-60M and LLaMA-130M are trained in parallel on 2 NVIDIA L40. LLaMA-350M is trained in parallel on 4 NVIDIA RTX Pro 6000. LLaMA-1B is trained in parallel on 8 NVIDIA GPUs. \fwedu{For FineWeb-Edu-100B~\cite{penedo2024finewebdatasetsdecantingweb}\footnote{\fwedu{We use the shuffled version by \citet{karpathy2024finewebedu100b}: \url{https://huggingface.co/datasets/karpathy/fineweb-edu-100b-shuffle}.}}, the GPT-2 Small, Medium, and Large configurations are identical to the OpenWebText setup, while the GPT-2 XLarge (1.5B) model is trained only on FineWeb-Edu-100B (no OpenWebText counterpart). Optimizer hyperparameters are listed in Table~\ref{tab:fwedu_hyperparam}, and evaluation results in Figure~\ref{fig:fwedu_results_bar} and Table~\ref{tab:fwedu_results}.}

\begin{table*}[!htbp]
\centering
\caption{GPT-2 Model Configurations and specified hyperparameters for OpenWebText experiments.}
\label{tab:gpt2_config}
\begin{tabular}{@{}lccccccccc@{}}
\toprule
\textbf{Model} & Params & Layer & Heads & $d_{\text{emb}}$ & Steps & Warm-up & Token Count & Batch Size & LR schedule \\
\midrule
GPT-2 Small  & 125M & 12 & 12 & 768 & 10K & 1K & 5B & 480 & Cosine \\
GPT-2 Medium & 355M & 24 & 16 & 1024 & 20K & 2K & 10B & 480 & Cosine\\
GPT-2 Large  & 770M & 36 & 20 & 1280 & 40K & 4K & 20B & 480 & Cosine\\
\bottomrule
\end{tabular}
\end{table*}

\begin{table*}[!htbp]
\centering
\caption{GPT-2 Model Configurations and specified hyperparameters for FineWeb-Edu-100B experiments.}
\label{tab:gpt2_fwedu_config}
\begin{tabular}{@{}lccccccccc@{}}
\toprule
\textbf{Model} & Params & Layer & Heads & $d_{\text{emb}}$ & Steps & Warm-up & Token Count & Batch Size & LR schedule \\
\midrule
GPT-2 Small  & 125M & 12 & 12 & 768  & 10K & 1K & 5B & 480 & Cosine \\
GPT-2 Medium & 355M & 24 & 16 & 1024 & 20K & 2K & 10B & 480 & Cosine \\
GPT-2 Large  & 770M & 36 & 20 & 1280 & 40K & 4K & 20B & 480 & Cosine \\
GPT-2 XLarge & 1.5B & 48 & 25 & 1600 & 50K & 5K & 25B & 480 & Cosine \\
\bottomrule
\end{tabular}
\end{table*}

\begin{table*}[!htbp]
\centering
\caption{Optimizer hyperparameters for GPT-2 pre-training on FineWeb-Edu-100B.}
\label{tab:fwedu_hyperparam}
\begin{tabular}{@{}lcccccc@{}}
\toprule
\textbf{Model} & Optimizer & $\text{lr}_{\text{AdamW}}$ & $\text{lr}_{\text{Matrix}}$ & Weight Decay & $\beta$ & Schedule \\
\midrule
\multirow{3}{*}{Small (125M)}
& \textsc{AdamW} & $6 \times 10^{-4}$ & --- & 0.1 & (0.9, 0.95) & Cosine \\
& \textsc{Muon} & $3 \times 10^{-3}$ & $2 \times 10^{-2}$ & 0.1 & (0.9, 0.95) & Cosine \\
& \textsc{RMNP} & $3 \times 10^{-3}$ & $2 \times 10^{-2}$ & 0.1 & (0.9, 0.95) & Cosine \\
\midrule
\multirow{3}{*}{Medium (355M)}
& \textsc{AdamW} & $3 \times 10^{-4}$ & --- & 0.1 & (0.9, 0.95) & Cosine \\
& \textsc{Muon} & $1.5 \times 10^{-3}$ & $1 \times 10^{-2}$ & 0.1 & (0.9, 0.95) & Cosine \\
& \textsc{RMNP} & $1.5 \times 10^{-3}$ & $1 \times 10^{-2}$ & 0.1 & (0.9, 0.95) & Cosine \\
\midrule
\multirow{3}{*}{Large (770M)}
& \textsc{AdamW} & $2 \times 10^{-4}$ & --- & 0.1 & (0.9, 0.95) & Cosine \\
& \textsc{Muon} & $1 \times 10^{-3}$ & $6.67 \times 10^{-3}$ & 0.1 & (0.9, 0.95) & Cosine \\
& \textsc{RMNP} & $1 \times 10^{-3}$ & $6.67 \times 10^{-3}$ & 0.1 & (0.9, 0.95) & Cosine \\
\midrule
\multirow{3}{*}{XLarge (1.5B)}
& \textsc{AdamW} & $2 \times 10^{-4}$ & --- & 0.1 & (0.9, 0.95) & Cosine \\
& \textsc{Muon} & $1 \times 10^{-3}$ & $6.67 \times 10^{-3}$ & 0.1 & (0.9, 0.95) & Cosine \\
& \textsc{RMNP} & $1 \times 10^{-3}$ & $2 \times 10^{-3}$ & 0.1 & (0.9, 0.95) & Cosine \\
\bottomrule
\end{tabular}
\end{table*}

\begin{table*}[!htbp]
\centering
\caption{LLaMA Model Configurations and specified hyperparameters.}
\scalebox{0.99}{
\begin{tabular}{@{}lccccccccc@{}}
\toprule
 Params & Hidden  & Intermediate &  Heads & Blocks &  Steps & Warm-up  & Token Count & Batch Size & LR schedule\\
\midrule
60M & 512 & 1376 & 8  & 8&  10K & 1K& 1B & 512 &Cosine \\
130M &  768& 2048 &  12& 12&   20K &2K& 2B & 512& Cosine \\
350M &  1024& 2736 &  16 & 24&   60K &6K& 6B & 512& Cosine \\
1B  & 2048 & 5461 &  32 & 24&   90K &9K& 9B & 512& Cosine \\
\bottomrule
\end{tabular}}
\label{tab:llama_config}
\end{table*}

\begin{table*}[!htbp]
\centering
\caption{Hyperparameter search on GPT-2 Small with \textsc{AdamW} learning rate fixed at $3 \times 10^{-3}$.}
\label{tab:small_hyperparam}
\begin{tabular}{lcccc}
\toprule
Matrix LR & 0.01 & 0.015 & 0.02 & 0.025 \\
\midrule
\textsc{Muon} & 23.62 & 26.74 & \textbf{22.86} & 22.87 \\
\midrule
Matrix LR & 0.002 & 0.003 & 0.004 & 0.005 \\
\midrule
\textsc{RMNP} & 23.58 & 22.95 & \textbf{22.82} & 26.42  \\
\bottomrule
\end{tabular}
\end{table*}

\begin{table*}[!htbp]
\centering
\caption{Hyperparameter search on GPT-2 Medium with \textsc{AdamW} learning rate fixed at $1.5 \times 10^{-3}$.}
\label{tab:medium_hyperparam}
\begin{tabular}{lcccc}
\toprule
Matrix LR & 0.005 & 0.01 & 0.02 & 0.03 \\
\midrule
\textsc{Muon} & 18.33 & 18.26 & \textbf{17.38} & 17.44 \\
\midrule
Matrix LR & 0.001 & 0.002 & 0.003 & 0.005 \\
\midrule
\textsc{RMNP} & 18.58 & \textbf{17.31} & 17.42 & 17.88  \\
\bottomrule
\end{tabular}
\end{table*}

\clearpage

\begin{table*}[!htbp]
\centering
\caption{Hyperparameter search on LLaMA-60M (LM head and token-embedding parameters handled by \textsc{AdamW}). Validation perplexity is reported.}
\label{tab:llama_60m_hyperparam}
\begin{tabular}{lccccc}
\toprule
Matrix LR & 0.005 & 0.01 & 0.02 & 0.03 & 0.04 \\
\midrule
\textsc{Muon} & 29.90 & \textbf{29.58} & 30.46 & 30.49 & 30.03 \\
\midrule
Matrix LR & 0.001 & 0.004 & 0.005 & 0.01 & 0.02 \\
\midrule
\textsc{RMNP} & 31.00 & 28.99 & \textbf{28.95} & 29.26 & 29.64 \\
\midrule
Matrix LR & 0.005 & 0.01 & 0.02 & 0.03 & 0.04 \\
\midrule
\textsc{Shampoo} & 31.04 & 30.69 & 31.07 & \textbf{29.74} & 30.61 \\
\midrule
Matrix LR & 0.001 & 0.002 & 0.003 & 0.004 & 0.005 \\
\midrule
\textsc{SOAP} & 30.85 & 29.30 & \textbf{29.14} & 29.36 & 29.57 \\
\bottomrule
\end{tabular}
\end{table*}

\begin{table*}[!htbp]
\centering
\caption{Hyperparameter search on LLaMA-130M (LM head and token-embedding parameters handled by \textsc{AdamW}). Validation perplexity is reported.}
\label{tab:llama_130m_hyperparam}
\begin{tabular}{lcccc}
\toprule
Matrix LR & 0.005 & 0.01 & 0.02 & 0.03  \\
\midrule
\textsc{Muon} & 22.51 & \textbf{22.42} & 22.47 & 22.51  \\
\midrule
Matrix LR & 0.01 & 0.02 & 0.03 & 0.04  \\
\midrule
\textsc{RMNP} & 22.42 & 22.49 & \textbf{22.14} & 23.31  \\
\midrule
Matrix LR & 0.005 & 0.01 & 0.03 & 0.04  \\
\midrule
\textsc{Shampoo} & 23.22 & 22.70 & \textbf{22.69} & 23.49  \\
\midrule
Matrix LR & 0.001 & 0.002 & 0.003 & 0.005  \\
\midrule
\textsc{SOAP} & 23.13 & \textbf{22.61} & 22.78 & 23.11  \\
\bottomrule
\end{tabular}
\end{table*}

\begin{table*}[!htbp]
\centering
\caption{Hyperparameter search on LLaMA-350M (LM head and token-embedding parameters handled by \textsc{AdamW}). Validation perplexity is reported.}
\label{tab:llama_350m_hyperparam}
\begin{tabular}{lccc}
\toprule
Matrix LR & 0.003 & 0.004 & 0.005 \\
\midrule
\textsc{Muon} & 17.01 & \textbf{16.87} & 16.89 \\
\midrule
Matrix LR & 0.003 & 0.004 & 0.005 \\
\midrule
\textsc{RMNP} & 17.02 & 16.86 & \textbf{16.85} \\
\bottomrule
\end{tabular}
\end{table*}

\subsection{Extended Training Budget}
\label{appendix:extended_training}

To verify that the advantage of \textsc{RMNP} over \textsc{Muon} and \textsc{AdamW} persists at longer training horizons, we additionally extend the training budget to $2\times$ the standard length for three model-dataset combinations: GPT-2 Small on OpenWebText (20K steps), LLaMA-60M on C4 (20K steps), and LLaMA-130M on C4 (40K steps). Final validation perplexity is reported in Table~\ref{tab:extended_training}. \textsc{RMNP} achieves the lowest perplexity in every cell, indicating that its advantage is not a short-horizon artifact.

\begin{table*}[!htbp]
\centering
\caption{Final validation PPL (\textdownarrow) under an extended training budget ($2\times$ standard). Lower is better.}
\label{tab:extended_training}
\begin{tabular}{@{}lccc@{}}
\toprule
Optimizer & LLaMA 60M & LLaMA 130M & GPT-2 Small (OWT) \\
\midrule
\textsc{AdamW} & 28.23 & 21.35 & 20.97 \\
\textsc{Muon}  & 27.03 & 20.84 & 20.88 \\
\textsc{RMNP}  & \textbf{26.44} & \textbf{20.53} & \textbf{20.41} \\
\bottomrule
\end{tabular}
\end{table*}

\subsection{LM Head and Embedding Ablation}
\label{appendix:lm_head_embedding_ablation}

We additionally study the effect of extending the matrix-aware optimizer to cover the LM head and token-embedding parameters (rather than letting \textsc{AdamW} handle them). Tables~\ref{tab:llama_60m_lmhead_hyperparam} and~\ref{tab:llama_130m_lmhead_hyperparam} report the LR-sweep results for \textsc{Muon} and \textsc{RMNP} on LLaMA-60M and LLaMA-130M when LM head and embedding parameters are included in the matrix-optimizer parameter group.

\begin{table*}[!htbp]
\centering
\caption{Hyperparameter search on LLaMA-60M with LM head and embedding parameters optimized by the matrix optimizer. Validation perplexity is reported.}
\label{tab:llama_60m_lmhead_hyperparam}
\begin{tabular}{lccccc}
\toprule
Matrix LR & 0.005 & 0.01 & 0.02 & 0.03 & 0.04 \\
\midrule
\textsc{Muon} & 30.41 & 29.49 & 29.63 & \textbf{29.38} & 30.57 \\
\midrule
Matrix LR & 0.001 & 0.004 & 0.005 & 0.01 & 0.02 \\
\midrule
\textsc{RMNP} & 34.92 & 29.56 & 29.28 & \textbf{29.03} & 31.45 \\
\bottomrule
\end{tabular}
\end{table*}

\begin{table*}[!htbp]
\centering
\caption{Hyperparameter search on LLaMA-130M with LM head and embedding parameters optimized by the matrix optimizer. Validation perplexity is reported.}
\label{tab:llama_130m_lmhead_hyperparam}
\begin{tabular}{lcccc}
\toprule
Matrix LR & 0.005 & 0.01 & 0.02 & 0.03 \\
\midrule
\textsc{Muon} & 22.89 & \textbf{22.55} & 22.80 & 22.87 \\
\midrule
Matrix LR & 0.01 & 0.02 & 0.03 & 0.04 \\
\midrule
\textsc{RMNP} & 22.16 & 22.11 & \textbf{22.06} & 23.62 \\
\bottomrule
\end{tabular}
\end{table*}

Overall, as shown in Tables~\ref{tab:llama_60m_lmhead_hyperparam} and~\ref{tab:llama_130m_lmhead_hyperparam}, including the LM head and token-embedding parameters in the matrix-optimizer group has a negligible effect on final perplexity: the differences across all settings are within 0.13 PPL and show no consistent trend across model scales or optimizers. For the GPT-2 experiments reported in the main body, the LM head and embedding parameters are optimized together with the other matrix parameters using the matrix optimizer.

\section{Full Training Curves}
\label{appendix:training_curves}

This section presents the complete set of training-loss, validation-loss, and gradient clip-rate curves for every model-dataset combination evaluated in this paper, comparing \textsc{AdamW}, \textsc{Muon}, and \textsc{RMNP}. In every plot \textsc{RMNP} is drawn on top so that it is never occluded by the other two curves. All curves use the canonical hyperparameters reported in Appendix~\ref{appendix:model_config}.

\subsection{Final Validation Perplexity Summary}
\label{appendix:final_ppl_summary}

Before presenting the full training curves, we summarize the final validation perplexity for the three main pre-training settings as bar charts paired with the corresponding numeric tables. \textsc{RMNP} attains the lowest final perplexity in every cell.

\begin{figure}[!htbp]
    \centering
    \includegraphics[width=0.65\linewidth]{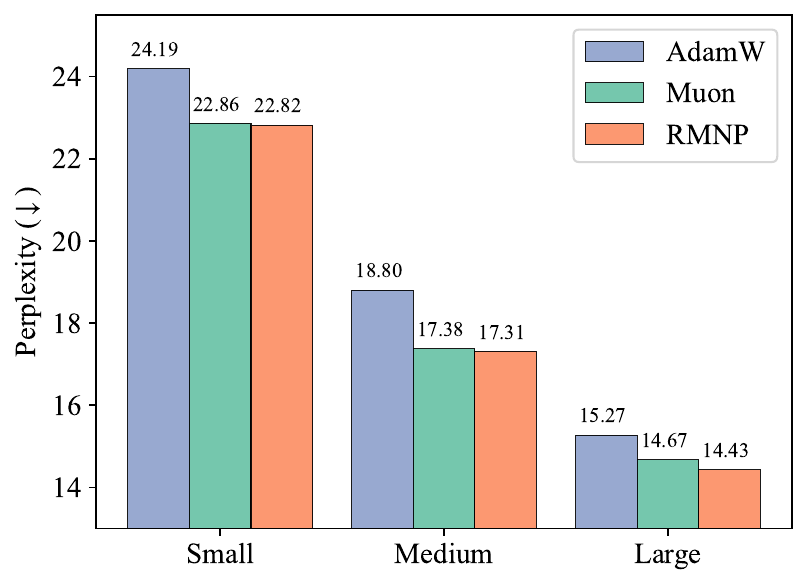}
    \caption{Final validation perplexity ($\downarrow$) on OpenWebText for GPT-2 Small, Medium, and Large. Numeric values are reported in Table~\ref{tab:owt_results}.}
    \label{fig:owt_results_bar}
\end{figure}

\begin{table}[!htbp]
\centering
\caption{Final validation perplexity ($\downarrow$) on OpenWebText for GPT-2 models.}
\label{tab:owt_results}
\begin{tabular}{@{}lccc@{}}
\toprule
 & Small (125M) & Medium (355M) & Large (770M) \\
\midrule
\textsc{AdamW} & 24.19 & 18.80 & 15.27 \\
\textsc{Muon}  & 22.86 & 17.38 & 14.67 \\
\textsc{RMNP}  & \textbf{22.82} & \textbf{17.31} & \textbf{14.43} \\
\bottomrule
\end{tabular}
\end{table}

\begin{figure}[!htbp]
    \centering
    \includegraphics[width=0.7\linewidth]{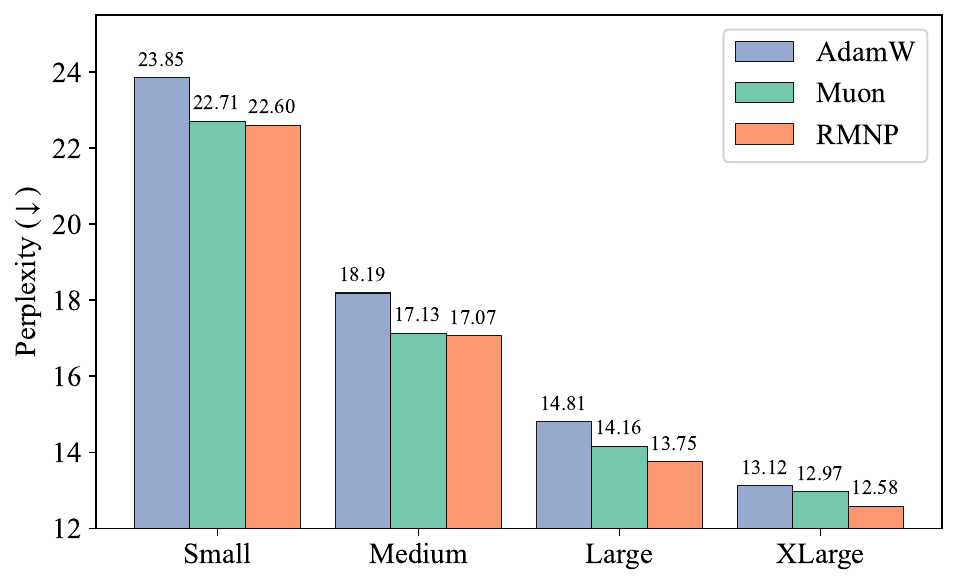}
    \caption{Final validation perplexity ($\downarrow$) on FineWeb-Edu-100B for GPT-2 Small, Medium, Large, and XLarge. Numeric values are reported in Table~\ref{tab:fwedu_results}.}
    \label{fig:fwedu_results_bar}
\end{figure}

\begin{table}[!htbp]
\centering
\caption{Final validation perplexity ($\downarrow$) on FineWeb-Edu-100B for GPT-2 models.}
\label{tab:fwedu_results}
\begin{tabular}{@{}lcccc@{}}
\toprule
 & Small (125M) & Medium (355M) & Large (770M) & XLarge (1.5B) \\
\midrule
\textsc{AdamW} & 23.85 & 18.19 & 14.81 & 13.12 \\
\textsc{Muon}  & 22.71 & 17.13 & 14.16 & 12.97 \\
\textsc{RMNP}  & \textbf{22.60} & \textbf{17.07} & \textbf{13.75} & \textbf{12.58} \\
\bottomrule
\end{tabular}
\end{table}

\begin{figure}[!htbp]
    \centering
    \includegraphics[width=0.7\linewidth]{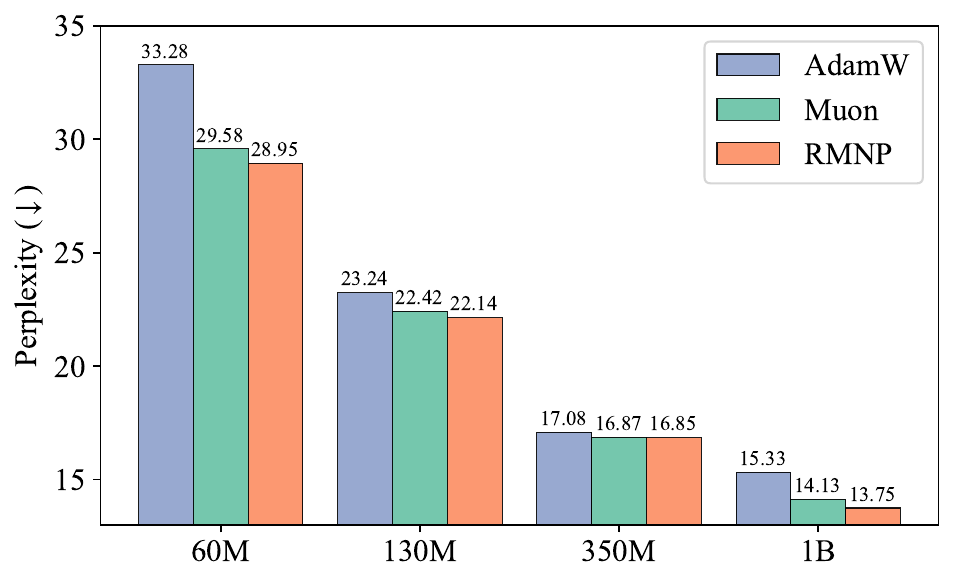}
    \caption{Final validation perplexity ($\downarrow$) on C4 for LLaMA 60M, 130M, 350M, and 1B. Numeric values are reported in Table~\ref{tab:llama_c4_results}.}
    \label{fig:llama_c4_results_bar}
\end{figure}

\begin{table}[!htbp]
\centering
\caption{Final validation perplexity ($\downarrow$) on C4 for LLaMA models.}
\label{tab:llama_c4_results}
\begin{tabular}{@{}lcccc@{}}
\toprule
 & 60M & 130M & 350M & 1B \\
\midrule
\textsc{AdamW} & 33.28 & 23.24 & 17.08 & 15.33 \\
\textsc{Muon}  & 29.58 & 22.42 & 16.87 & 14.13 \\
\textsc{RMNP}  & \textbf{28.95} & \textbf{22.14} & \textbf{16.85} & \textbf{13.75} \\
\bottomrule
\end{tabular}
\end{table}

\subsection{GPT-2 on OpenWebText}
\label{appendix:curves_owt}

Figures~\ref{fig:owt_small_curves}--\ref{fig:owt_large_curves} show the training and validation loss for \textsc{GPT-2 Small}, \textsc{Medium}, and \textsc{Large} pre-trained on OpenWebText. Across all three scales \textsc{RMNP} consistently matches or slightly improves upon \textsc{Muon}, while both clearly outperform \textsc{AdamW}.

\begin{figure}[!htbp]
    \centering
    \begin{subfigure}[t]{0.48\linewidth}
        \centering
        \includegraphics[width=\linewidth]{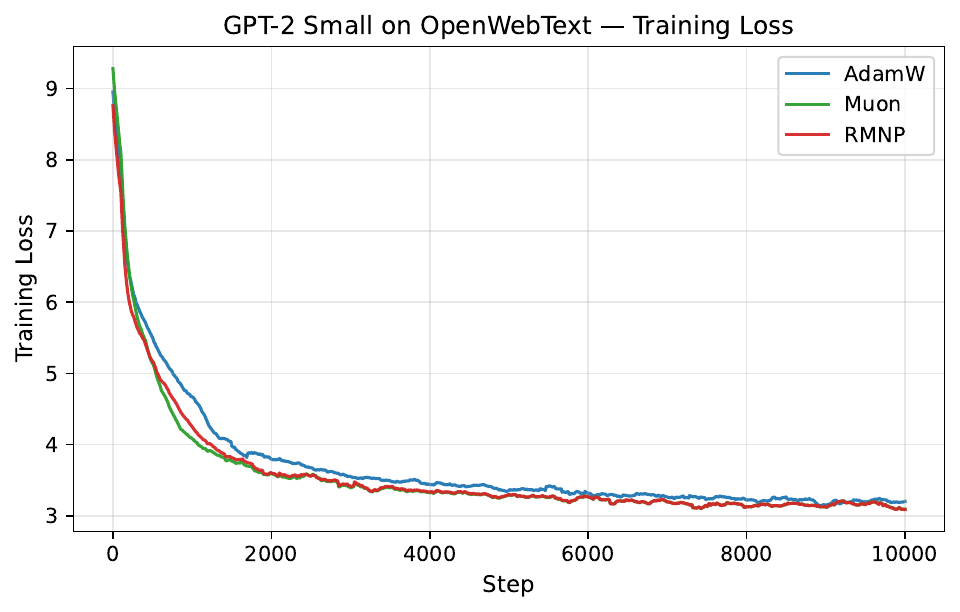}
        \caption{Training Loss}
    \end{subfigure}\hfill
    \begin{subfigure}[t]{0.48\linewidth}
        \centering
        \includegraphics[width=\linewidth]{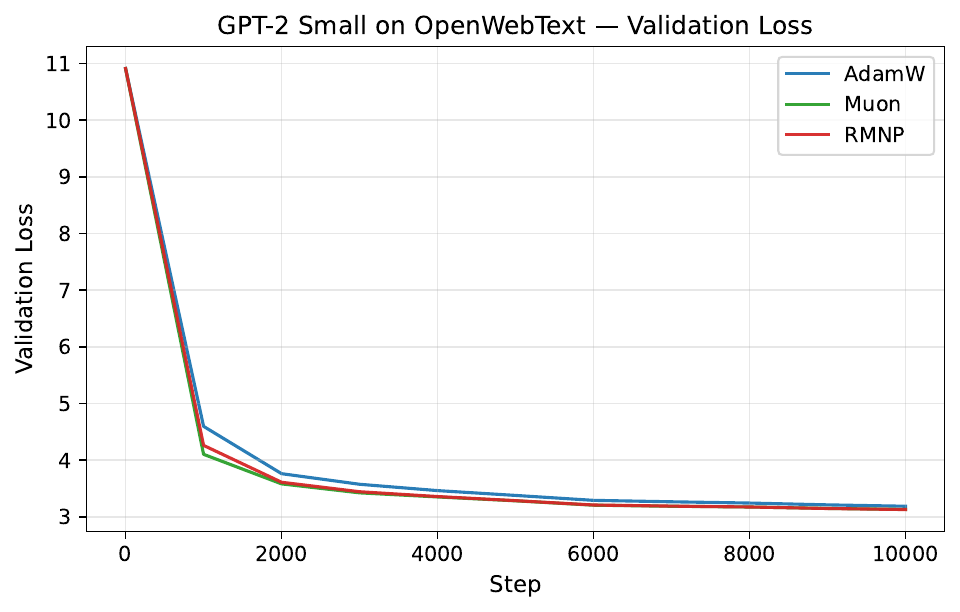}
        \caption{Validation Loss}
    \end{subfigure}
    \caption{\textsc{GPT-2 Small} (125M) on OpenWebText. Training loss is smoothed with a 20-step rolling window. \textsc{RMNP} ends with the lowest training and validation loss.}
    \label{fig:owt_small_curves}
\end{figure}

\begin{figure}[!htbp]
    \centering
    \begin{subfigure}[t]{0.48\linewidth}
        \centering
        \includegraphics[width=\linewidth]{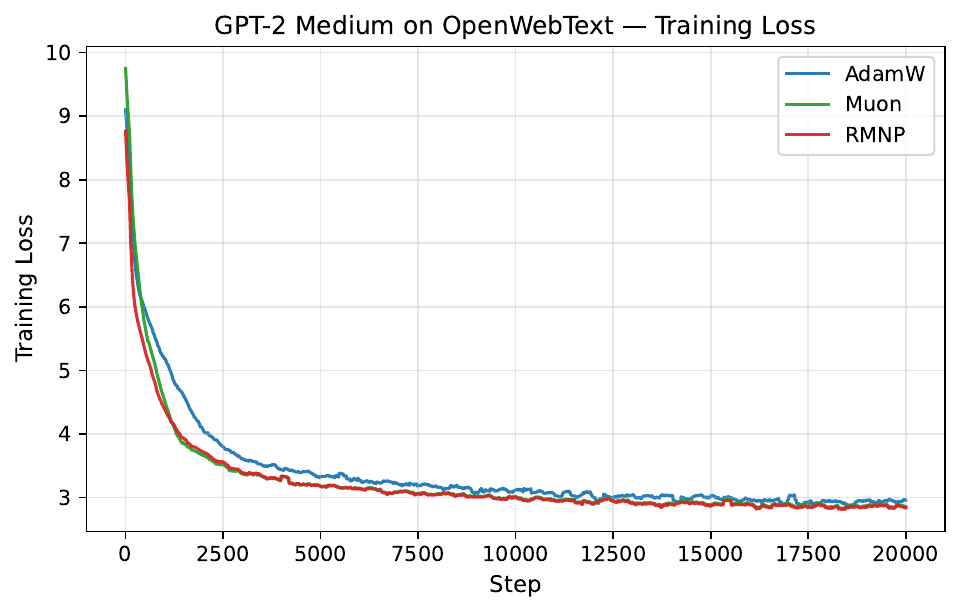}
        \caption{Training Loss}
    \end{subfigure}\hfill
    \begin{subfigure}[t]{0.48\linewidth}
        \centering
        \includegraphics[width=\linewidth]{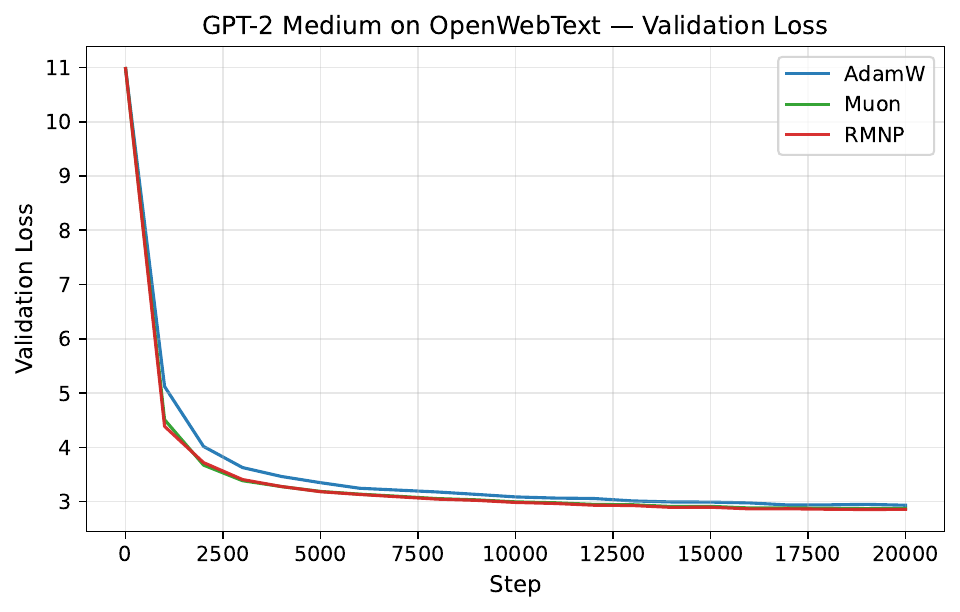}
        \caption{Validation Loss}
    \end{subfigure}
    \caption{\textsc{GPT-2 Medium} (355M) on OpenWebText. \textsc{RMNP} achieves the lowest validation loss among the three optimizers.}
    \label{fig:owt_medium_curves}
\end{figure}

\begin{figure}[!htbp]
    \centering
    \begin{subfigure}[t]{0.48\linewidth}
        \centering
        \includegraphics[width=\linewidth]{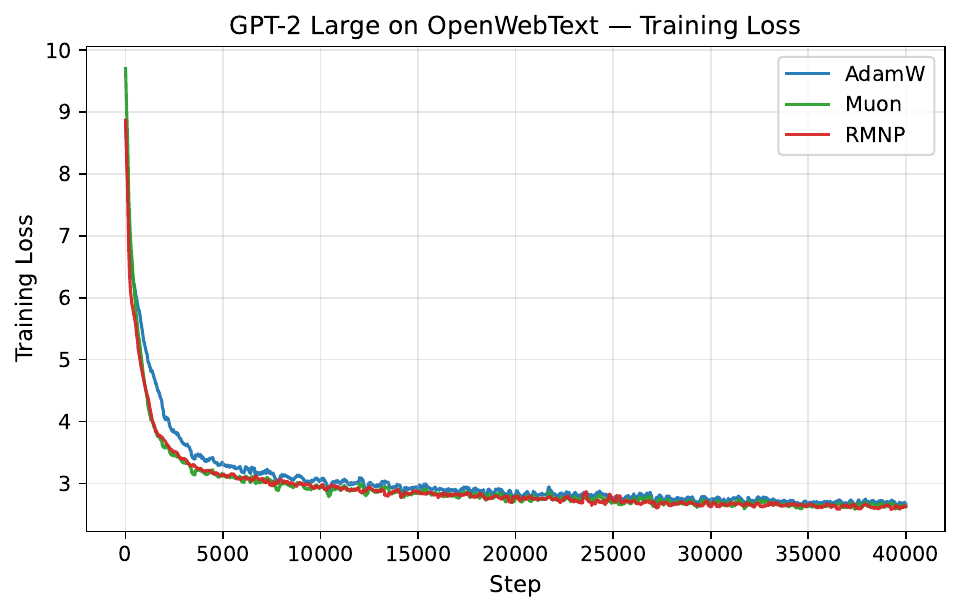}
        \caption{Training Loss}
    \end{subfigure}\hfill
    \begin{subfigure}[t]{0.48\linewidth}
        \centering
        \includegraphics[width=\linewidth]{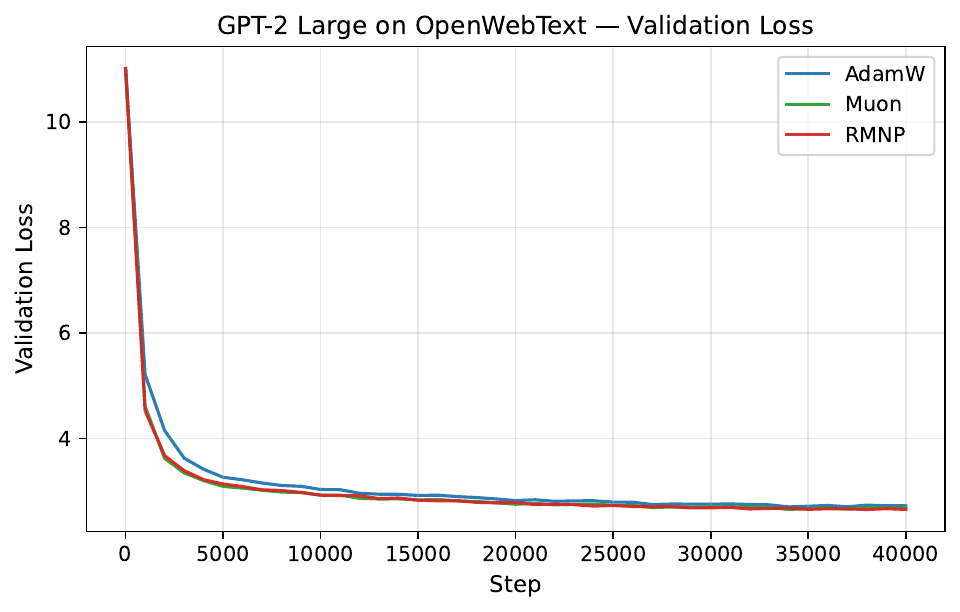}
        \caption{Validation Loss}
    \end{subfigure}
    \caption{\textsc{GPT-2 Large} (770M) on OpenWebText. \textsc{RMNP}'s lead over \textsc{Muon} grows with model scale.}
    \label{fig:owt_large_curves}
\end{figure}

\subsection{GPT-2 on FineWeb-Edu-100B}
\label{appendix:curves_fwedu}

Figures~\ref{fig:fwedu_small_curves}--\ref{fig:fwedu_xlarge_curves} present the training and validation loss curves for \textsc{GPT-2 Small}, \textsc{Medium}, \textsc{Large}, and \textsc{XLarge} pre-trained on FineWeb-Edu-100B. Across all four scales \textsc{RMNP} again matches or surpasses \textsc{Muon} and clearly outperforms \textsc{AdamW}, demonstrating that the trend observed on OpenWebText extends to a more competitive corpus and a larger token budget.

\begin{figure}[!htbp]
    \centering
    \begin{subfigure}[t]{0.48\linewidth}
        \centering
        \includegraphics[width=\linewidth]{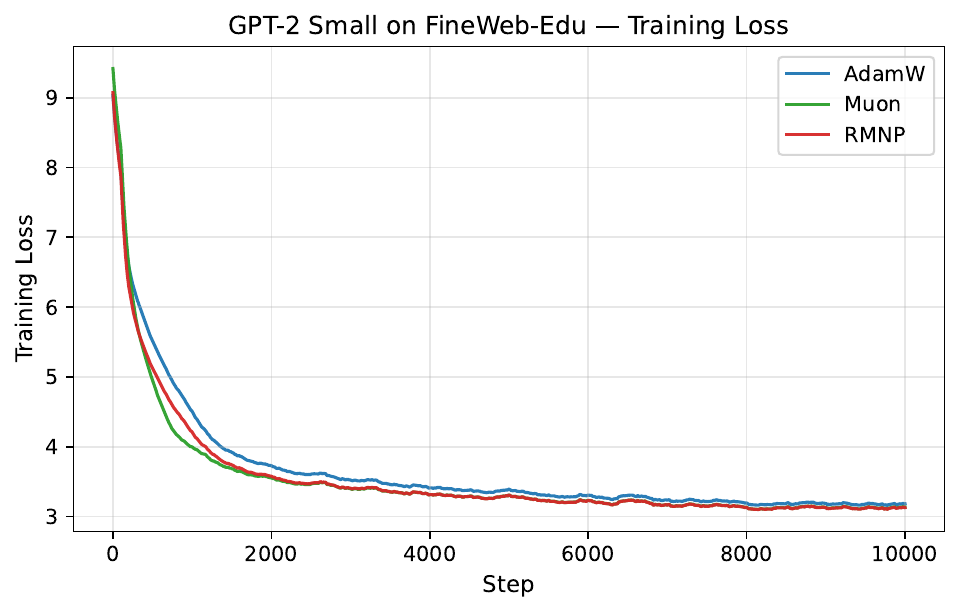}
        \caption{Training Loss}
    \end{subfigure}\hfill
    \begin{subfigure}[t]{0.48\linewidth}
        \centering
        \includegraphics[width=\linewidth]{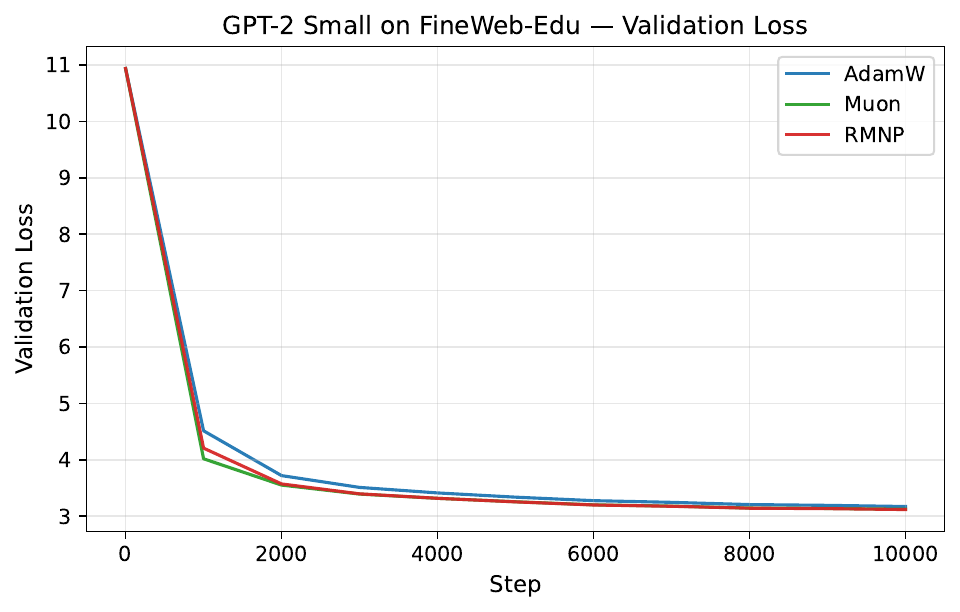}
        \caption{Validation Loss}
    \end{subfigure}
    \caption{\textsc{GPT-2 Small} (125M) on FineWeb-Edu-100B. \textsc{RMNP} attains the lowest training and validation loss.}
    \label{fig:fwedu_small_curves}
\end{figure}

\begin{figure}[!htbp]
    \centering
    \begin{subfigure}[t]{0.48\linewidth}
        \centering
        \includegraphics[width=\linewidth]{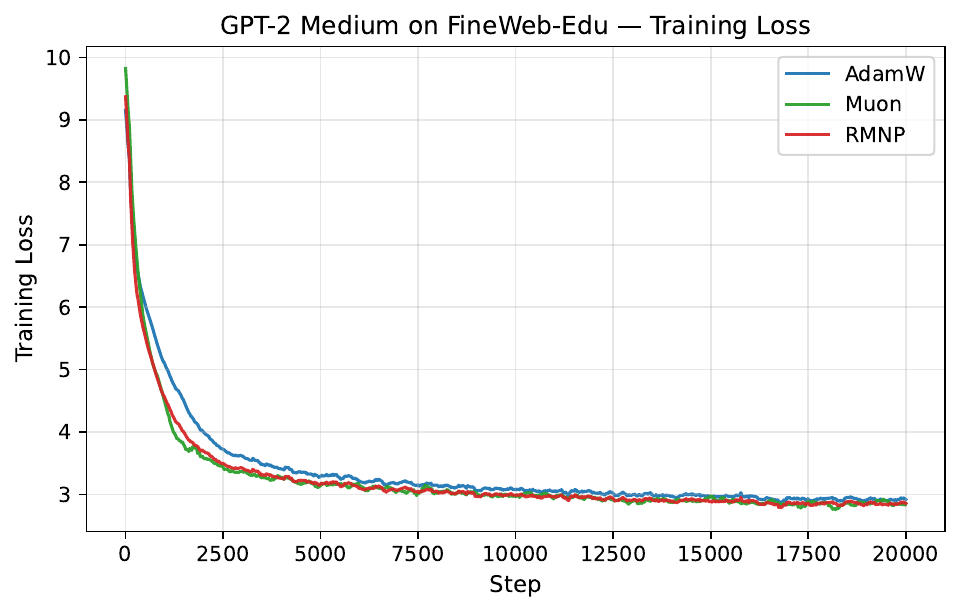}
        \caption{Training Loss}
    \end{subfigure}\hfill
    \begin{subfigure}[t]{0.48\linewidth}
        \centering
        \includegraphics[width=\linewidth]{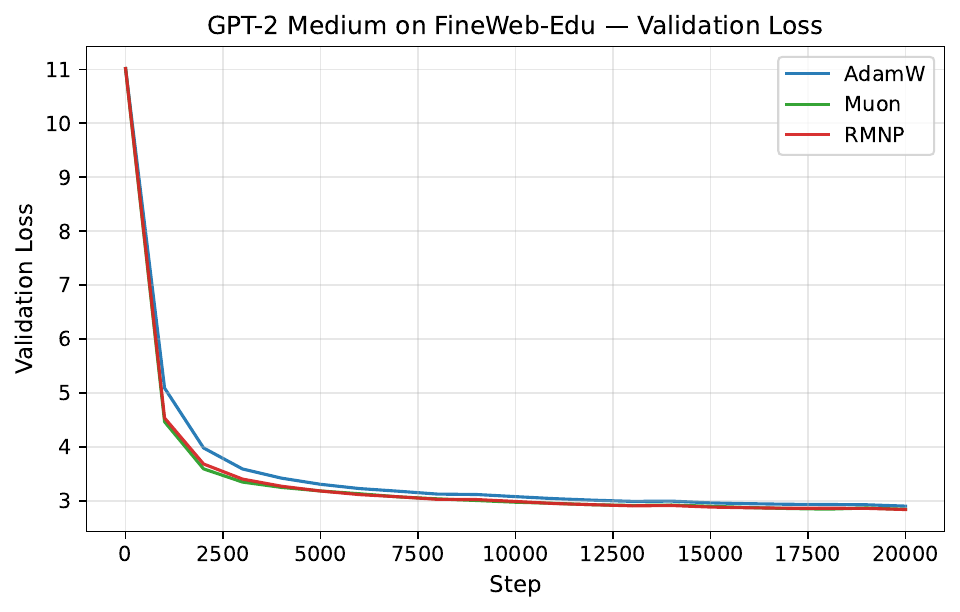}
        \caption{Validation Loss}
    \end{subfigure}
    \caption{\textsc{GPT-2 Medium} (355M) on FineWeb-Edu-100B. \textsc{RMNP} maintains a slight but consistent edge over \textsc{Muon} on validation loss while \textsc{AdamW} lags throughout training.}
    \label{fig:fwedu_medium_curves}
\end{figure}

\begin{figure}[!htbp]
    \centering
    \begin{subfigure}[t]{0.48\linewidth}
        \centering
        \includegraphics[width=\linewidth]{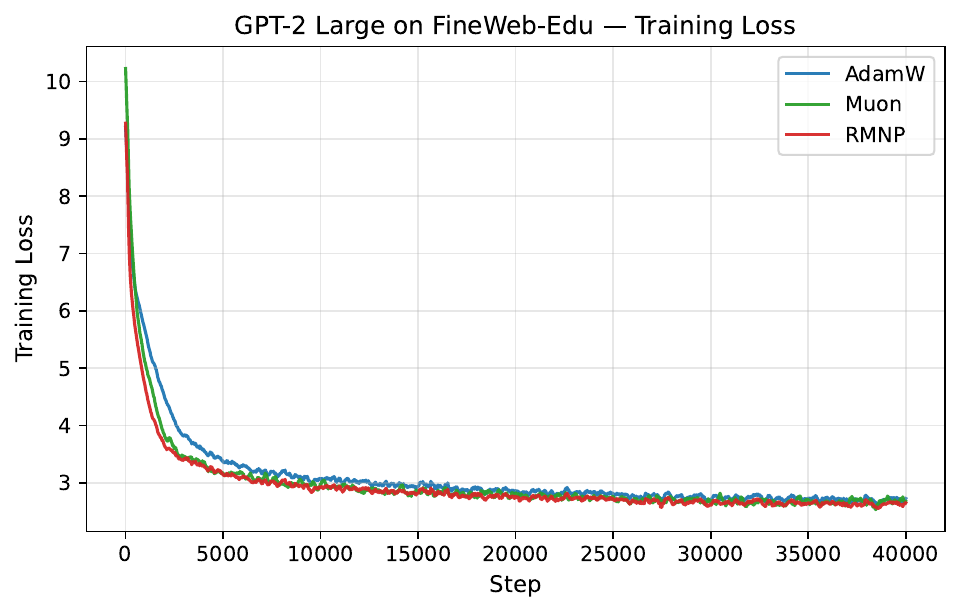}
        \caption{Training Loss}
    \end{subfigure}\hfill
    \begin{subfigure}[t]{0.48\linewidth}
        \centering
        \includegraphics[width=\linewidth]{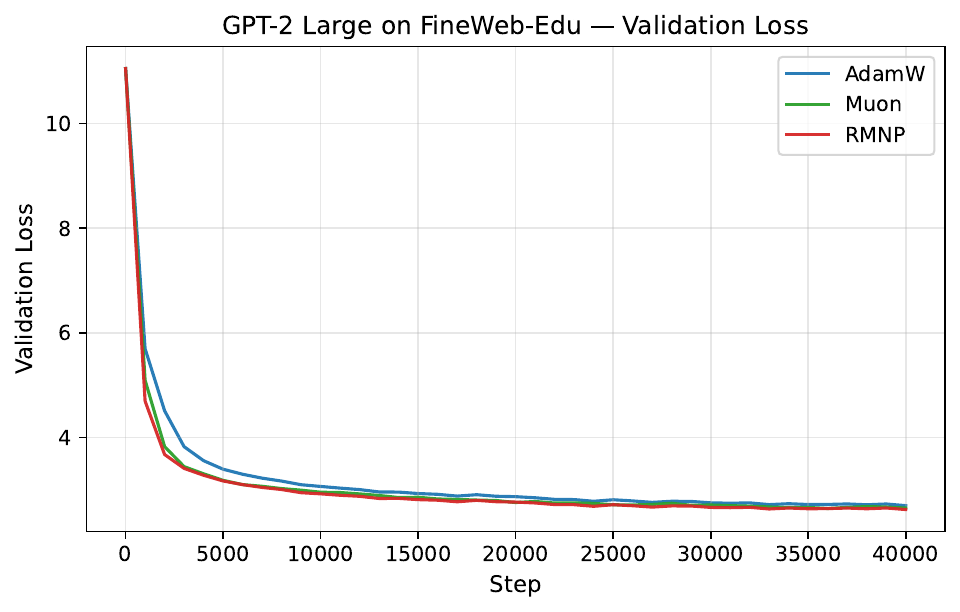}
        \caption{Validation Loss}
    \end{subfigure}
    \caption{\textsc{GPT-2 Large} (770M) on FineWeb-Edu-100B. \textsc{RMNP}'s lead over \textsc{Muon} grows with model scale, while \textsc{AdamW} converges to a noticeably higher validation loss.}
    \label{fig:fwedu_large_curves}
\end{figure}

\begin{figure}[!htbp]
    \centering
    \begin{subfigure}[t]{0.48\linewidth}
        \centering
        \includegraphics[width=\linewidth]{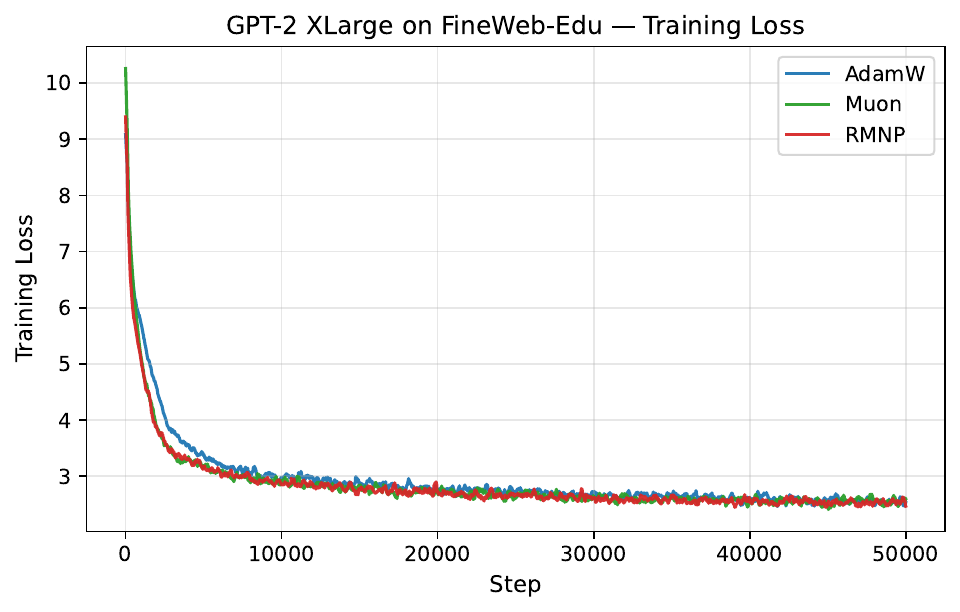}
        \caption{Training Loss}
    \end{subfigure}\hfill
    \begin{subfigure}[t]{0.48\linewidth}
        \centering
        \includegraphics[width=\linewidth]{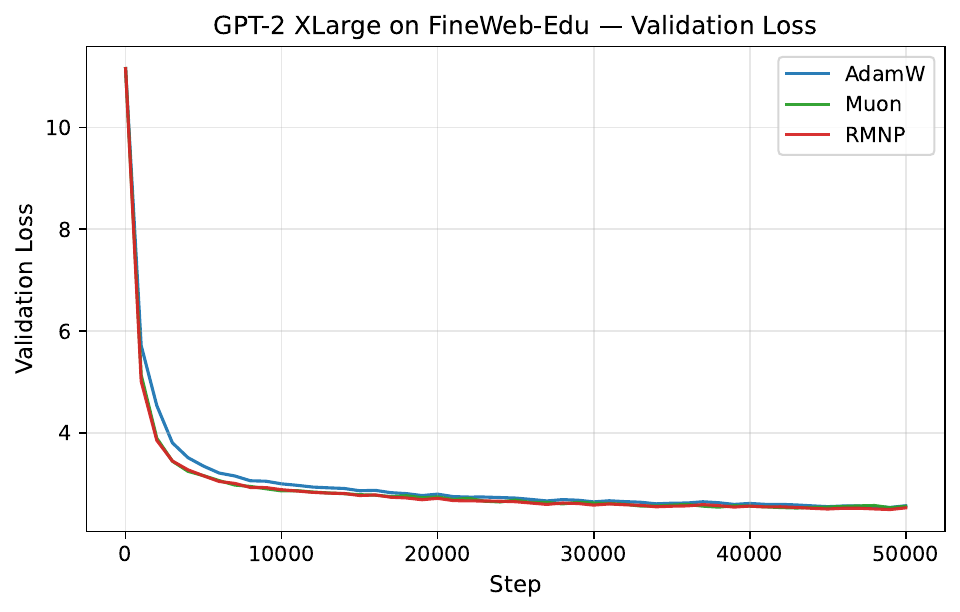}
        \caption{Validation Loss}
    \end{subfigure}
    \caption{\textsc{GPT-2 XLarge} (1.5B) on FineWeb-Edu-100B. \textsc{RMNP} continues to track \textsc{Muon} closely and surpasses it in late training, while delivering an order-of-magnitude reduction in preconditioning wall-clock cost (Appendix~\ref{appendix:wall_clock}).}
    \label{fig:fwedu_xlarge_curves}
\end{figure}

\subsection{LLaMA on C4}
\label{appendix:curves_c4}

Figures~\ref{fig:c4_60m_curves}--\ref{fig:c4_1b_curves} report the training and validation loss curves for the four LLaMA scales pretrained on C4. \textsc{RMNP} consistently delivers a slight but stable improvement over \textsc{Muon} across all sizes, and the gap between matrix-aware optimizers and \textsc{AdamW} widens as model scale grows.

\begin{figure}[!htbp]
    \centering
    \begin{subfigure}[t]{0.48\linewidth}
        \centering
        \includegraphics[width=\linewidth]{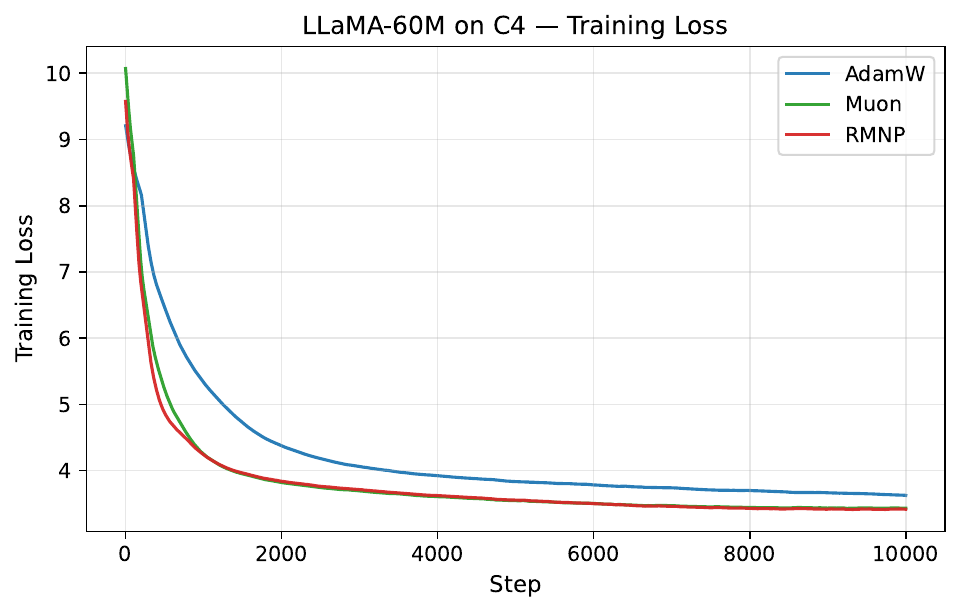}
        \caption{Training Loss}
    \end{subfigure}\hfill
    \begin{subfigure}[t]{0.48\linewidth}
        \centering
        \includegraphics[width=\linewidth]{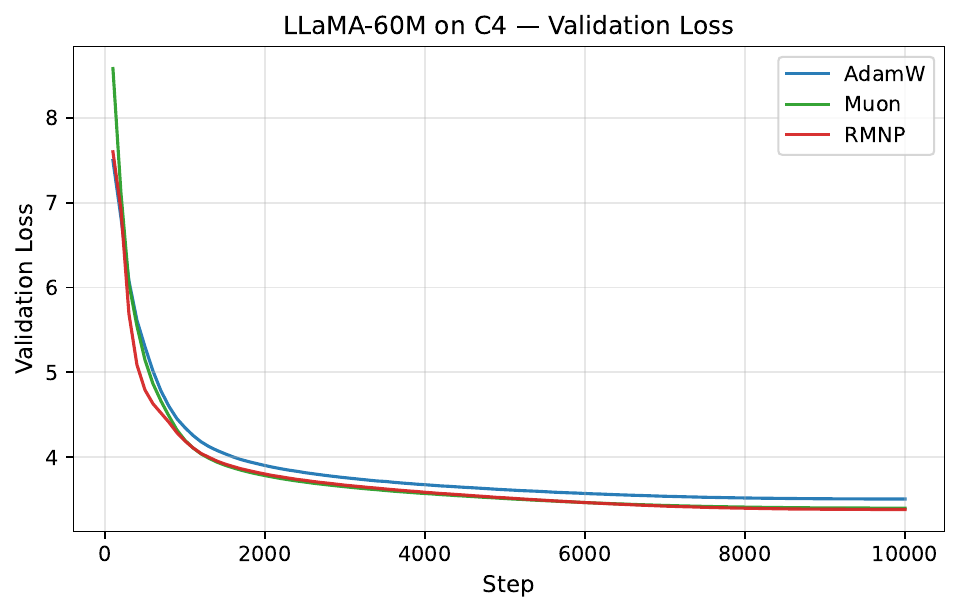}
        \caption{Validation Loss}
    \end{subfigure}
    \caption{\textsc{LLaMA-60M} on C4. The available \textsc{AdamW} log extends beyond the canonical training horizon and has been clipped to match \textsc{Muon} and \textsc{RMNP} on a shared x-range. \textsc{RMNP} achieves the lowest validation loss; \textsc{Muon} is close behind, while \textsc{AdamW} converges to a clearly higher value.}
    \label{fig:c4_60m_curves}
\end{figure}

\begin{figure}[!htbp]
    \centering
    \begin{subfigure}[t]{0.48\linewidth}
        \centering
        \includegraphics[width=\linewidth]{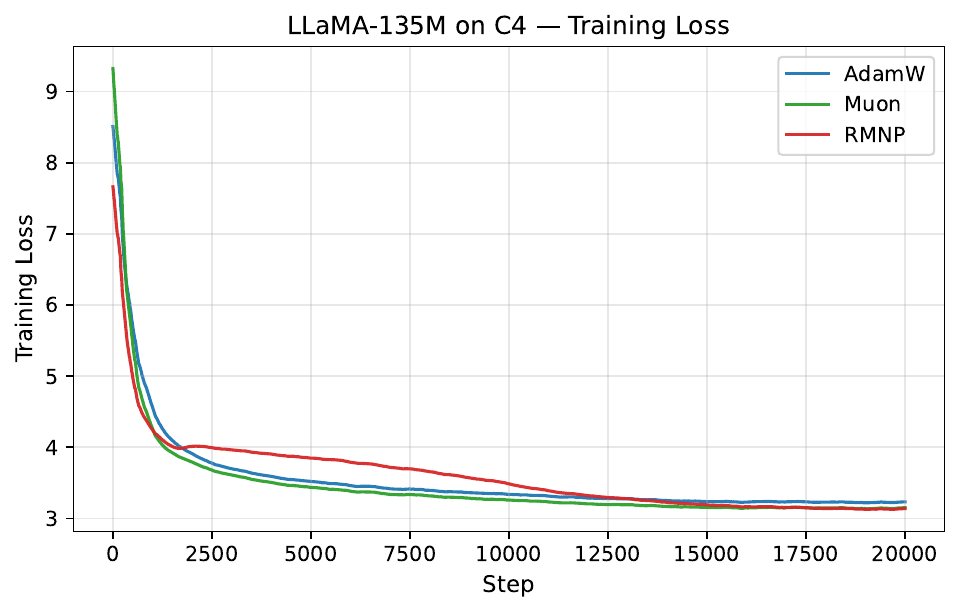}
        \caption{Training Loss}
    \end{subfigure}\hfill
    \begin{subfigure}[t]{0.48\linewidth}
        \centering
        \includegraphics[width=\linewidth]{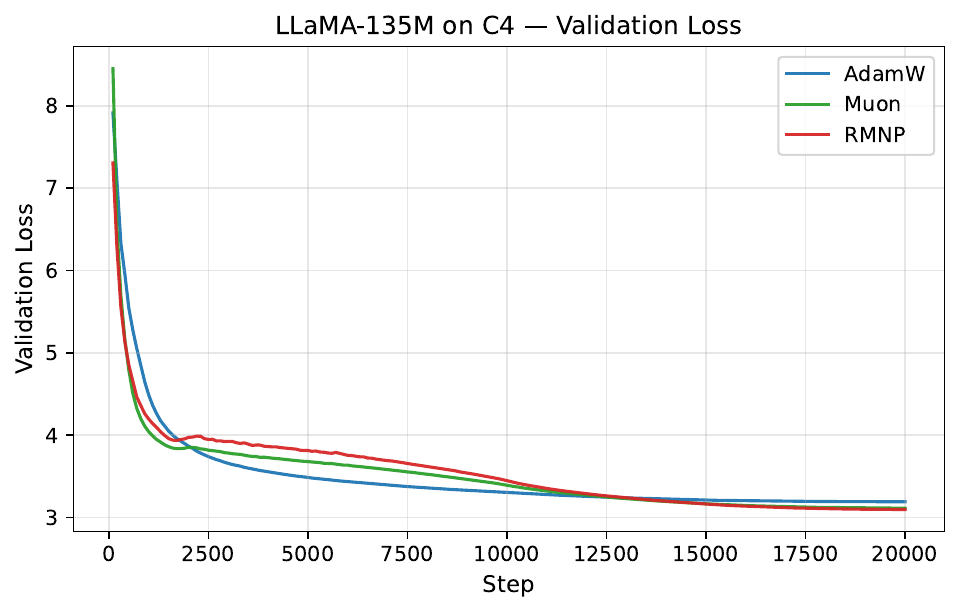}
        \caption{Validation Loss}
    \end{subfigure}
    \caption{\textsc{LLaMA-130M} on C4. \textsc{RMNP} outperforms both baselines in validation loss throughout training.}
    \label{fig:c4_130m_curves}
\end{figure}

\begin{figure}[!htbp]
    \centering
    \begin{subfigure}[t]{0.48\linewidth}
        \centering
        \includegraphics[width=\linewidth]{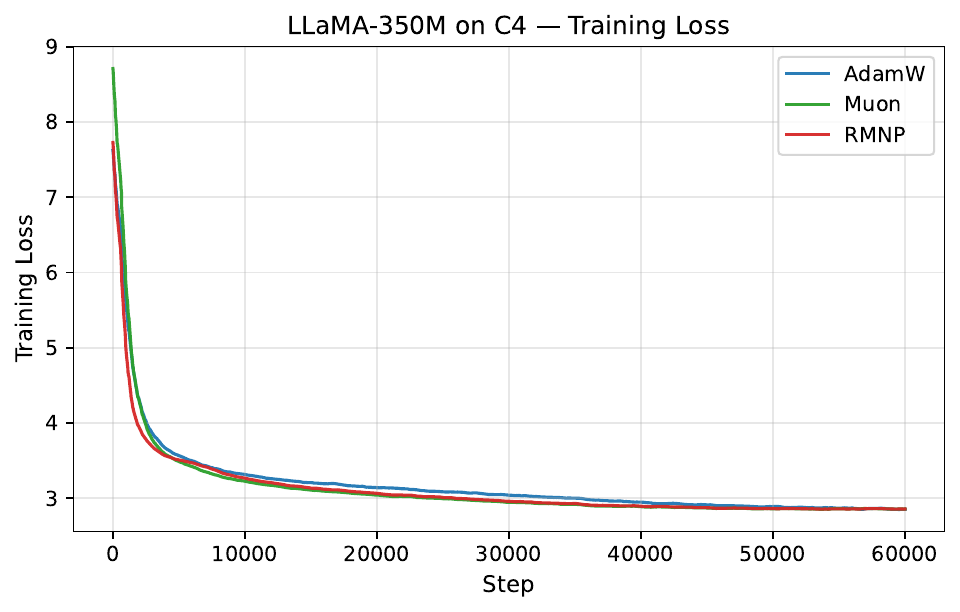}
        \caption{Training Loss}
    \end{subfigure}\hfill
    \begin{subfigure}[t]{0.48\linewidth}
        \centering
        \includegraphics[width=\linewidth]{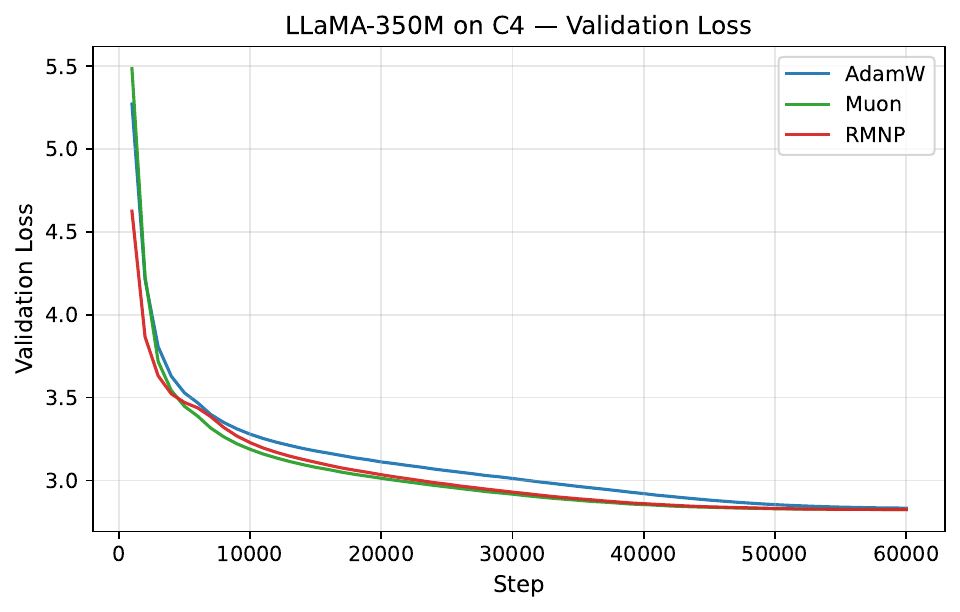}
        \caption{Validation Loss}
    \end{subfigure}
    \caption{\textsc{LLaMA-350M} on C4. \textsc{RMNP} matches \textsc{Muon} on training loss and edges ahead on validation loss in late training.}
    \label{fig:c4_350m_curves}
\end{figure}

\begin{figure}[!htbp]
    \centering
    \begin{subfigure}[t]{0.48\linewidth}
        \centering
        \includegraphics[width=\linewidth]{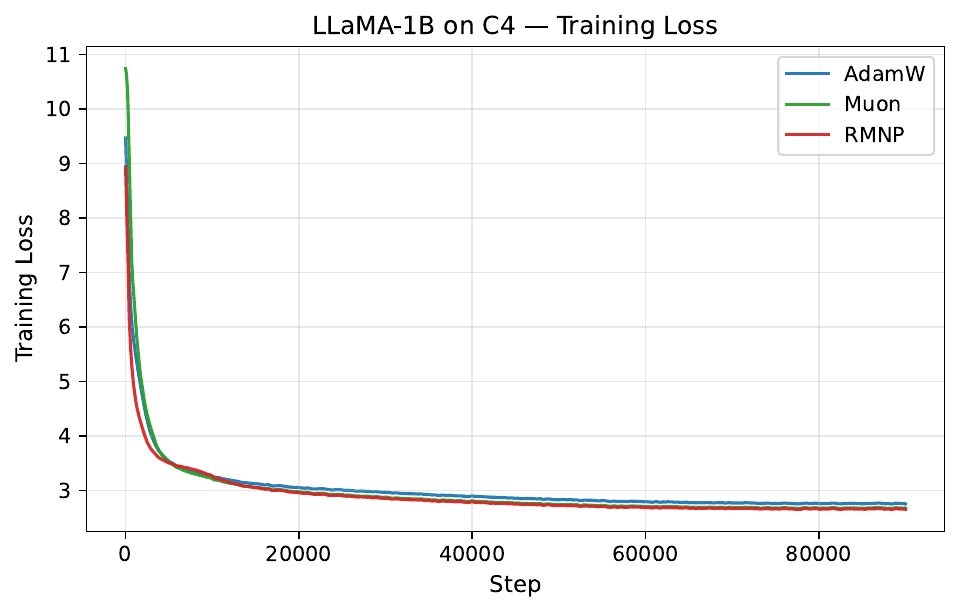}
        \caption{Training Loss}
    \end{subfigure}\hfill
    \begin{subfigure}[t]{0.48\linewidth}
        \centering
        \includegraphics[width=\linewidth]{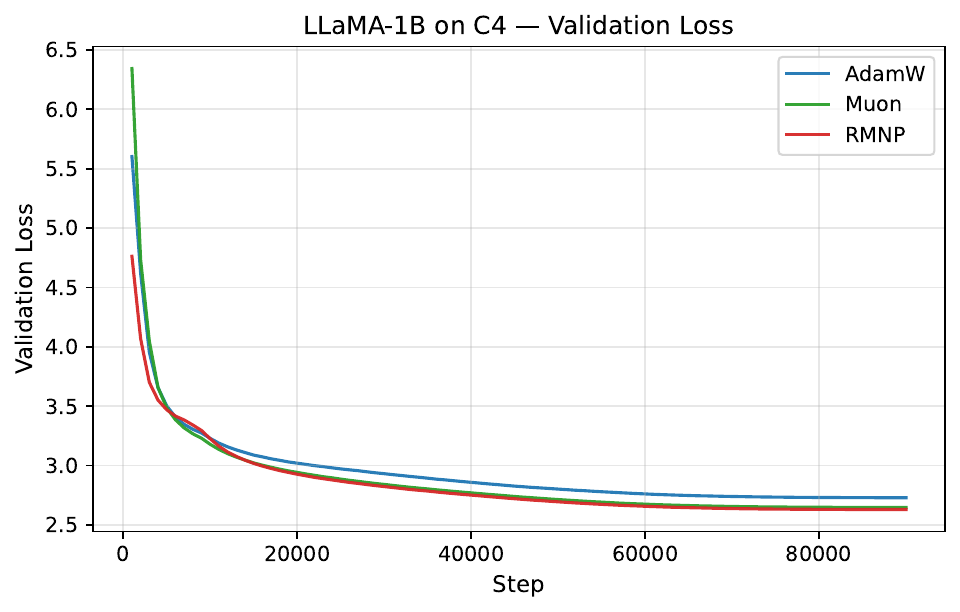}
        \caption{Validation Loss}
    \end{subfigure}
    \caption{\textsc{LLaMA-1B} on C4. The \textsc{RMNP} curve tracks \textsc{Muon} closely on both training and validation loss while delivering a substantially lower preconditioning cost.}
    \label{fig:c4_1b_curves}
\end{figure}

\subsection{Mamba on FineWeb-Edu}
\label{appendix:curves_mamba}

We additionally evaluate \textsc{RMNP} on a Mamba state-space language model trained on FineWeb-Edu to verify that the row-wise normalized preconditioner generalizes beyond Transformer attention. Figure~\ref{fig:fwedu_mamba_curves} reports the training loss and validation perplexity, comparing \textsc{AdamW}, \textsc{Muon}, and \textsc{RMNP}. Despite the architectural difference, \textsc{RMNP} tracks \textsc{Muon} essentially in lockstep and both clearly outperform \textsc{AdamW}.

\begin{figure}[!htbp]
    \centering
    \begin{subfigure}[t]{0.48\linewidth}
        \centering
        \includegraphics[width=\linewidth]{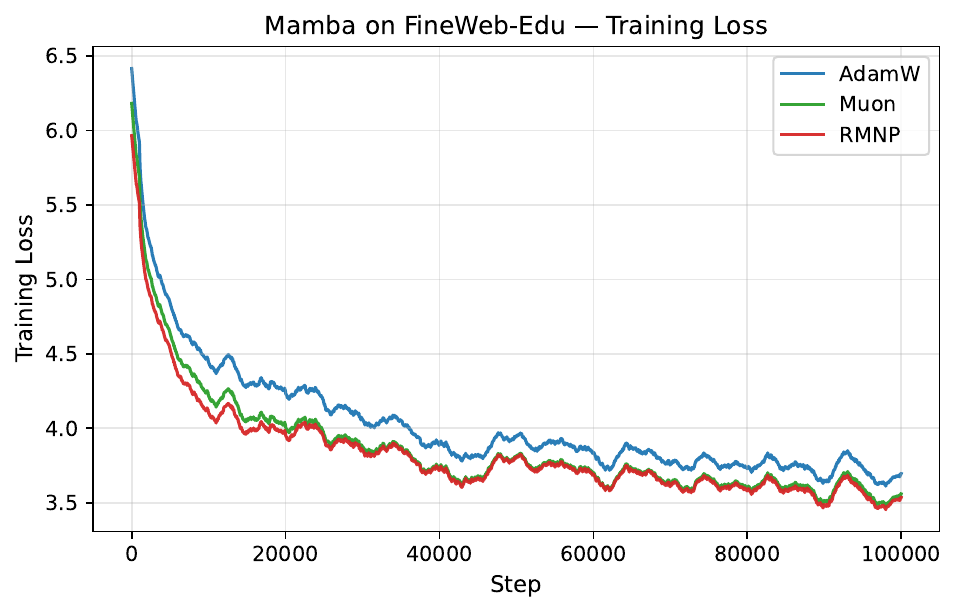}
        \caption{Training Loss}
    \end{subfigure}\hfill
    \begin{subfigure}[t]{0.48\linewidth}
        \centering
        \includegraphics[width=\linewidth]{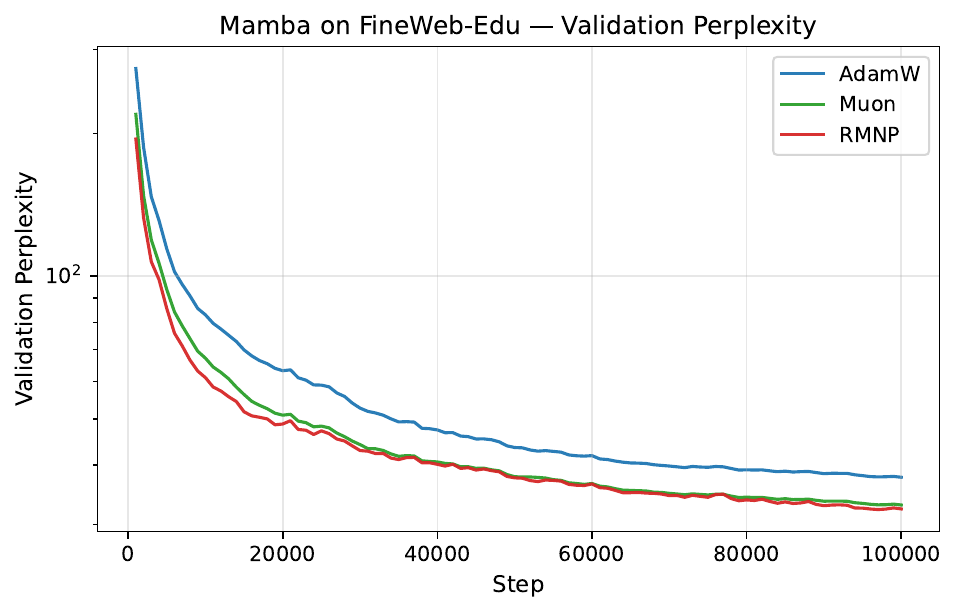}
        \caption{Validation Perplexity}
    \end{subfigure}
    \caption{Mamba on FineWeb-Edu. Validation perplexity is shown on a log scale. \textsc{RMNP} matches \textsc{Muon} throughout training and clearly outperforms \textsc{AdamW}, demonstrating that the row-wise normalized preconditioner generalizes beyond Transformer architectures to state-space models.}
    \label{fig:fwedu_mamba_curves}
\end{figure}

The same diagonal-dominance property observed for Transformer-family models continues to hold for Mamba's matrix parameters. Figure~\ref{fig:Mamba_dominance_combined_appendix} reports both the global aggregate metrics (panel~(a)) and the per-parameter metrics for three representative matrix parameters (panel~(b)) of Mamba; all three ratio metrics rise above the threshold $r=1$ shortly after warm-up and remain there throughout training.

\begin{figure}[!htbp]
    \centering
    \begin{subfigure}[c]{0.46\linewidth}
        \centering
        \includegraphics[width=\linewidth]{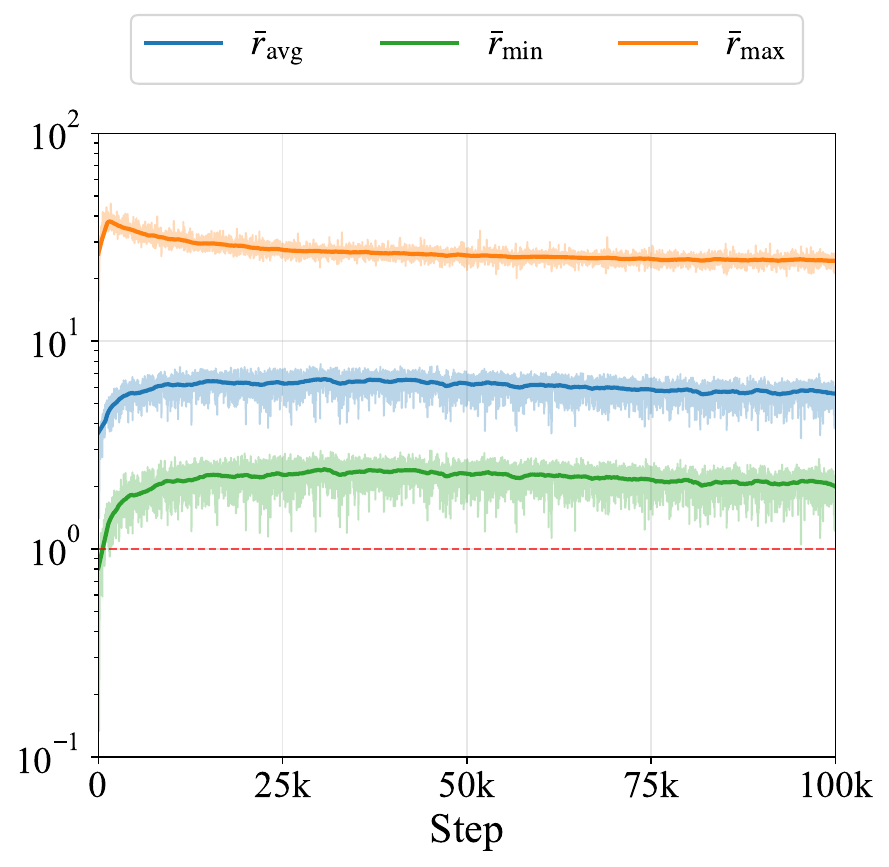}
        \caption{Global ratios $\overline{r}_{\text{avg}}$, $\overline{r}_{\min}$, $\overline{r}_{\max}$ (log-scale y-axis).}
        \label{fig:Mamba_global_dominance_curves_appendix}
    \end{subfigure}\hfill
    \begin{subfigure}[c]{0.51\linewidth}
        \centering
        \includegraphics[width=\linewidth]{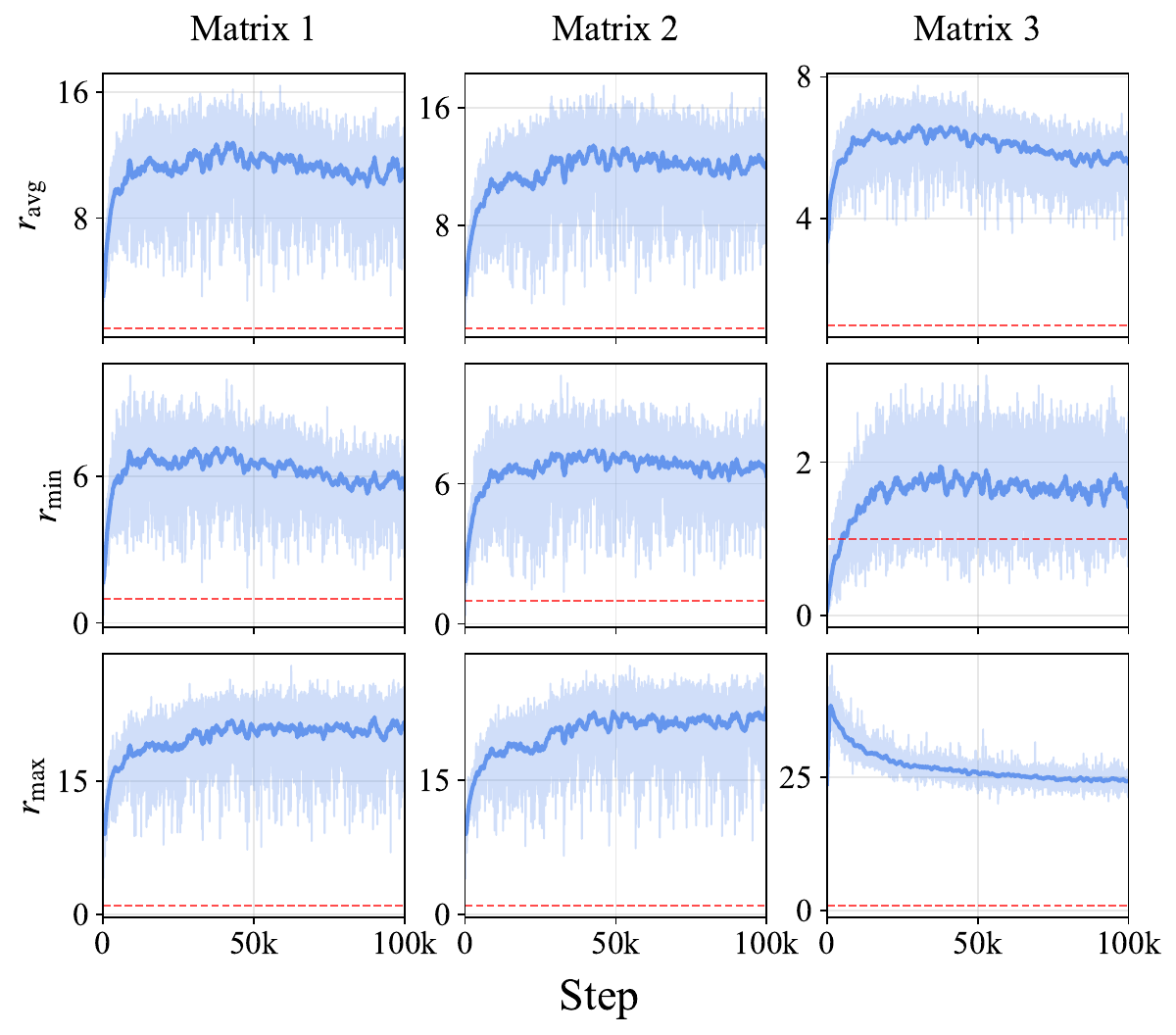}
        \caption{Per-parameter ratios $r_{\text{avg}}$, $r_{\min}$, $r_{\max}$ (rows) for three representative matrix parameters (columns).}
        \label{fig:Mamba_dominance_curves_appendix}
    \end{subfigure}
    \caption{Diagonal dominance ratios for Mamba pre-training on FineWeb-Edu. Transparent curves: raw values; solid curves: smoothed with window size 50. Red dashed line: $y=1$ threshold. All metrics remain above the threshold throughout training, demonstrating that the row-wise block-diagonal dominance property holds for the Mamba state-space architecture both at the global aggregate level (panel~(a)) and at the per-parameter level (panel~(b)).}
    \label{fig:Mamba_dominance_combined_appendix}
\end{figure}

The learning-rate sweep underlying the Mamba experiment is reported in Table~\ref{tab:mamba_hyperparam}. We fix the \textsc{AdamW} learning rate at $1\times 10^{-4}$ and sweep the matrix learning rate; the table reports final validation perplexity (lower is better).

\begin{table}[!htbp]
\centering
\caption{Hyperparameter search on Mamba (FineWeb-Edu) with \textsc{AdamW} learning rate fixed at $1\times 10^{-4}$. Validation perplexity is reported.}
\label{tab:mamba_hyperparam}
\begin{tabular}{lccc}
\toprule
Matrix LR & $6.67\times 10^{-4}$ & 0.008 & 0.009 \\
\midrule
\textsc{Muon} & 36.55 & \textbf{32.95} & 33.02 \\
\midrule
Matrix LR & $6.67\times 10^{-4}$ & $8\times 10^{-4}$ & $1\times 10^{-3}$ \\
\midrule
\textsc{RMNP} & 32.56 & \textbf{32.32} & 32.33 \\
\bottomrule
\end{tabular}
\end{table}

\subsection{ResNet-18 on CIFAR-10}
\label{appendix:curves_cv}

To verify that \textsc{RMNP} is competitive on architectures and modalities outside of language modeling, we compare \textsc{RMNP} and \textsc{Muon} on the canonical \textsc{ResNet-18} / \textsc{CIFAR-10} image-classification benchmark. Figure~\ref{fig:cv_resnet18_cifar10_curves} reports the training/test loss and training/test accuracy for both optimizers (\textsc{AdamW} omitted to keep the comparison focused on the two matrix-aware methods). \textsc{RMNP} closely tracks \textsc{Muon} throughout training and converges to essentially identical final accuracy, indicating that the row-wise normalized preconditioner is effective in the convolutional regime as well.

\begin{figure}[!htbp]
    \centering
    \begin{subfigure}[t]{0.48\linewidth}
        \centering
        \includegraphics[width=\linewidth]{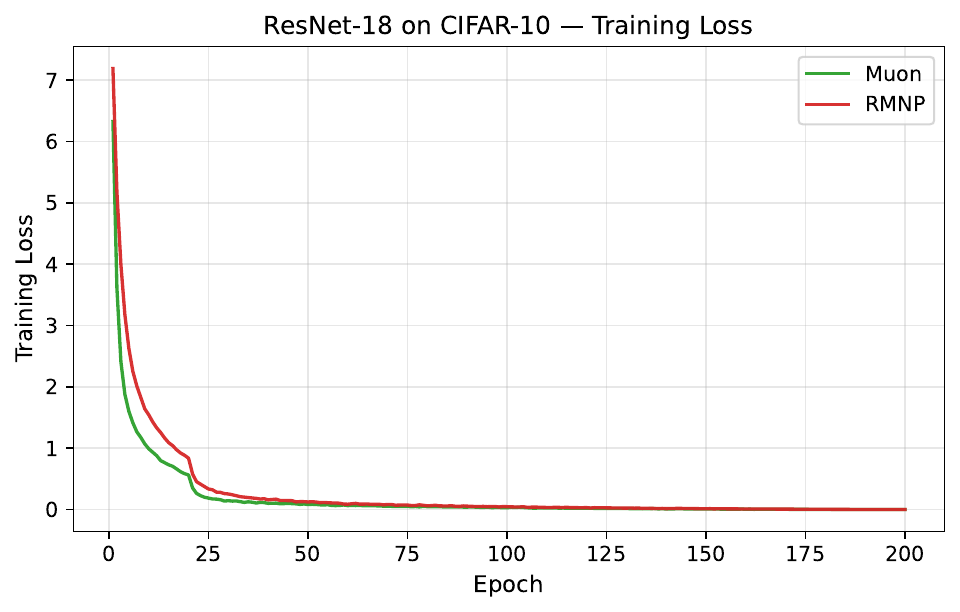}
        \caption{Training Loss}
    \end{subfigure}\hfill
    \begin{subfigure}[t]{0.48\linewidth}
        \centering
        \includegraphics[width=\linewidth]{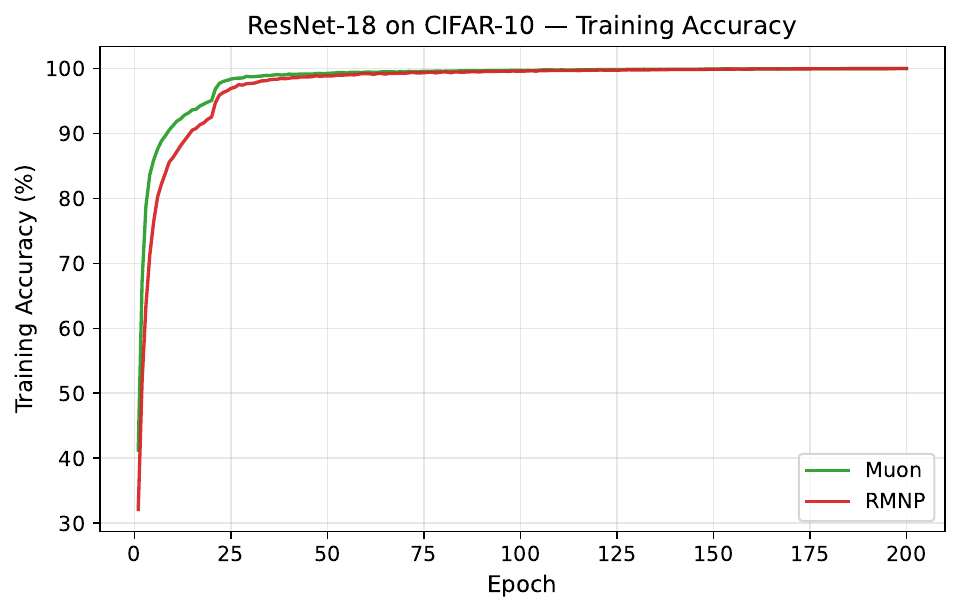}
        \caption{Training Accuracy}
    \end{subfigure}

    \vspace{0.4em}

    \begin{subfigure}[t]{0.48\linewidth}
        \centering
        \includegraphics[width=\linewidth]{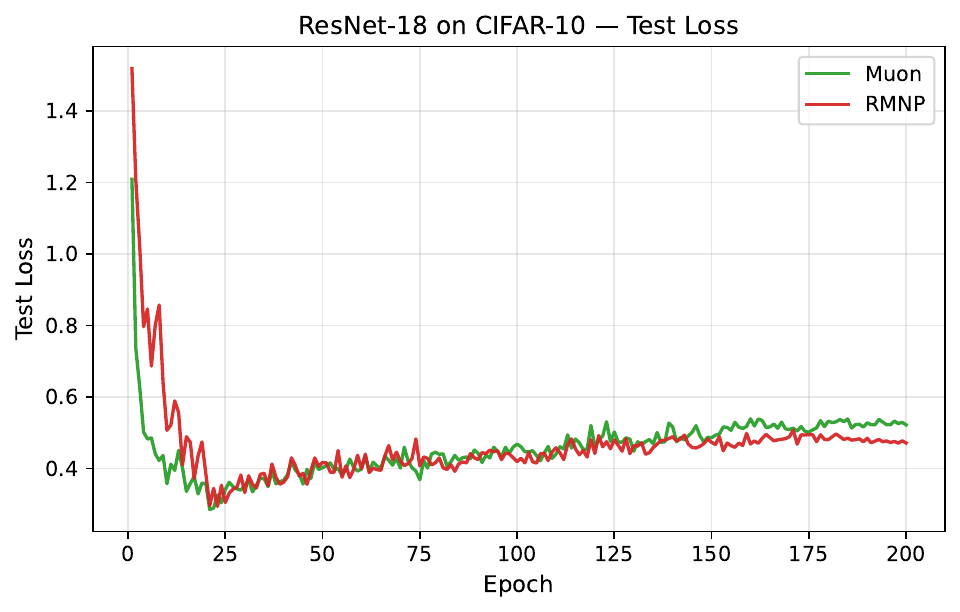}
        \caption{Test Loss}
    \end{subfigure}\hfill
    \begin{subfigure}[t]{0.48\linewidth}
        \centering
        \includegraphics[width=\linewidth]{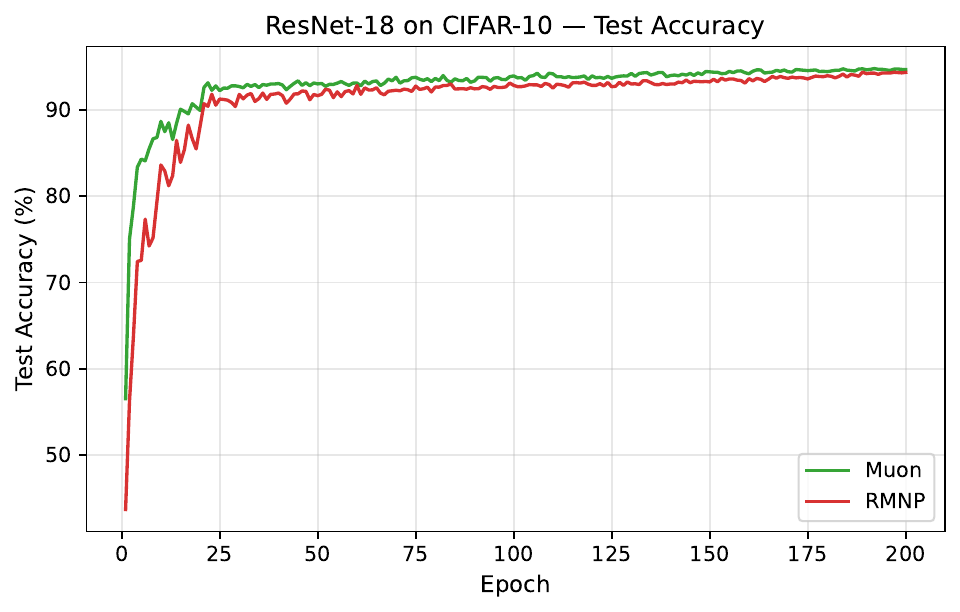}
        \caption{Test Accuracy}
    \end{subfigure}
    \caption{\textsc{ResNet-18} on \textsc{CIFAR-10}, comparing \textsc{Muon} and \textsc{RMNP}. The two matrix-aware optimizers track each other closely and converge to essentially identical final accuracy, demonstrating that \textsc{RMNP} extends to convolutional vision tasks without loss of optimization quality.}
    \label{fig:cv_resnet18_cifar10_curves}
\end{figure}

We also extend the diagonal-dominance analysis of Section~\ref{appendix:diagonal_dominance_setup} to \textsc{ResNet-18}: the row-wise block-diagonal dominance property continues to hold beyond fully-connected matrix parameters. Figure~\ref{fig:CV_dominance_combined_appendix} reports both the global aggregate metrics (panel~(a)) and the per-parameter metrics for three representative matrix parameters (panel~(b)).

\begin{figure}[!htbp]
    \centering
    \begin{subfigure}[c]{0.46\linewidth}
        \centering
        \includegraphics[width=\linewidth]{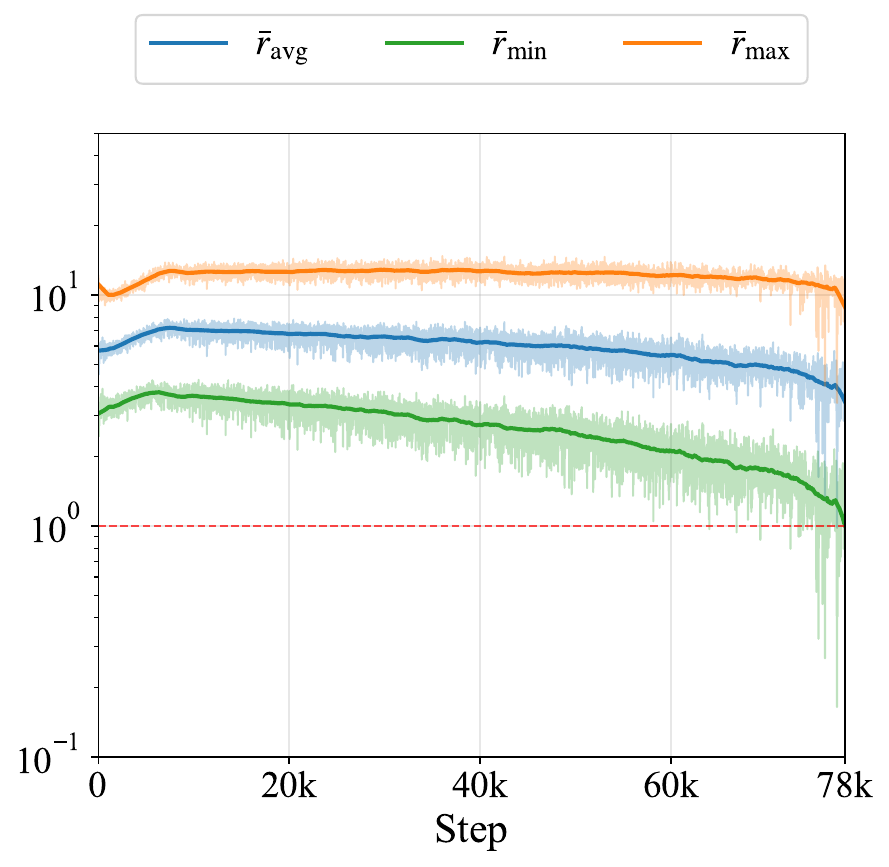}
        \caption{Global ratios $\overline{r}_{\text{avg}}$, $\overline{r}_{\min}$, $\overline{r}_{\max}$ (log-scale y-axis).}
        \label{fig:CV_global_dominance_curves_appendix}
    \end{subfigure}\hfill
    \begin{subfigure}[c]{0.51\linewidth}
        \centering
        \includegraphics[width=\linewidth]{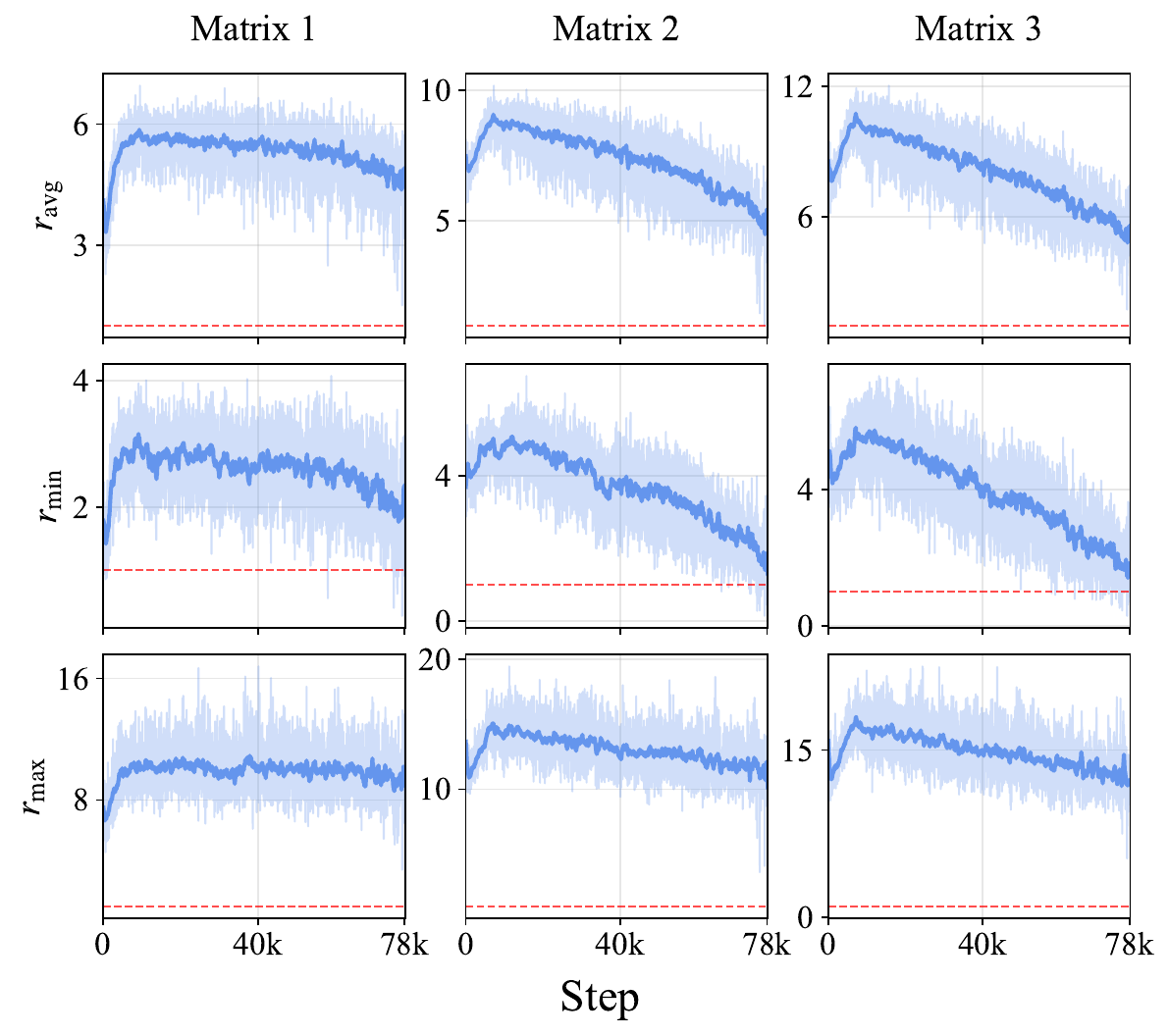}
        \caption{Per-parameter ratios $r_{\text{avg}}$, $r_{\min}$, $r_{\max}$ (rows) for three representative matrix parameters (columns).}
        \label{fig:CV_dominance_curves_appendix}
    \end{subfigure}
    \caption{Diagonal dominance ratios for \textsc{ResNet-18} training on \textsc{CIFAR-10}. Transparent curves: raw values; solid curves: smoothed with window size 50. Red dashed line: $y=1$ threshold. All metrics remain above the threshold throughout training, demonstrating that the row-wise block-diagonal dominance property holds for the convolutional vision architecture both at the global aggregate level (panel~(a)) and at the per-parameter level (panel~(b)).}
    \label{fig:CV_dominance_combined_appendix}
\end{figure}

The matrix learning-rate sweep for \textsc{ResNet-18} is reported in Table~\ref{tab:cv_resnet18_hyperparam}. We fix the \textsc{AdamW} learning rate at $0.006$ and sweep the matrix learning rate; the table reports final test accuracy (higher is better).

\begin{table}[!htbp]
\centering
\caption{Test accuracy (\%) on CIFAR-10 for \textsc{ResNet-18}: matrix LR search with \textsc{AdamW} learning rate fixed at $0.006$. Higher is better.}
\label{tab:cv_resnet18_hyperparam}
\begin{tabular}{lccc}
\toprule
Matrix LR & 0.01 & 0.04 & 0.05 \\
\midrule
\textsc{Muon} & 94.57 & \textbf{94.65} & 94.39 \\
\midrule
Matrix LR & 0.006 & 0.008 & 0.01 \\
\midrule
\textsc{RMNP} & \textbf{94.33} & 93.93 & 94.31 \\
\bottomrule
\end{tabular}
\end{table}

\subsection{Gradient Clip-Rate Trajectories}
\label{appendix:clip_rate}

We additionally report the gradient clip rate (the per-step fraction of times the gradient norm exceeds the clip threshold) for the GPT-2 pre-training runs on OpenWebText and FineWeb-Edu-100B. Two views are provided per dataset:
\begin{itemize}
\item \emph{Per-size grid} (Figures~\ref{fig:owt_clip_rate_per_size} and~\ref{fig:fwedu_clip_rate_per_size}) overlays \textsc{AdamW}, \textsc{Muon}, and \textsc{RMNP} within each cell, with the raw step on the x-axis. \textsc{RMNP} is drawn on top.
\item \emph{Cross-scale comparison} (Figures~\ref{fig:owt_clip_rate_size_comparison} and~\ref{fig:fwedu_clip_rate_size_comparison}) places the four model scales together within each optimizer panel, with the x-axis rescaled to relative training progress (\%). Within each panel a single hue is used and shade encodes model scale, so a darker line is a larger model.
\end{itemize}
Across both datasets, larger models keep their gradients clipped for a longer fraction of training, with \textsc{AdamW} on \textsc{GPT-2 XLarge} an extreme case where every step is clipped throughout the run. \textsc{RMNP} consistently begins to release the clip threshold earliest of the three optimizers, indicating that its row-normalized update reduces gradient-norm volatility relative to both \textsc{AdamW} and \textsc{Muon}.

\begin{figure}[!htbp]
    \centering
    \includegraphics[width=1.00\linewidth]{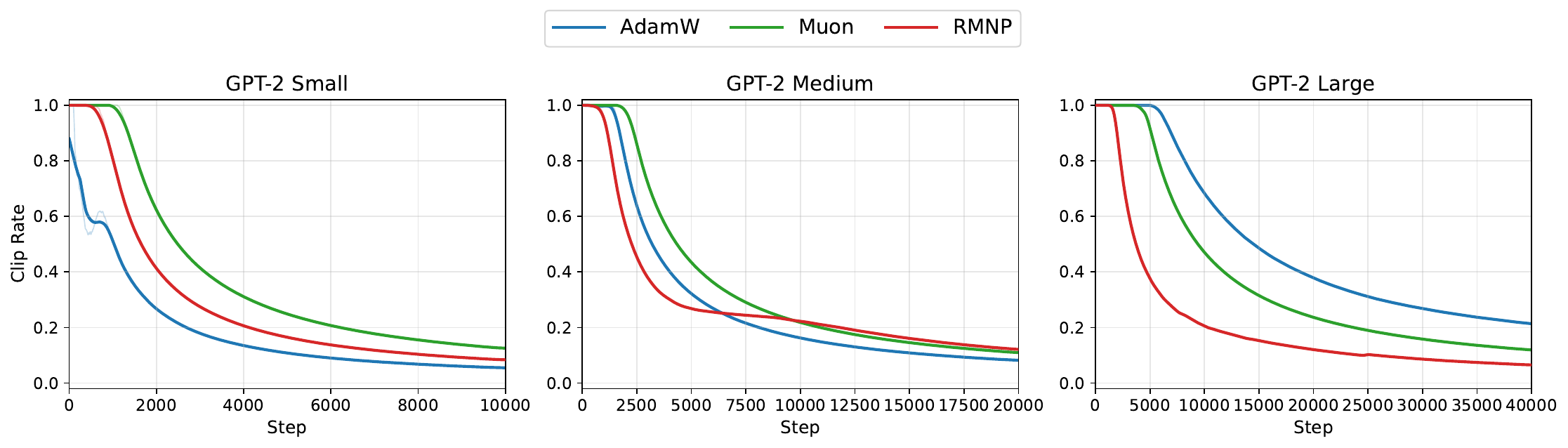}
    \caption{Gradient clip rate during GPT-2 pre-training on OpenWebText, one panel per model size. Transparent line: raw values; solid line: 50-step rolling mean. \textsc{RMNP} (red) is drawn last so it sits on top of \textsc{AdamW} (blue) and \textsc{Muon} (green).}
    \label{fig:owt_clip_rate_per_size}
\end{figure}

\begin{figure}[!htbp]
    \centering
    \includegraphics[width=1.00\linewidth]{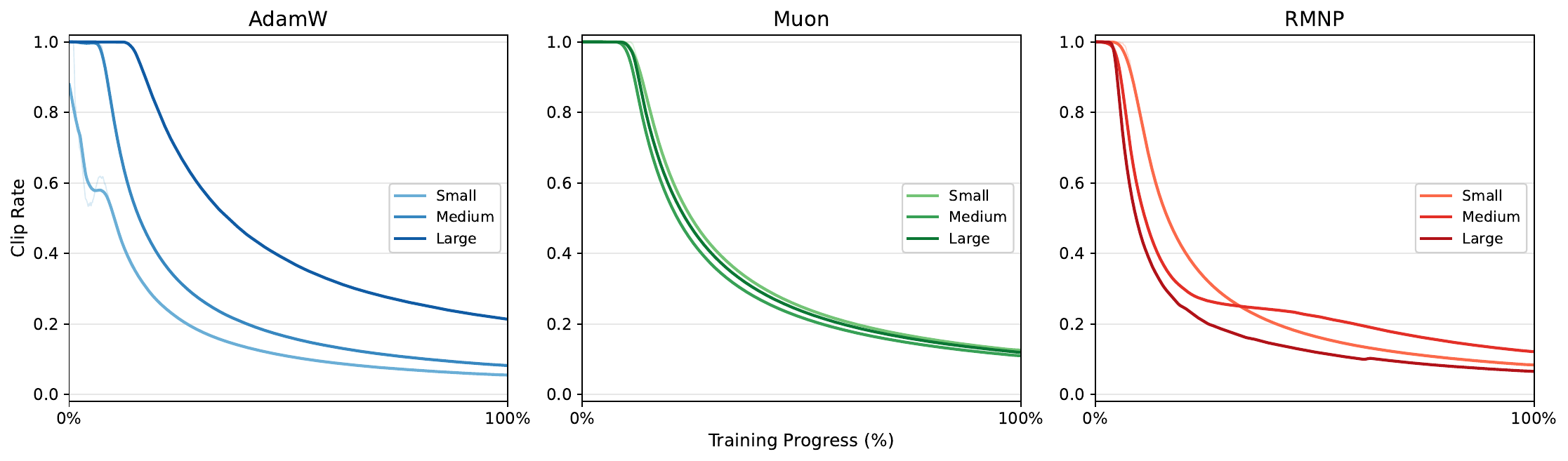}
    \caption{Gradient clip rate during GPT-2 pre-training on OpenWebText, with x-axis rescaled to relative training progress (\%). Each panel shows one optimizer; within a panel, lighter to darker shades encode \textsc{Small}/\textsc{Medium}/\textsc{Large}. The clip rate falls below 1.0 progressively later for larger models.}
    \label{fig:owt_clip_rate_size_comparison}
\end{figure}

\begin{figure}[!htbp]
    \centering
    \includegraphics[width=0.95\linewidth]{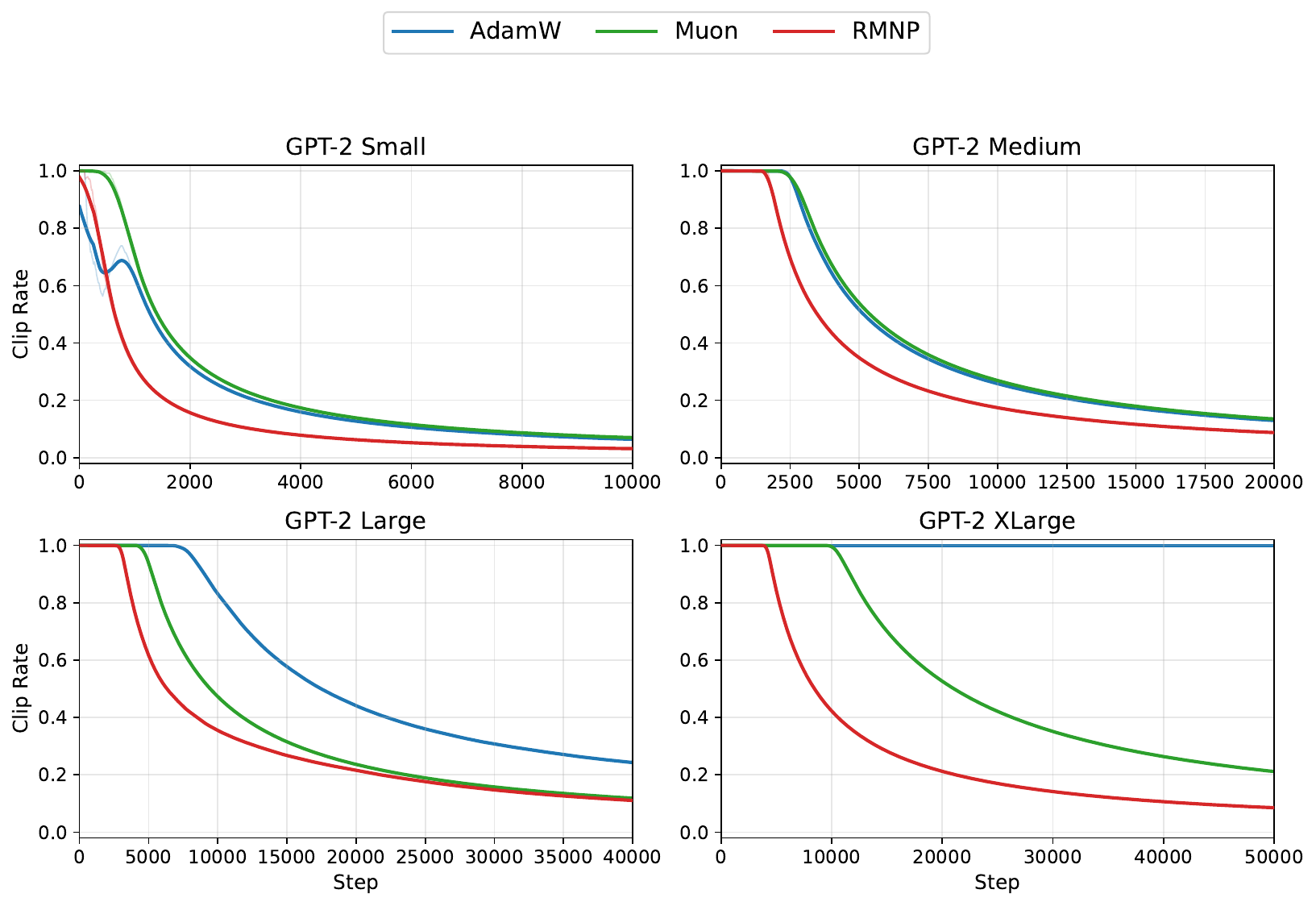}
    \caption{Gradient clip rate during GPT-2 pre-training on FineWeb-Edu-100B, one panel per model size. Transparent line: raw values; solid line: 50-step rolling mean. \textsc{AdamW} on the \textsc{XLarge} (1.5B) model has its gradients clipped at every step throughout the entire run; both \textsc{Muon} and \textsc{RMNP} progressively reduce the clip rate.}
    \label{fig:fwedu_clip_rate_per_size}
\end{figure}

\begin{figure}[!htbp]
    \centering
    \includegraphics[width=1.00\linewidth]{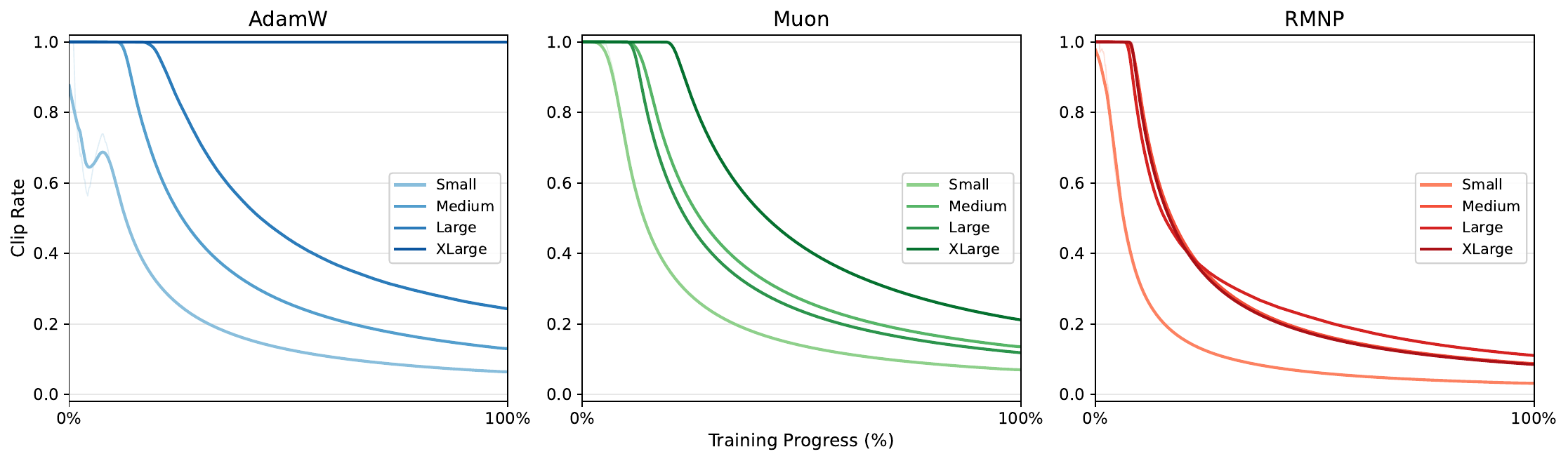}
    \caption{Gradient clip rate during GPT-2 pre-training on FineWeb-Edu-100B, with x-axis rescaled to relative training progress (\%). Each panel shows one optimizer; lighter to darker shades within a panel encode \textsc{Small}/\textsc{Medium}/\textsc{Large}/\textsc{XLarge}. The size-dependent delay before the clip rate begins to drop is most pronounced under \textsc{AdamW} (left) and least pronounced under \textsc{RMNP} (right).}
    \label{fig:fwedu_clip_rate_size_comparison}
\end{figure}

\stopcontents[appendix]

\end{document}